%% file: neurips_2020.tex
\documentclass{article}

\pdfoutput=1
\usepackage[preprint, nonatbib]{neurips_2020}

\usepackage[utf8]{inputenc} %
\usepackage[T1]{fontenc}    %
\usepackage{hyperref}       %
\usepackage{url}            %
\usepackage{booktabs}       %
\usepackage{amsfonts}       %
\usepackage{nicefrac}       %
\usepackage{microtype}      %

\usepackage{tikz}
\usepackage{verbatim}
\usetikzlibrary{tikzmark}
\usetikzlibrary{arrows}
\usetikzlibrary{patterns,snakes}
\usepackage{xcolor}
\usepackage{amssymb,amsmath,amsfonts,amsthm,mathtools}
\usepackage{float}
\usepackage{xspace}
\usepackage{multirow}
\usepackage{makecell}
\usepackage{subcaption}
\usepackage{wrapfig}

\usepackage[ruled,vlined,linesnumbered]{algorithm2e}
\DontPrintSemicolon

\newcommand{\subopt}{{\tt SUBOPT}}
\newcommand{\aware}{{\small\sf Aware}\xspace}
\newcommand{\erm}{{\small\sf ERM}\xspace}
\newcommand{\incsaga}{{\small\sf STRSAGA}\xspace}
\newcommand{\incsagaH}{{\large\sf STRSAGA}\xspace}
\newcommand{\dsurf}{{\small\sf DriftSurf}\xspace}
\newcommand{\dsurfH}{{\large\sf DriftSurf}\xspace}

\newcommand{\strisk}{{\cal H}}
\newcommand{\weight}{\mathbf{w}}
\newcommand{\sample}{{\cal S}}
\newcommand{\newsample}{{\cal T}}
\newcommand{\funcclass}{\cal F}
\newcommand{\expec}[2]{{\mathbb E}_{#1}[ #2]}
\newcommand{\var}[1]{{\mathbb Var}\left [ #1 \right ]}
\newcommand{\risk}{{\cal R}}
\newcommand{\ex}{\mathbf{x}}
\newcommand{\Ex}{\mathbf{X}}
\newcommand{\argmin}{\mbox{arg min}}
\newcommand{\buffer}{{\tt WaitingRoom}}
\newcommand{\prob}[1]{{\mathbb P}\left( #1 \right )}
\newcommand{\rcount}{i}

\makeatletter
\newtheorem*{rep@theorem}{\rep@title}
\newcommand{\newreptheorem}[2]{%
\newenvironment{rep#1}[1]{%
 \def\rep@title{#2 \ref{##1}}%
 \begin{rep@theorem}}%
 {\end{rep@theorem}}}

\newtheorem{lemma}{Lemma}
\newreptheorem{lemma}{Lemma}

\newreptheorem{theorem}{Theorem}
\newtheorem{corollary}{Corollary}
\newreptheorem{corollary}{Corollary}

\newtheorem{definition}{Definition}

\newcommand{\remove}[1]{}
\newcommand{\cmmnt}[1]{\ignorespaces}

\title{DriftSurf: A Risk-competitive Learning Algorithm under Concept Drift}

\author{%
 Ashraf Tahmasbi\thanks{Equal contribution}\\
 Iowa State University\\
 \texttt{tahmasbi@iastate.edu} \\
\And
 Ellango Jothimurugesan\footnotemark[1]\\
 Carnegie Mellon University\\
 \texttt{ejothimu@cs.cmu.edu} \\
\AND
 Srikanta Tirthapura\\
 Iowa State University\\
 \texttt{snt@iastate.edu} \\
\And
 Phillip B. Gibbons\\
 Carnegie Mellon University\\
 \texttt{gibbons@cs.cmu.edu} \\
}

\begin{document}

\maketitle

\input{abstract.tex}
\input{intro.tex}
  \setlength{\textfloatsep}{0pt}
\input{related.tex}

\input{prelim.tex}

\input{algo.tex}

\input{analysis.tex}

\input{expt.tex}

\input{conclusion.tex}

\begin{ack}
Funding in direct support of this work: NSF grants 1527541, 1725702, 1725663
\end{ack}

\bibliography{smf}
\bibliographystyle{abbrv}

\newpage
\onecolumn
\appendix
\input{appendix-pseudocode}

\input{appendix-proofs}
\input{appendix-expt-setup}

\input{appendix-expt-results}

\input{broader.tex}

\end{document}

%% file: abstract.tex
\begin{abstract}
When learning from streaming data, a change in the data distribution, also known as concept drift, can render a previously-learned model inaccurate and require training a new model. We present an adaptive learning algorithm that extends previous drift-detection-based methods by incorporating drift detection into a broader stable-state/reactive-state process. The advantage of our approach is that we can use aggressive drift detection in the stable state to achieve a high detection rate, but mitigate the false positive rate of standalone drift detection via a reactive state that reacts quickly to true drifts while eliminating most false positives. The algorithm is generic in its base learner and can be applied across a variety of supervised learning problems. Our theoretical analysis shows that the risk of the algorithm is competitive to an algorithm with oracle knowledge of when (abrupt) drifts occur. Experiments on synthetic and real datasets with concept drifts confirm our theoretical analysis.
\end{abstract}

%% file: intro.tex
\section{Introduction}
\label{sec:intro}

Learning from streaming data is an ongoing process in which a model is
continuously updated as new training data arrive. We focus on the
problem of concept drift, which refers to an unexpected change in the
distribution of data over time. The objective is high prediction
accuracy at each time step on test data from the current
distribution. To achieve this goal, a learning algorithm should adapt
quickly whenever drift occurs by focusing on the \textit{most recent
  data points} that represent the new concept, while also, in the
absence of drift, optimizing over \textit{all the past data points}
from the current distribution (for statistical accuracy).  The latter
has greater importance in the setting we consider where data points
may be stored and revisited to achieve accuracy greater than what can
be obtained in a single pass. Moreover, computational efficiency of
the learning algorithm is critical to keep pace with the continuous
arrival of new data.

In a survey from Gama et al.~\cite{gama2014survey}, concept drift
between time steps $t_0$ and $t_1$ is defined as a change in the joint
distribution of examples: $p_{t_0}(X, y) \ne p_{t_1}(X, y)$. Gama et
al. categorize drifts in several ways, distinguishing between
\emph{real drift} that is a change in $p(y|X)$ and \emph{virtual
  drift} (also known as \emph{covariate drift}) that is a change only
in $p(X)$ but not $p(y|X)$. Drift is also categorized as either
\emph{abrupt} when the change happens across one time step, or
\emph{gradual} if there is a transition period between the two
concepts.

A learning algorithm that reacts (well) to concept drift is referred
to as an \emph{adaptive algorithm}. In contrast, an \emph{oblivious
  algorithm}, which optimizes the empirical risk over all data points
observed so far under the assumption that the data are i.i.d.,
performs poorly in the presence of drift. One major class of adaptive
algorithms is drift detection, which includes DDM
\cite{gama2004learning}, EDDM \cite{baena2006early}, ADWIN
\cite{bifet2007learning}, PERM \cite{harel2014concept}, FHDDM
\cite{pesaranghader2016fast}, and MDDM
\cite{pesaranghader2018mcdiarmid}. Drift detection tests commonly work
by tracking the prediction accuracy of a model over time, and signal
that a drift has occurred whenever the accuracy degrades by more than a
significant threshold. After a drift is signaled, the
previously-learned model can be discarded and replaced with a model
trained solely on the data going forward.

There are a couple of challenges with using drift detection. Different
tests are preferred depending on whether a drift is abrupt or gradual,
and most drift detection tests have a user-defined parameter that
governs a trade-off between the detection accuracy and speed
\cite{gama2014survey}; choosing the right test and the right
parameters is hard when the types of drift that will occur are not
known in advance. There is also a significant cost in prediction
accuracy when a false positive results in the discarding of a long-trained
model and data that are still relevant. Furthermore, even when drift
is accurately detected, not all drifts require restarting with a new
model. Drift detection can trigger following a virtual drift when the
model misclassifies data points drawn from a previously unobserved
region of the feature space, but the older data still have valid
labels and should be retained. We have also encountered real drifts in
our experimental study where a model with high parameter dimension can
adapt to simultaneously fit data from both the old and new concepts,
and it is more efficient to continue updating the original model
rather than starting from scratch.

Our contribution is \dsurf, an adaptive algorithm that overcomes these
drift detection challenges. \dsurf works by maintaining two
models at each time step and incorporating drift detection into a
broader two-state process. The algorithm begins in the \emph{stable
  state} and transitions to the \emph{reactive state} based on a drift
detection trigger, and then starts a new model. During the reactive
state, the model used for prediction is greedily chosen as the best
performer over data from the immediate previous time step (each time
step corresponds to a batch of arriving data points). At the
end of the reactive state, the algorithm transitions back to
the stable state, keeping the model that was the best performer
throughout the entire reactive state. Our approach has several
advantages over standalone drift detection: (i) most false positives
will be caught by the reactive state and lead to continued use of the
original long-trained model and all the relevant past data; (ii) when
restarting with a new model does not lead to better post-drift
performance, the original model will continue to be used; and (iii)
switching to the new model for predictions happens only when it begins
outperforming the old model, accounting for potentially lower
accuracy of the new model as it warms up. Meanwhile, the addition of
this stable-state/reactive-state process does not unduly delay
the time to recover from a drift, because the switch to a new
model happens greedily within one time step of it outperforming the
old model (as opposed to switching only at the end of the reactive
state).

We present a theoretical analysis of \dsurf, showing that it is
``risk-competitive'' with \aware, an adaptive algorithm
that has oracle access to when a drift
occurs and at each time step maintains a model trained over the set of
all data since the previous drift. We also provide experimental
comparisons of \dsurf to \aware and two adaptive learning algorithms: a
state-of-the-art drift-detection-based method MDDM and a
state-of-the-art ensemble method AUE~\cite{brzezinski2013reacting}.
Our results on eight datasets with concept drifts show that \dsurf
generally outperforms both MDDM and AUE.

%% file: related.tex
\section{Related Work}
\label{sec:related}

Most adaptive learning algorithms can be classified into three major categories: Window-based, drift detection, and ensembles. Window-based methods, which include the family of FLORA algorithms \cite{widmer1996learning} train models over a sliding window of the recent data in the stream. Alternatively, older data can be forgotten gradually by weighting the data points according to their age with either linear \cite{koychev2000gradual} or exponential \cite{hentschel2019online,
  klinkenberg2004learning} decay. Window-based methods are guaranteed to adapt to drifts, but at a cost in accuracy in the absence of drift.

The aforementioned drift detection methods can be further classified as either detecting degradation in prediction accuracy with respect to a given model, which include all of the tests mentioned in Section \ref{sec:intro}, or detecting change in the underlying data distribution which include tests given by \cite{kifer2004detecting, sebastiao2007change}. In this paper, we focus on the subset of concept drifts that are performance-degrading, and that can be detected by the first class of these drift detection methods.  As observed in \cite{harel2014concept}, under this narrower focus, the problem of drift detection has lower sample and computational complexity when the feature space is high-dimensional. Furthermore, this approach ignores drifts that do not require adaptation, such as changes only in features that are weakly correlated with the label.

Finally, there are ensemble methods, such as DWM \cite{kolter2007dynamic}, Learn++.NSE \cite{elwell2011incremental}, AUE \cite{brzezinski2013reacting}, DWMIL \cite{lu2017dynamic}, DTEL \cite{sun2018concept}, Diversity Pool \cite{chiu2018diversity}, and Condor \cite{zhao2020handling}. 
An ensemble is a collection of individual models, often referred to as experts, that differ in the subset of the stream they are trained over. Ensembles adapt to drift by including both older experts that perform best in the absence of drift and newer experts that perform best after drifts. The predictions of each individual expert are typically combined using a weighted vote, where the weights depend on each expert's recent prediction accuracy. Strictly speaking, \dsurf is an ensemble method, but differs from traditional ensembles by maintaining only two models and where only one model is used to make a prediction at any time step. The advantage of \dsurf is its efficiency, as the maintenance of each additional model in an ensemble comes at either a cost in additional training time, or at a cost in the accuracy of each individual model if the available training time is divided among them. The ensemble algorithm  most similar to ours is from \cite{bach2008paired}, which also maintains just two models: a long-lived model that is best-suited in the stationary case, and a newer model trained over a sliding window that is best-suited in the case of drift. Their algorithm differs from \dsurf in that instead of using a drift detection test to switch, they are essentially always in what we call the reactive state of our algorithm, where they choose to switch to a new model whenever its performance is better over a window of recent data points. Their algorithm has no theoretical guarantee, and without the stable-state/reactive-state process of our algorithm, there is no control over false switching to the newer model in the stationary case.

%% file: prelim.tex
\section{Model and Preliminaries}
\label{prelim}

We consider a data stream setting in which the training data points
arrive over time.  For $t=1,2,\ldots,$ let $\Ex_t$ be the set of
data points arriving at time step $t$. We consider a constant arrival
rate $m = |\Ex_t| > 0$ for all $t$.  (Our discussion and results can be
readily extended to Poisson and other arrival distributions.)  Let
$\sample_{t_1,t_2} = \cup_{t=t_1}^{t_2-1} \Ex_t$ be a segment of the
stream of points arriving in time steps $t_1$ through $t_2-1$.  Let
$n_{t_1, t_2} = m (t_2 - t_1)$ be the number of data points in
$\sample_{t_1,t_2}$.
Each $\Ex_t$ consists of data points drawn from a distribution $I_t$
not known to the learning algorithm.  In the \textbf{stationary} case,
$I_t = I_{t-1}$; otherwise, a \textbf{concept drift} has occurred at
time $t$.  We seek an adaptive learning algorithm with high prediction
accuracy at each time step.

The model being trained is drawn from a class of functions
$\funcclass$. A function in this class is parameterized by a vector of
weights $\weight \in \mathbb{R}^{d}$. 
To achieve high prediction accuracy at time $t$, we want to minimize the expected risk over the distribution $I_t$. The \textit{expected risk} of function $\weight$ over a distribution $I$ is: $\risk_I(\weight) = \expec{\ex \sim I}{f_{\ex}(\weight)}$, where $f_\ex(\weight)$ is the loss of function $\weight$ on input $\ex$.
Given a stream segment $\sample_{t_1,t_2}$ of training data
points, the best we can do when the data are all drawn from the same distribution is to minimize the empirical risk over $\sample_{t_1,t_2}$.
The \textit{empirical risk} of function $\weight$ over
a sample $\sample$ of $n$ elements is: $\risk_{\sample}(\weight) =
\frac{1}{n}\sum_{\ex \in \sample} f_{\ex}(\weight)$. The optimizer of
the empirical risk is denoted as $\weight_{\sample}^*$, defined as
$\weight_{\sample}^* = \argmin_{\weight \in \funcclass}
\risk_{\sample}(\weight)$. The optimal empirical risk is
$\risk_{\sample}^* = \risk_{\sample}(\weight_{\sample}^*)$.

We assume that the expected risk over a distribution $I$ and the empirical risk over a sample $\sample$ of size $n$ drawn from $I$ are related through the following bound:
\begin{equation}
\label{eq:generalization}
\expec{}{\sup_{\weight \in \funcclass} | \risk_I(\weight) - \risk_\sample(\weight) |} \leq \strisk(n)/2
\end{equation}
where $\strisk(n)
= hn^{-\alpha}$, for a constant $h$ and $1/2 \leq \alpha \leq 1$. From this relation, $\strisk(n)$ is an upper bound on the statistical error (also known as the estimation error) over a sample of size $n$ \cite{bousquet2008tradeoffs}.

Let $\weight$ be the solution learned by an algorithm $A$ over stream
segment $\sample=\sample_{t_1,t_2}$. Following prior
work~\cite{bousquet2008tradeoffs, jothimurugesan2018variance}, we
define the difference between $A$'s empirical risk and the optimal
empirical risk over this stream segment as its sub-optimality:
$\subopt_{\sample}(A) := \risk_{\sample}(\weight) - \risk_{\sample}
(\weight_{\sample}^*)$.  Based on \cite{bousquet2008tradeoffs},
in the stationary case,
achieving a sub-optimality on the order of $\strisk(n_{t_1,t_2})$ over
stream segment $\sample_{t_1,t_2}$ asymptotically minimizes the total
(statistical + optimization) error for $\funcclass$.

However, suppose a concept drift occurs at time $t_d$ such that $t_1 <
t_d < t_2$.
We could
still define empirical risk and sub-optimality of an algorithm $A$
over stream segment $\sample_{t_1,t_2}$. But, balancing sub-optimality
with $\strisk(n_{t_1,t_2})$ does not necessarily minimize the total
error. Algorithm $A$ needs to first recover from the drift such that
the predictive model is trained only over data points drawn from the
new distribution.
We define recovery time as follows:
The\/ \textbf{recovery time} of an algorithm $A$ is the time it takes after a
drift for $A$ to provide a solution $\weight$ that is
maintained solely over data points drawn from the new distribution.

Let $t_{d_1}, t_{d_2},\ldots$ be the sequence of time steps at which a
drift occurs, and define $t_{d_0} = 1$.  The goals for an adaptive
learning algorithm $A$ are {\bf (G1)} to have a small recovery time $r_i$ at
each $t_{d_i}$ and {\bf (G2)} to achieve sub-optimality on the order of
$\strisk(n_{t_{d_i},t})$ over every stream segment
$\sample_{t_{d_i},t}$ for $t_{d_i}+r_i < t < t_{d_{i+1}}$ (i.e.,
during the stationary, recovered periods between drifts).  In
Section~\ref{sec:analysis}, we formalize the latter as $A$ being
``risk-competitive'' with an oracle algorithm \aware.  It implies that
$A$ is asymptotically optimal in terms of its total error, despite
concept drifts.

%% file: algo.tex
\section{\dsurfH: Adaptive Learning over Streaming Data in Presence of Drift}
\label{sec:algorithm}

We present our algorithm \dsurf for adaptively learning from streaming
data that may experience drift. Incremental learning algorithms work
by repeatedly sampling a data point from a training set $\sample$ and
using the corresponding gradient to determine an update direction. This set
$\sample$ expands as new data points arrive.
In the presence of a drift from distribution $I_1$ to $I_2$, without a
strategy to remove from $\sample$ data points from $I_1$, the model
trains over a mixture of data points from $I_1$ and $I_2$, often
resulting in poor prediction accuracy on $I_2$.
One systematic approach to mitigating this problem would be to use a
sliding window-based set $\sample$ from which further sampling is
conducted.  Old data points are removed when they fall out of the
sliding window (regardless of whether they are from the current or an
old distribution).  However, the problem with this approach is that
the sub-optimality of the model trained over $\sample$
suffers from the limited size of $\sample$. Using larger window sizes
helps with achieving a better sub-optimality, but increases the
recovery time. Smaller window sizes, on the other hand, 
provide better recovery time, but the sub-optimality of the algorithm
over $\sample$ increases. An ideal algorithm manages the set $\sample$
such that it contains as many as possible data points from the current
distribution and resets it whenever a (significant) drift happens, so
that it contains only data points from the new distribution.

As noted in Section~\ref{sec:intro}, prior
work~\cite{baena2006early, bifet2007learning, gama2004learning, 
  harel2014concept, pesaranghader2016fast, pesaranghader2018mcdiarmid}
has sought to achieve this ideal algorithm by developing better and
better drift detection tests, but with limited success due to the
challenges of balancing detection accuracy and speed, and the high
cost of false positives.  Instead, we couple aggressive drift
detection with a stable-state/reactive-state process that mitigates
the shortcomings of prior approaches.  Unlike prior drift detection
approaches, \dsurf views performance degrading as only a \textit{sign} of a
potential drift: the final decision about resetting $\sample$ and the
predictive model will not be made until the end of the reactive
state, when more evidence has been gathered and a higher
confidence decision can be made.

\begin{algorithm}[t!]
    \caption{\dsurf: Processing a set of training points $\Ex_t$ that arrives in time step $t$}
    \label{algo:SR}

\tcp{$\weight_{t-1} (\sample)$, $\weight'_{t-1} (\sample')$, $\weight''_{t-1} (\sample'')$ are respectively the parameters (stream segments for training) of the predictive, reactive, stable models}

\If{state == stable}
{
    \If (\tcp*[h]{if condition~\ref{eq:enter_reactive_first} or~\ref{eq:enter_reactive_second} holds}){Enter Reactive State}
    {
        state $\leftarrow$ reactive\;
        
        $T  \leftarrow \emptyset$ \tcp{$T$ is the segment arriving during the reactive state}
        
        $\weight'_{t-1} \leftarrow \weight_0, \sample' \leftarrow \emptyset$  \tcp{ initialize randomly a new reactive model}

        $\rcount \leftarrow 0$ \tcp{ time steps in the current reactive state}
    }
    \Else{
      $\weight_{t} \leftarrow$ Update($\weight_{t-1}, \sample, \Ex_t$), $\weight''_{t} \leftarrow$ Update($\weight''_{t-1}, \sample'', \Ex_t$) \tcp{update
        $\weight, \sample, \weight'', \sample''$}
    }
}

\If{state == reactive}
{
    Add $\Ex_t$ to $T$ 
    
    $\weight_{t} \leftarrow$ Update($\weight_{t-1}, \sample, \Ex_t$), $\weight'_{t} \leftarrow$ Update($\weight'_{t-1}, \sample', \Ex_t)$ \tcp{update
              $\weight, \sample, \weight', \sample'$}
    
    $\rcount \leftarrow \rcount + 1$\;
    
    \If(\tcp*[h]{if $\rcount == r$, the length of the reactive state}){Exit Reactive State}
    {
        state $\leftarrow$ stable\;
        
        $\weight''_{t-1} \leftarrow \weight_0, \sample'' \leftarrow \emptyset$ \tcp{initialize randomly a new stable model}
        
        \If(\tcp*[h]{if condition~\ref{eq:switch} holds}){$\risk_{T}(\weight_t) > \risk_{T}(\weight'_t)$}
        {
            $\weight_t \leftarrow \weight'_t$, $\sample \leftarrow \sample'$ \tcp{change the predictive model}
        }
    }
    \ElseIf{$\risk_{\Ex_t}(\weight'_t) < \risk_{\Ex_t}(\weight_t)$}{
    	Use $\weight'_t$ instead of $\weight_t$ for predictions at the next time step \tcp{greedy policy}
    }
}
\end{algorithm}

Our algorithm, \dsurf, is depicted in Algorithm~\ref{algo:SR}.
The algorithm starts in the stable state, and the steps are shown for processing
the batch of points arriving at time step $t$.
If \dsurf is in the stable state at time $t$, it enters the reactive state at the sign of a drift, given by the following condition: 
\begin{equation}
\label{eq:enter_reactive_first}
    \risk_{\Ex_t}(\weight_{t-1}) > \risk_b + \delta
\end{equation}
where $\weight_{t-1}$ is the parameters of the current predictive model (before updated with the current batch), $\risk_b$ is the best observed risk of this model and $\delta$ is a predetermined threshold that represents the tolerance in performance degradation.

If condition~\ref{eq:enter_reactive_first} (and condition~\ref{eq:enter_reactive_second} discussed below)
do not hold, \dsurf
assumes there was no drift in the underlying distribution and remains
in the stable state.
It calls Update, an \textit{update process} that expands $\sample$ to include the newly
arrived set of data points $\Ex_t$ and then updates the (predictive) model
parameters using $\sample$ for incremental training.
Otherwise, \dsurf enters the reactive state,
adds a new model $\weight'_{t-1}$, called the \textit{reactive model},
with randomly initialized parameters, and initializes its
sample set $\sample'$ to be empty. To save space, the growing sample set
$\sample'$ can be represented by pointers into $\sample$.

If, at time step $t$, \dsurf is in the reactive state (including the time
step that it has just entered the reactive state), \dsurf adds
$\Ex_t$ to $\sample$ and $\sample'$, sample sets of the predictive and
reactive models, and updates $\weight_{t-1}$ and
$\weight'_{t-1}$. During the reactive state, \dsurf uses for prediction
at $t$ whichever model $\weight$ or $\weight'$ performed the best
in the previous time step $t-1$.
This greedy heuristic yields 
better performance during the reactive state by switching to
the newly added model sooner in the presence of drift.

Upon exiting the reactive state (when $\rcount == r$), \dsurf chooses
the predictive model to use for the subsequent stable state.  It
switches to the reactive model $\weight'$ if the
reactive model outperforms the prior predictive model $\weight$ over the
set of data points $T$ that arrived during the $r$ time steps of the
reactive state:
\begin{equation}
\label{eq:switch}
    \risk_{T}(\weight') < \risk_{T}(\weight).
\end{equation}
Otherwise, \dsurf continues with the prior predictive model.

\textbf{Handling a corner case}. Consider the case that a drift happens
when \dsurf is in the reactive state. In
this case, no matter what predictive model \dsurf chooses at the end of the
reactive state, both the current predictive and reactive models are trained
over a mixture of data points from both the old and new distributions. This
will decrease the chance of recovering from the actual drift. To avoid
this problem, \dsurf adds a new model with parameters $\weight''$ upon
returning to the stable state. This model, which we refer to as the
\textit{stable model}, is trained over the stream segment $\sample''$ arriving after entering the stable
state. At each time step $t$, \dsurf compares the performance of the predictive and
stable models, and enters the reactive state (in addition to condition~\ref{eq:enter_reactive_first})
upon the following condition:
\begin{equation}
\label{eq:enter_reactive_second}
    \risk_{\Ex_t}(\weight_{t-1}) > \risk_{\Ex_t}(\weight''_{t-1}) + \delta'
\end{equation}
where $\delta'$ is set to be much smaller than $\delta$ (our experiments use $\delta' = \delta/2$).

Algorithm \ref{algo:SR} is generic in the individual base learner and the update 
process used for each of the parameters. For the theoretical analysis in Section 
\ref{sec:analysis} and experimental evaluation in Section \ref{sec:expt}, the 
update process we focus on is \incsaga~\cite{jothimurugesan2018variance}, which is 
a variance-reduced SGD for streaming data. Compared to SGD, \incsaga has a faster 
convergence rate and better performance under different arrival distributions. For 
any update process, we let $\rho$ denote the computational power available at each 
time step. For an SGD-based algorithm, $\rho$ is the number of gradients that can 
be computed.
The time and space complexity of \dsurf is within a constant factor of that of learning a single model.

%% file: analysis.tex
\section{Analysis of \dsurfH}
\label{sec:analysis}

In this section, we show that \dsurf achieves goals {\bf G1} and
{\bf G2} from Section~\ref{prelim}.
As in prior
work~\cite{bousquet2008tradeoffs,jothimurugesan2018variance}, we
assume that $\strisk(n) = hn^{-\alpha}$, for a constant $h$ and
$\frac{1}{2} \leq \alpha \leq 1$, is an upper bound on the statistical
error over a set of data points of size $n$ all drawn from the same
distribution.

\textbf{\aware} is an adaptive learning algorithm with oracle knowledge of when
drifts occur. At each drift, the algorithm restarts the predictive model
to a random initial point and trains it over data points that 
arrive after the drift. The main obstacle for other adaptive learning
algorithms to compete with \aware is that they are not told exactly when drifts occur.

The \aware implementation we are comparing to uses \incsaga for
incremental training (i.e., as its update process).
At any time step $t$, the sub-optimality of this algorithm over
its training sample set $\sample$ of size $n$ is bounded as follows:

\begin{lemma}
\label{lemma: subopt}
\textsc{(Lemma 3 in \cite{jothimurugesan2018variance})}
Suppose all $f_{\ex}$ are convex and their gradients are $L$-Lipschitz continuous, and that $\risk_{\sample}$ is $\mu$-strongly convex. Also, assume that the condition number $L/\mu$ is bounded by a constant at each time step. At the end of each time step, the expected sub-optimality of \incsaga over its sample set $\sample$ of size $n$ is
  $\expec{}{\subopt_{\sample}(\incsaga)} \leq (1 + o(1))\strisk(n)$.
\end{lemma}

As a means of achieving goal \textbf{G2} (sub-optimality on the order
of $\strisk(n_{t_d,t})$ after a drift at time $t_d$), we will show that
the empirical risk of \dsurf after a drift is ``close'' to the risk of
\aware, where \textit{close} is defined formally in terms of our
notion of risk-competitiveness in
Definition~\ref{defn:competitiveness}.

\begin{definition}
\label{defn:competitiveness}
For $c \geq 1$, an adaptive learning algorithm $A$ is said to be $c$-risk-competitive to $\aware$ at time step $t > t_d$ if\/
    $\expec{}{\subopt_{\sample_{t_d,t}}(A)} \leq c \cdot (1 + o(1))\strisk(n_{t_d,t})$,
where $t_d$ is the time step of the most recent drift and $n_{t_d,t} = |\sample_{t_d,t}|$.
\end{definition}

We will analyze the risk-competitiveness of \dsurf in a
stationary environment and after a drift. Additionally, we will
provide high probability analysis of the recovery time after a drift
(goal \textbf{G1}).

Let $t_{d_1}, t_{d_2},\ldots$ be the sequence of time steps at which a drift 
occurs. For simplicity in our analysis, in the rest of this section we assume
each drift at $t_{d_i}$ is an abrupt drift.
We assume all
loss functions $f_{\ex}$ are convex and their gradients are
$L$-Lipschitz continuous, and that the empirical risk
$\risk_{\sample}$ is $\mu$-strongly convex, where $\mu$ is the regularization hyperparameter. 
In addition, we assume the
condition number $L/\mu$ is bounded by a constant at each time step, and assume the batch size $m > L/\mu r$.
Lastly, we assume $\expec{}{\subopt_\newsample(\weight)} \leq b \expec{\newsample' \sim \sample}{\subopt_{\newsample'}(\weight)}$ for $\weight$ trained over $\sample$ using \incsaga, where $\newsample$ is a suffix of $\sample$ and $\newsample'$ is a random subsample of $\sample$ where $|\newsample'| = |\newsample|$. This last assumption bounds the bias of the training loss from the order of arrivals, given enough iterations, when no drift.

\vspace{-0.1 in}
\subsection{Stationary Environment}
\label{sec:analysis-stationary}
We will show that \dsurf is competitive to \aware in the stationary environment during the time $1 < t < t_{d_1}$ before any drifts happen. By Lemma~\ref{lemma: subopt} the expected sub-optimality of \aware and \dsurf are (respectively) bounded by $(1 + o(1))\strisk(n_{1,t})$ and $(1 + o(1))\strisk(n_{t_e,t})$, where $t_e$ is the time that the current predictive model of \dsurf was initialized. To prove \dsurf is risk-competitive to \aware, we need to show that $n_{t_e,t}$, the size of the predictive model's sample set, is close to $n_{1,t}$. To achieve this, we first in Lemma~\ref{lemma: noDrift-react-first} (and similarly in Lemma~\ref{lemma: noDrift-react-second} in the Appendix) show that the probability of entering the reactive state in a stationary environment is very small. 

\begin{lemma}
\label{lemma: noDrift-react-first}
In the stationary environment for $1 < t < t_{d_1}$, the probability of entering the reactive state because of condition~\ref{eq:enter_reactive_first} is bounded by $(\strisk(m) + (2 + o(1))\strisk(n_{t_e,t}))/\delta$, where $|\sample_{t_e,t}| = n_{t_e,t}$, $\sample_{t_e,t}$ is the stream segment that the predictive model of \dsurf is trained over, and $m$ is the batch size at each time step.
\end{lemma}

In the proof (Appendix \ref{sec:stationary}), we use Equation \ref{eq:generalization} to relate the risk over a new batch $\Ex_t$ to the risk over $\sample_{t_e,t}$, where the latter is bounded by Lemma \ref{lemma: subopt}.

Besides, if \dsurf enters the reactive state, Lemma~\ref{lemma: noDrift-switch} shows that the probability of switching to the reactive model is also very small.

\begin{lemma}
\label{lemma: noDrift-switch}
In the stationary environment for $1 < t < t_{d_1}$, if \dsurf enters the reactive state, the probability of switching to the reactive model at the end of the reactive state is bounded by $be^{-(\frac{\beta}{r})}$, where $r$ is the length of the reactive state and $\beta$ is the number of time steps that the predictive model was around before entering the reactive state, i.e. $|\sample| = \beta\times m$.
\end{lemma}

In the proof (Appendix \ref{sec:stationary}), we let $T'$ be the first $r\times m$ elements of $\sample$, and define a model $\widetilde{\weight'}$ trained over $T'$, so that the expected sub-optimality of $\widetilde{\weight'}$ and $\weight'$ over $T'$ and $T$ are the same. Bounding the probability of condition~\ref{eq:switch} follows from the convergence rate of \incsaga and Markov's inequality. 

Using the above, we can bound the size of the predictive model's sample set.
Let $\theta$ be the false positive rate for entering the reactive state in a stationary environment (bounded by Lemma \ref{lemma: noDrift-react} in Appendix \ref{sec:stationary}.)

\begin{corollary}
\label{corollary:noDrift-ave age}
With probability $1 - \epsilon$, the size of the sample set $\sample$ for the predictive model in the stable state is larger than $n_{1,t}/2$ at any time step $2r \ln \left(\frac{2(br\theta)^2}{\epsilon-br\theta} \right) \leq t < t_{d_1}$, where $n_{1,t}$ is the total number of data points that arrived until time $t$ and $r$ is the length of the reactive state.
\end{corollary}

Based on the result of Corollary~\ref{corollary:noDrift-ave age}, we will show that the predictive model of \dsurf in the stable state is $\frac{7}{4^{1-\alpha}}$-risk-competitive with \aware with probability $ 1 - \epsilon$, at any time step $2r \ln \left(\frac{2(br\theta)^2}{\epsilon-br\theta} \right) \leq t < t_{d_1}$. This is a special case of the forthcoming Lemma \ref{lemma:risk-competitive-drift} in Section \ref{sec:analysis-abrupt-drift}.

\subsection{In Presence of Abrupt Drifts}
\label{sec:analysis-abrupt-drift}

Consider the abrupt drift that occurs at time $t_{d_i}$.
Let $p$ ($p^*$)
be the probability that \dsurf enters the reactive state (switches to
the reactive model at the end of the reactive state, respectively). In
this section, we first show that w.h.p.~\dsurf has a small
recovery time (goal \textbf{G1}).

\begin{lemma}
\label{lemma:recoverytime}
With probability $1-\epsilon$, the recovery time of \dsurf is bounded by $kr + \frac{2}{p}(\ln{\frac{1}{\epsilon_1}} + k\ln{2})$, where $k < \frac{1}{p^*} + \sqrt{\frac{1-\epsilon_2}{\epsilon_2} (\frac{1-p^*}{{p^*}^2})}$ is the number of times \dsurf enters the reactive state before recovering from drift, and $\epsilon = \epsilon_1 + \epsilon_2$.
\end{lemma}

The high-level proof sketch for this Lemma (full proof in Appendix \ref{sec:abrupt-drift})
is to divide the recovery time of \dsurf into two parts: (i) time
steps spent in reactive state, and (ii) time steps spent in the stable
state before recovery.  To bound the first part, we need to bound the
number of times \dsurf enters the reactive state and multiply that by
$r$, the length of each reactive state. This can be obtained using
Cantelli's inequality. On the other hand, the second part can be
bounded by bounding the sum of $k$ independent geometric random
variables, each with distribution $\sim Ge(p)$.

We next show the risk-competitiveness of \dsurf after recovery (goal \textbf{G2}). The time period after recovery until the next drift is a stationary environment for \dsurf, in which each model is trained solely over points drawn from a single distribution, allowing for an analysis similar to the stationary environment before any drifts occurred.
\begin{lemma}
\label{lemma:risk-competitive-drift}
With probability $1 - \epsilon$, the predictive model of \dsurf in the stable state is $\frac{7}{4^{1-\alpha}}$-risk-competitive with \aware at any time step $t_{d_i} + l + \max\left(l, 2r \ln \left(\frac{2(br\theta)^2}{\epsilon_3-br\theta} \right)\right) \leq t < t_{d_{i+1}}$, where $t_{d_i}$ is the time step of the most recent drift, $l = kr + \frac{2}{p}(\ln{\frac{1}{\epsilon_1}} + k\ln{2})$ where $k < \frac{1}{p^*}+ \sqrt{\frac{1-\epsilon_2}{\epsilon_2} (\frac{1-p^*}{{p^*}^2})}$, and $\epsilon = \epsilon_1 + \epsilon_2 + \epsilon_3$.
\end{lemma}
At a high level, $\epsilon_1$ and $\epsilon_2$, respectively, capture the errors due to false negatives of entering the reactive state and switching to the new model at the end of the reactive state. Aggregating these two using Lemma~\ref{lemma:recoverytime} bounds the recovery time. On the other hand, $\epsilon_3$ captures the error in the false positive of switching models after recovery. Using a form of Corollary~\ref{corollary:noDrift-ave age} generalized to stationary environments between drifts (Appendix \ref{sec:stationary}), a lower bound on the size of the stream segment used for training the predictive model can be obtained. Finally, the expected sub-optimality of the predictive model can be bounded using Lemma~\ref{lemma: subopt}.
The full proof is in Appendix \ref{sec:abrupt-drift}.

%% file: expt.tex
\section{Experimental Results}
\label{sec:expt}

In this section, we present experimental results that (i) empirically
confirm the risk-competitiveness of \dsurf with \aware throughout a
set of experiments on datasets with drifts, and (ii) show the
effectiveness of \dsurf via comparison to two state-of-the-art
adaptive learning algorithms, the drift-detection-based method MDDM
and the ensemble method AUE.
More details on these algorithms, and additional
algorithm comparisons, are provided in Appendix~\ref{sec:expt-setup-algs}.

We use five synthetic, two semi-synthetic and three real datasets
for binary classification, chosen to include all such datasets that the authors of MDDM and AUE use in their evaluations.
Drifts
in semi-synthetic datasets are generated by rotating data points or
changing the labels of the real-world datasets that originally do not
contain any drift. We divide each dataset into equally-sized batches
that arrive over the course of the stream. More detail on the datasets
is provided in Appendix \ref{sec:expt-setup-datasets}.

In our experiments, a batch of data points arrives at each time step. We first evaluate the performance of each algorithm by measuring the misclassification rate over this batch, and then each algorithm gains access to the labeled data to update their model(s). The base learner in each algorithm is a logistic regression model trained using \incsaga. Hyperparameter settings are discussed in Appendix~\ref{sec:expt-setup-training}.
All reported results of the misclassification rates represent the median over five trials.

\begin{figure*}[t!]
\vskip -0.25 in
\centering
\begin{minipage}{.32\textwidth}
  \centering
  \includegraphics[width=\linewidth]{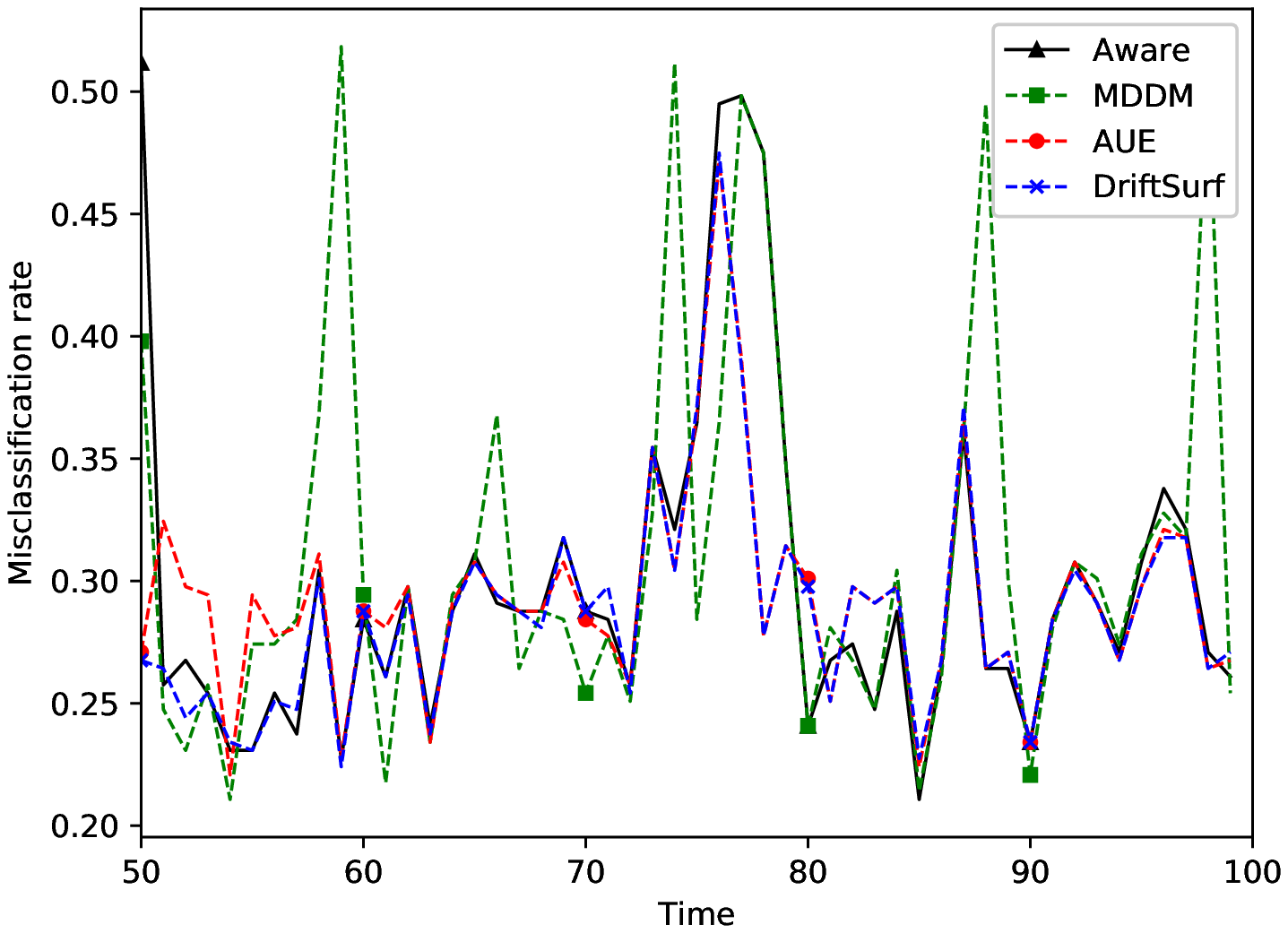}
  \caption{Misclassification rate over time for PowerSupply}
  \label{fig:power}
\end{minipage}%
\hfill
\begin{minipage}{.32\textwidth}
  \centering
  \includegraphics[width=\linewidth]{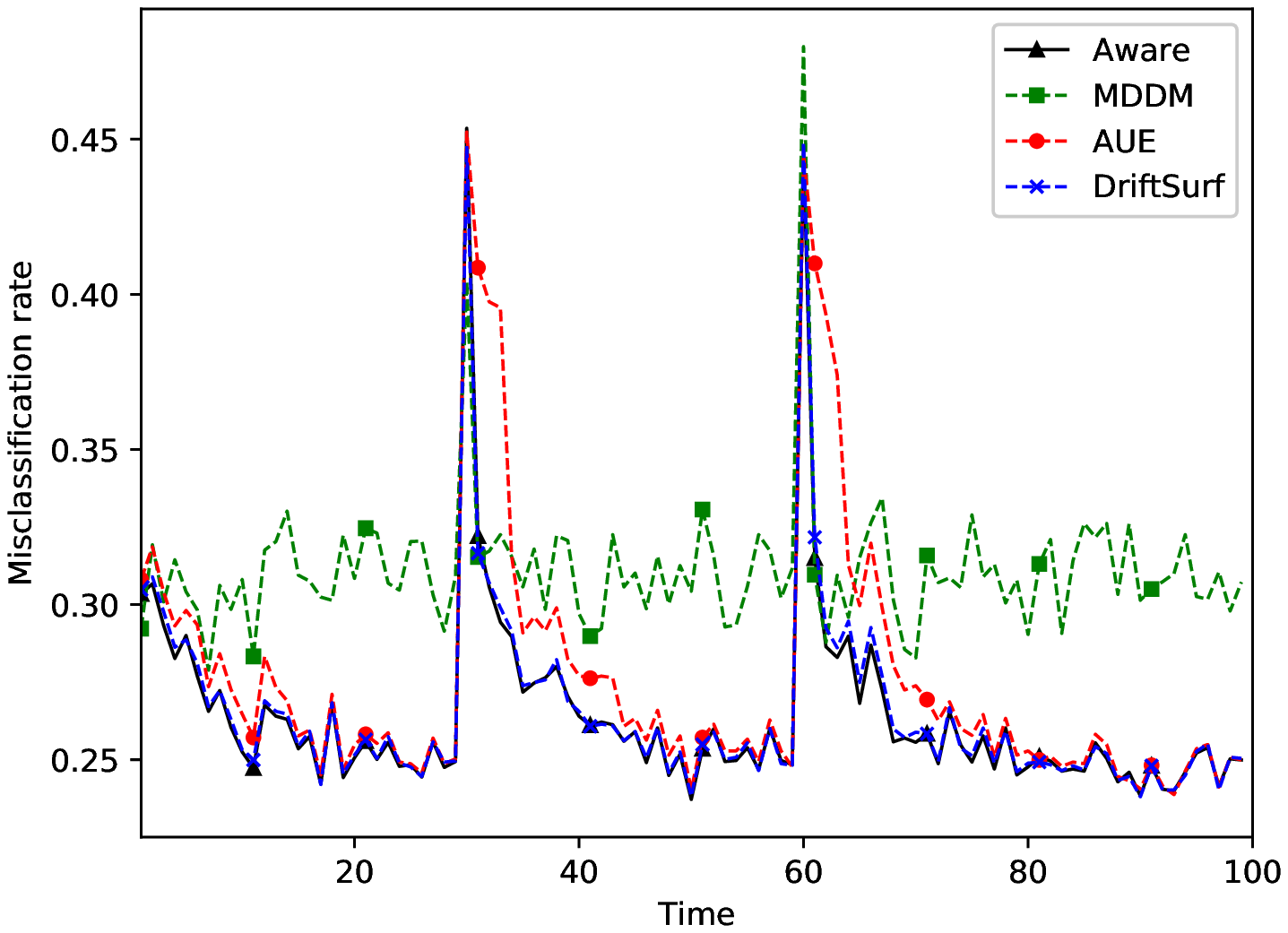}
  \caption{Misclassification rate over time for CoverType}
  \label{fig:covtype}
\end{minipage}%
\hfill
\begin{minipage}{.32\textwidth}
  \vskip 0.15 in
  \centering
  \includegraphics[width=\linewidth]{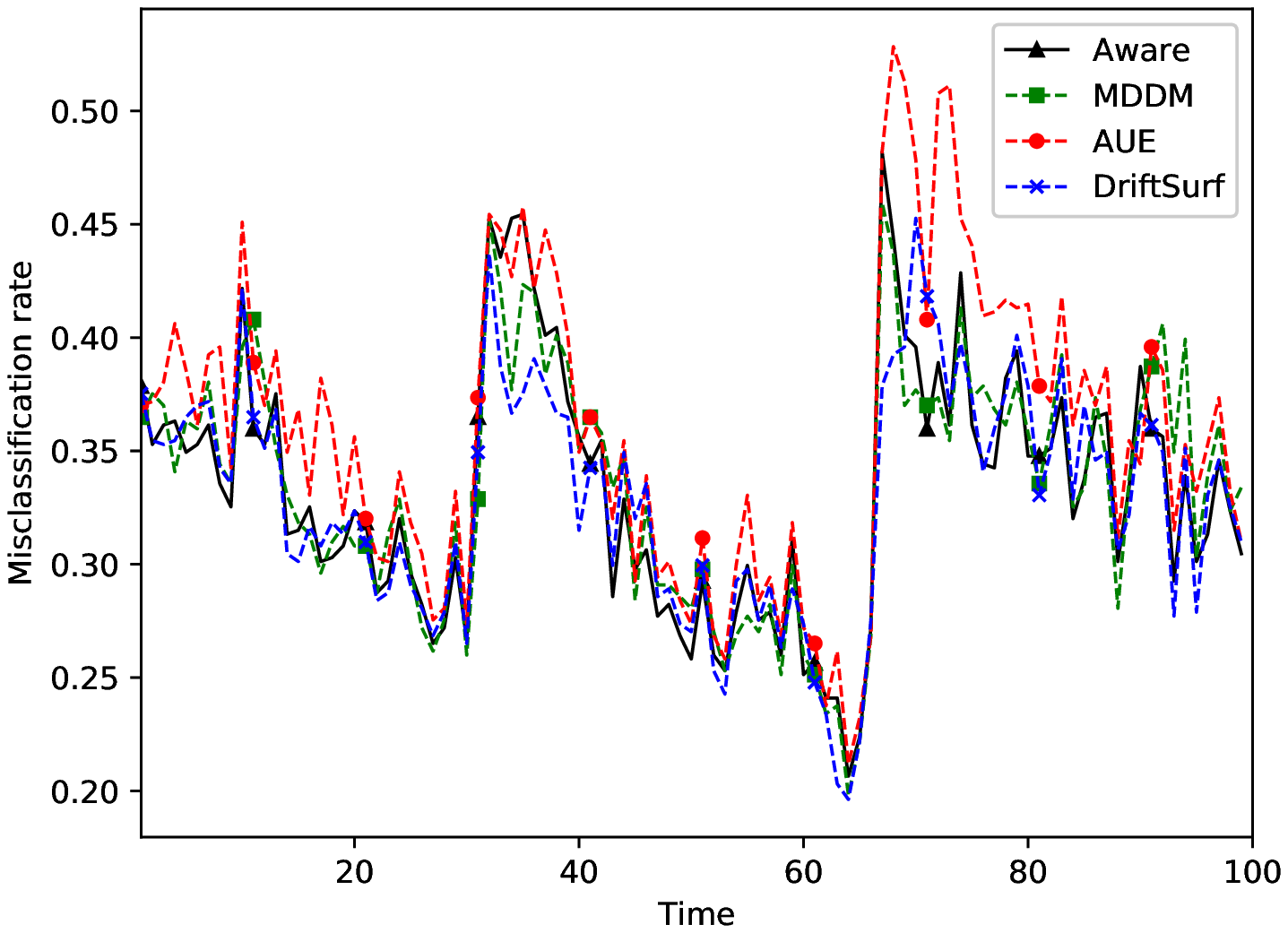}
  \caption{Misclassification rate over time for Airline ($\rho=4m$ divided among models)}
  \label{fig:air}
\end{minipage}
\end{figure*}

We present the misclassification rates at each time step in Figures \ref{fig:power} and \ref{fig:covtype} on the PowerSupply and CoverType datasets (see Appendix~\ref{sec:expt-results-equal-model} for other datasets). A drift occurs at time 76 in PowerSupply, and at times 30 and 60 in CoverType. We observe \dsurf outperforms MDDM because false positives in drift detection lead to unnecessary resetting of the predictive model in MDDM, while \dsurf avoids the performance loss by catching most false positives via the reactive state and returning to the older model. In particular, the CoverType dataset was especially problematic for MDDM, which continually signaled a drift. We also observe \dsurf adapts faster than AUE on CoverType. This is because after an abrupt drift, the predictions of \dsurf are solely from the new model, while for AUE, the predictions are a weighted average of each expert in the ensemble. Immediately after a drift, the older, inaccurate experts of AUE have reduced, but non-zero weights that negatively impact the accuracy. On both datasets, we observe the recovery time of \dsurf is within one reactive state, and confirm that \dsurf is competitive with \aware.

\begin{wraptable}{R}{0.580\textwidth}
\vspace{-0.18in}
\caption{Total average of misclassification rate}
\label{table:ave-misclassification-compact}
\vspace{-0.15in}
\begin{center}
\begin{small}
\begin{sc}
\begin{tabular}{lccccr}
\toprule
 Dataset            & \aware  & \dsurf& MDDM  & AUE   \\
\midrule
SEA0       		& 0.137 & \textbf{0.088} & \textbf{0.088} & 0.094 \\
SEA20       	& 0.264 & \textbf{0.246} & 0.293 & 0.247 \\
SEA-gradual     & 0.177 & \textbf{0.159} & 0.177 & 0.163 \\
Hyper-slow       & 0.116 & 0.117  & 0.117 & \textbf{0.112} \\
Hyper-fast         & 0.191 & 0.174 & \textbf{0.162} & 0.179 \\
SINE1               & 0.171 & 0.197 & \textbf{0.179} & 0.211 \\
Mixed               	& 0.192 & \textbf{0.204} & \textbf{0.204} & 0.211 \\
Circles             	& 0.368 & \textbf{0.371} & 0.376 & 0.380 \\
RCV               	& 0.121 & 0.135 & \textbf{0.127} & 0.167\\
CoverType           & 0.267 & \textbf{0.268} & 0.313 & 0.278 \\
Airline         & 0.338 & \textbf{0.332} & 0.348 & 0.333 \\
Electricity         & 0.315 & 0.306 & 0.341 & \textbf{0.303} \\
PowerSupply               & 0.309 & 0.300 & 0.323 & \textbf{0.299} \\
\bottomrule
\end{tabular}
\end{sc}
\end{small}
\end{center}
\vspace{-0.1in}
\end{wraptable}
Table~\ref{table:ave-misclassification-compact}
summarizes the results for all the datasets in terms of the total average of the misclassification rate over time. 
In the first two rows, we observe the stability of \dsurf in the presence of 20\% additive noise in the synthetic SEA dataset, again demonstrating the benefit of the reactive state while MDDM's performance suffers due to the increased false positives.
We observe that on a majority of the datasets in
Table~\ref{table:ave-misclassification-compact},
\dsurf is the best performer. For some datasets (Electricity, Hyper-Slow) AUE outperforms \dsurf. A factor is the different computational power (number of gradient computations per time step) used by each algorithm. AUE maintains an ensemble of ten experts, while \dsurf maintains just two, and so AUE uses five times the computation of \dsurf. To account for the varying computational efficiency of each algorithm, we did another experiment where the available computational power for each algorithm is divided equally among all of its models, shown in Figure \ref{fig:air} for the Airline dataset, with more results in Appendix \ref{sec:expt-results-equal-alg}. After normalizing for equal computational power, we observe \dsurf has better accuracy and recovers faster after drift compared to AUE.

Appendices~\ref{sec:expt-results-oblivious}--\ref{sec:expt-results-sgd}
contain additional experimental results. In Appendix
\ref{sec:expt-results-oblivious}, we report the results for
single-pass SGD and an oblivious algorithm (\incsaga with no
adaptation to drift), which are generally worse across each
dataset. One exception is that the oblivious algorithm has the best
accuracy on the Electricity dataset because the drift does not warrant
training a new model from scratch.  Appendix
\ref{sec:expt-results-no-greedy} studies the impact of \dsurf's design
choice of using greedy prediction during the reactive state, showing
that it performs similarly or better than waiting until the end of the
reactive state before deciding whether to transition to a new model.
Finally, Appendix \ref{sec:expt-results-sgd} includes results for each
algorithm when SGD is used as the update process instead of
\incsaga. We observe that using SGD results in lower accuracy for each
algorithm, and also that, relatively, AUE gains an edge because its
ensemble of ten experts mitigates the higher variance updates of SGD.

%% file: conclusion.tex
\section{Conclusion}

We presented \dsurf, an adaptive algorithm for learning from streaming
data that contains concept drifts. Our risk-competitive theoretical
analysis showed that \dsurf has high accuracy competitive with \aware
both in a stationary environment and in the presence of abrupt
drifts. Our experimental results confirmed our theoretical analysis
and also showed high accuracy in the presence of either abrupt or gradual
drift that generally outperforms state-of-the-art algorithms MDDM and
AUE. Furthermore, \dsurf maintains just two models while achieving
high accuracy, and therefore its computational efficiency is
significantly better than an ensemble method like AUE.

%% file: appendix-pseudocode.tex
\section{Pseudocode of \incsagaH}

As shown in Figure~\ref{algo:SR}, \dsurf calls a function Update($\weight, \sample, \Ex_t$) that takes a model $\weight$, a sample set $\sample$, and a set of training points $\Ex_t$.
We let $\rho$ be the computational power available at each time step; for an SGD-based algorithm, $\rho$ is the number of gradients that can be computed in a time step.
The Update function performs $\rho$ updates to $\weight$ and returns the resulting model; as a side effect, it also updates $\sample$.

In this paper, we primarily use \incsaga~\cite{jothimurugesan2018variance}, shown in Algorithm~\ref{algo:strsaga}, for our Update function.
\incsaga differs from SGD (Algorithm~\ref{algo:update}) in that (i) it uses variance-reduced update steps that result in faster convergence, and (ii) it handles streaming data that do not arrive at a steady rate by controlling the rate at which its sample set grows. (In this paper, we only consider data that arrive at a fixed rate at each time step, but by using \incsaga, the results can be readily extended to Poisson and other arrival distributions.) In  \incsaga, data points are not sampled from the entire available stream segment, but instead from a separately maintained sample set. Newly arriving data are first added to a buffer (called $\buffer$), and then points are moved from $\buffer$ to the sample set at a controlled rate ``to ensure that the optimization error on the subset that has been trained is balanced with the statistical error of the effective sample size'' \cite{jothimurugesan2018variance}. The implementation of \incsaga we use in this paper uses the ``alternating schedule'' in its sampling.

\begin{algorithm}
    \caption{Update($\weight, \sample, \Ex_t$): Process of updating parameters $\weight$ using SGD, given sample set $\sample$ and newly arrived data points $\Ex_t$}
    \label{algo:update}
 
 \tcp{$\rho$ is the computational power and determines the number of update steps that can be performed.}
 
 \tcp{$\eta$ is the learning rate}
 
Add $\Ex_t$ to $\sample$ \;

\For{ $j=1$ to $\rho$}
{
    Sample a point $p$ uniformly from $\sample$\;
    
    $g \gets \nabla f_p(\weight)$ \tcp{$f_p$ is the loss function at $p$}
    
    $\weight \gets \weight- \eta \cdot g $
}

\Return $\weight$

\end{algorithm}

\begin{algorithm}
    \caption{Update($\weight, \sample, \Ex_t$): Process of updating parameters $\weight$ using \incsaga, given sample set $\sample$ and newly arrived data points $\Ex_t$}
    \label{algo:strsaga}
 
    \tcp{ $\rho$ is the computational power and determines the number of update steps that can be performed}
   
    \tcp{ $\eta$ is the learning rate}
    
    Add $\Ex_t$ to $\buffer$ \tcp{$\buffer$ is the set of training points not added to $\sample$ yet}

    \For{ $j=1$ to $\rho$ }
    {

        \If{$\buffer$ is non-empty \& $j$ is even}
        {
            Move a single point, $p$, from $\buffer$ to $\sample$
            
            $\alpha(p) \gets 0$ \tcp{$\alpha(p)$ is the prior gradient of $p$, initialized to $0$}
        
        }    
        \Else
        {    
            Sample a point $p$ uniformly from $\sample$
            
        }
        
       $A \gets \sum_{x \in \sample} \alpha(x)/|\sample|$ \tcp{$A$ is the average of all gradients and can be maintained incrementally}
        
        $g \gets \nabla f_p(\weight) $ \tcp{$f_p$ is the loss function at $p$}

        $ \weight \gets \weight - \eta (g - \alpha(p) + A) $ 
        
        $\alpha(p) \gets g$
        
    }

     \Return $\weight$

\end{algorithm}

The time complexity of Algorithms \ref{algo:update} and \ref{algo:strsaga} is on the order of $\rho$ times the cost of a gradient computation with respect to a single data point. Each gradient computation is typically $O(d)$ for model parameter dimension $d$.
The space complexity of Algorithm \ref{algo:update} is $O(D(|\sample| + |\Ex_t|) + d)$ to store the samples and model parameters, where $D$ is the dimension of the data. The space complexity of Algorithm \ref{algo:strsaga} incurs an additive $O(d(|\sample| + |\Ex_t|))$ to store the prior gradients $\alpha(p)$ for each data point (for linear models, this cost is reduced to $O(|\sample| + |\Ex_t|)$ since each gradient is a scalar multiple of the corresponding data point).

%% file: appendix-proofs.tex
\section{Proofs from the Analysis of \dsurfH}

This section contains proof details from the analysis of \dsurf (Section~\ref{sec:analysis}). As noted in Section~\ref{sec:analysis}, we make the following assumptions throughout our analysis.  First, each drift that occurs is an abrupt drift (i.e., the distribution changes across a single time step)---this is solely for the analysis, as our algorithm more generally applies to gradual drfits, as our experimental results show.
Second, we assume that $\strisk(n) = hn^{-\alpha}$, for a constant $h$ and $1/2 \leq \alpha \leq 1$, is an upper bound on the statistical error over a set of data points of size $n$. We also assume all loss functions $f_x$ are convex and their gradients are $L$-Lipschitz continuous, and that $\risk_\sample$ is $\mu$-strongly convex for the set of training samples $\sample$. In addition, we assume that the condition number $L/\mu$ is bounded by a constant at each time, and that $m=|\Ex_t|$, the number of points arriving at each time $t$, is bounded $m > L/\mu r$, We assume that $\rho = 2m$, where $\rho$ denotes the number of gradients that can be computed at each time step $t$.
Finally, we assume $\expec{}{\subopt_\newsample(\weight)} \leq b \expec{\newsample' \sim \sample}{\subopt_{\newsample'}(\weight)}$ for $\weight$ trained over $\sample$ using \incsaga, where $\sample$ is all drawn from an identical distribution, $\newsample$ is a suffix of $\sample$, and $\newsample'$ is a random subsample of $\sample$ where $|\newsample'| = |\newsample|$. This last assumption bounds (by a factor $b$) the bias of the training loss from the order of arrivals, given enough iterations, when no drift.

Table~\ref{tbl:notation} summarizes the notation used in this section.
\begin{table}[t]
  \caption{Summary of notation used in the analysis
  \label{tbl:notation}}
\begin{center}
  \begin{tabular}{|c|l|}\hline
    $\Ex_t$ & Data points arriving at time step $t$ \\
    $m$ & = $|\Ex_t|$, the number of points arriving at each $t$ \\
    $r$ & length of the reactive state (in time steps) \\
    $\rho$ & the number of gradients computed at each time step \\
    $b$ & bias of \incsaga in sub-optimality over a suffix \\
    $\alpha$ & the exponent in the statistical error bound $\strisk(n) = h n^{-\alpha}$ \\
    $\theta$ & false positive rate a recovered \dsurf{} enters the reactive state in a stationary environment \\
    $p$ & probability \dsurf{} enters the reactive state after a given drift \\
    $p^*$ & probability \dsurf{} switches to the reactive model at end of a given reactive state \\ \hline
  \end{tabular}
  \end{center}
  \end{table}

In Section~\ref{sec:analysis}, we defined \aware to be an adaptive learning algorithm with oracle knowledge of when drifts occur. Through Lemma~\ref{lemma:risk-competitive-drift}, we showed that under certain conditions \dsurf is risk-competitive to \aware. Here, we state the consequence with regards to the total error. As stated in Section \ref{prelim}, the total error of an algorithm $A$ over the stream segment $\sample$ is the sum of the statistical and optimization errors; under uniform convergence bounds, the total error is bounded by $\strisk(|\sample|) + \expec{}{\subopt_{\sample}(A)}$ \cite{bousquet2008tradeoffs}. Empirical risk minimization (\erm), which is a process with no limit on the computational power, over a stream segment $\sample$ yields a model with total error equal to the statistical error. When \dsurf is risk-competitive with \aware, then the total error of \dsurf can be bounded relative to the error of \erm by the following lemma.

\begin{lemma}
\label{lemma:error-erm}
Suppose the last drift occurred at time step $t_d$. If \dsurf is $c$-risk-competitive to \aware at time $t > t_d$, then the total error of \dsurf is at most a $(c + 1 +o(1))$ factor of the error bound of \erm, $\strisk(\sample_{t_d,t})$.
\end{lemma}

\begin{proof}
By the definition of risk-competitiveness to \aware, $\expec{}{\subopt_{\sample_{t_d,t}}(\dsurf)} \leq c(1 + o(1)) \strisk(n_{t_d, t})$. Adding the statistical error, the total error is at most $(c(1 + o(1)) + 1)\strisk(n_{t_d, t})$.
\end{proof}
Note that although the \erm error bound, $\strisk()$, is only an upper bound, it is usually considered to be a tight bound~\cite{bousquet2008tradeoffs}.

The theoretical analysis in this paper establishes when \dsurf is risk-competitive with \aware following abrupt drifts. To give further motivation for the notion of risk-competitiveness, we briefly discuss related work on online learning algorithms built upon \emph{single-pass} online gradient descent. In both \cite{besbes2015non} and \cite{yi2016tracking}, the authors assume an adversarial sequence of loss functions  constrained to satisfy a variational budget, and give algorithms based on online gradient descent, equipped with periodic resets and/or decaying step sizes that are parameterized by the adversary's budget, that are shown to have optimal dynamic regret bounds. In \cite{wang2018minimizing}, the authors instead assume no constraint on the adversary, and give an ensemble algorithm of experts that are each trained with online gradient descent (with a modification to use a more efficiently computable surrogate loss) that is shown to have an optimal adaptive regret bound.

In contrast, in this paper we study the streaming data setting where previous data points can be stored and revisited in order to achieve a better sub-optimality beyond what can be attained in a single-pass. This additional power is evidently useful when explicitly stationary periods of the stream exist for resampling from. Thus, we restrict the analytical consideration to a simple and practical case of abrupt drifts at times $t_{d_i}$ that yield stationary stream segments $\sample_{t_{d_i}, t_{d_{i+1}}}$. In the course of our risk-competitive analysis, we estimate how large a subset of such a stream segment that \dsurf can identify to be used for empirical risk minimization, where larger subsets correspond to better generalization error over the relevant distribution.

In the remainder of this section we complete the proofs for the results in Section~\ref{sec:analysis} that establish the conditions under which \dsurf is risk-competitive with \aware  both in a stationary environment and in the presence of abrupt drifts.

\subsection{In a Stationary Environment}
\label{sec:stationary}

In Section \ref{sec:analysis-stationary}, we considered only the stationary environment during the time $1 < t < t_{d_1}$ before any drifts. In this section, we generalize the results to the stationary environment for any time $t_{d_i} + r_i \leq t < t_{d_{i+1}}$, where $r_i$ is the recovery time for the drift at $t_{d_i}$. We refer to such a time period as a \textit{recovered state}, in which each model of \dsurf is trained solely over points from the newest distribution.

\begin{lemma}
\label{lemma: noDrift-react-first-general}
(Generalized statement of Lemma \ref{lemma: noDrift-react-first}.) 
In a recovered state, at any time step $t$ the probability of entering the reactive state because of condition~\ref{eq:enter_reactive_first} is bounded by $(\strisk(m) + (2 + o(1))\strisk(n_{t_e,t}))/\delta$, where $|\sample_{t_e,t}| = n_{t_e,t}$, $\sample_{t_e,t}$ is the stream segment that the predictive model of \dsurf is trained over, and $m$ is the batch size at each time step.
\end{lemma}
\begin{proof}
In a recovered state, each point in $\Ex_t$ and $\sample_{t_e,t}$ is drawn from the same distribution $I$. Using Markov's inequality we have $\Pr[\risk_{\Ex_t}(\weight_{t-1}) - \risk_b > \delta]$
\begin{align*}
    &\leq \frac{1}{\delta}\left(\expec{}{\risk_{\Ex_t}(\weight_{t-1}) - \risk_b}\right) \\
    &= \frac{1}{\delta}\left(\expec{}{\risk_{\Ex_t}(\weight_{t-1}) -\risk_I(\weight_{t-1}) + \risk_I(\weight_{t-1}) - \risk_{\sample_{t_e,t}}(\weight_{t-1}) + \risk_{\sample_{t_e,t}}(\weight_{t-1}) - \risk_b}\right)\\
    &\leq \frac{1}{\delta}\left(\expec{}{|\risk_{\Ex_t}(\weight_{t-1}) -\risk_I(\weight_{t-1})|} + \expec{}{|\risk_I(\weight_{t-1}) - \risk_{\sample_{t_e,t}}(\weight_{t-1})|} + \expec{}{\risk_{\sample_{t_e,t}}(\weight_{t-1}) - \risk_b}\right)\\
    &\leq \frac{1}{\delta}\left(\frac{1}{2}\strisk(m) + \frac{1}{2}\strisk(n_{t_e,t}) + \expec{}{\risk_{\sample_{t_e,t}}(\weight_{t-1}) - \risk_b}\right)
\end{align*}
where the last inequality holds by Equation \ref{eq:generalization}. On the other hand, $\risk_b \geq \risk_{\Ex_b}(\weight_b^*)$, where $\Ex_b$ is the batch corresponding to $\risk_b$ and $\weight_b^*$ minimizes its empirical risk. Applying Equation \ref{eq:generalization} twice more, we have $\expec{}{\risk_{\sample_{t_e,t}}(\weight_{t-1}) - \risk_b}$
\begin{align*}
    &\leq \expec{}{\risk_{\sample_{t_e,t}}(\weight_{t-1}) - \risk_{\Ex_b}(\weight_b^*)}\\
    &= \expec{}{\risk_{\sample_{t_e,t}}(\weight_{t-1}) - \risk_{\sample_{t_e,t}}(\weight_b^*) + \risk_{\sample_{t_e,t}}(\weight_b^*) - \risk_I(\weight_b^*) + \risk_I(\weight_b^*)- \risk_{\Ex_b}(\weight_b^*)}\\
    &\leq \expec{}{\risk_{\sample_{t_e,t}}(\weight_{t-1}) - \risk_{\sample_{t_e,t}}(\weight_b^*)} + \expec{}{|\risk_{\sample_{t_e,t}}(\weight_b^*) - \risk_I(\weight_b^*)|} + \expec{}{|\risk_I(\weight_b^*)- \risk_{\Ex_b}(\weight_b^*)|}\\
    &\leq \expec{}{\risk_{\sample_{t_e,t}}(\weight_{t-1}) - \risk_{\sample_{t_e,t}}(\weight_b^*)} + \frac{1}{2}\strisk(m) + \frac{1}{2}\strisk(n_{t_e,t}).
\end{align*}
We know $\risk_{\sample_{t_e,t}}(\weight_b^*) \geq \risk_{\sample_{t_e,t}}^*$. Therefore, 
\begin{align*}
    \Pr[\risk_{\Ex_t}(\weight_{t-1}) - \risk_b > \delta] &\leq \frac{1}{\delta}\left(\strisk(m) + \strisk(n_{t_e,t}) + \expec{}{\risk_{\sample_{t_e,t}}(\weight_{t-1}) - \risk_{\sample_{t_e,t}}(\weight_b^*)}\right)\\
    &\leq \frac{1}{\delta}\left(\strisk(m) + \strisk(n_{t_e,t}) + \expec{}{\risk_{\sample_{t_e,t}}(\weight_{t-1}) - \risk_{\sample_{t_e,t}}^*}\right) \\
    &\leq \frac{1}{\delta}\left(\strisk(m) + \strisk(n_{t_e,t}) + (1+o(1))\strisk(n_{t_e,t})\right)
\end{align*}
where the last inequality holds following Lemma~\ref{lemma: subopt}. 
\end{proof}

Before moving on to bound the probability of entering the reactive state due to condition \ref{eq:enter_reactive_second} in Lemma \ref{lemma: noDrift-react-second}, we will need the following fact. The update process 
\incsaga is a stochastic optimization method that provides the following property at each iteration $i$:
\begin{align}
\label{eq:switch-cost}
    \expec{}{\risk_{\sample}(\weight_i) - \risk_{\sample}^*} \leq \min \begin{cases}
\rho_n[\risk_{\sample}(\weight_{i-1}) - \risk_{\sample}^*]\\
\min\limits_{k < n} [(\risk_{T'}(\weight_{i}) - \risk_{T'}^*) + \dfrac{n - k}{n} \strisk(k)]
\end{cases}  
\end{align}
where $\rho_n = 1 - \min(\frac{1}{n}, \frac{\mu}{L})$, $|\sample| = n$, and $T'\subset \sample$ is of size $k$ \cite{daneshmand2016starting, jothimurugesan2018variance}.

\begin{lemma}
\label{lemma: noDrift-react-second}
In a recovered state, at any time step $t$ the probability of entering the reactive state because of condition~\ref{eq:enter_reactive_second} is bounded by $(\strisk(m) + (2 + o(1))\strisk(n_{t_e,t}) + (2 + o(1))\strisk(n_{t_s,t}))/\delta'$, where $|\sample_{t_e,t}| = n_{t_e,t}, |\sample_{t_s,t}| = n_{t_s,t}$, and $\sample_{t_e,t}$ and $\sample_{t_s,t}$ are the stream segments that the predictive model and stable model of \dsurf are trained over.
\end{lemma}

\begin{proof}
Applying Markov's inequality,
\begin{align*}
\Pr[\risk_{\Ex_t}(\weight_{t-1}) - \risk_{\Ex_t}(\weight_{t-1}'') > \delta'] 
&\leq \Pr[|\risk_{\Ex_t}(\weight_{t-1}) - \risk_{\Ex_t}(\weight_{t-1}'')| > \delta'] \\
&\leq \frac{1}{\delta'}\left(\expec{}{|\risk_{\Ex_t}(\weight_{t-1}) - \risk_{\Ex_t}(\weight_{t-1}'')|}\right).
\end{align*}
Similar to the proof of Lemma \ref{lemma: noDrift-react-first-general}, we denote $I$ to be the identical distribution from which $\Ex_t$ and $\sample_{t_e,t}$ are drawn. Adding and subtracting the terms $\risk(\weight_{t-1}), \risk_{\sample_{t_e,t}}(\weight_{t-1}), \risk_{\sample_{t_e,t}}(\weight_{t-1}''),$ and $\risk(\weight_{t-1}'')$, and applying Equation \ref{eq:generalization} four times, we have
\begin{align*} 
\expec{}{|\risk_{\Ex_t}(\weight_{t-1}) - \risk_{\Ex_t}(\weight_{t-1}'')|} \leq \strisk(m) + \strisk(n_{t_e,t}) + \expec{}{|\risk_{\sample_{t_e,t}}(\weight_{t-1}) - \risk_{\sample_{t_e,t}}(\weight_{t-1}'')|}.
\end{align*}
To bound the last term, we use the property in Equation \ref{eq:switch-cost} to relate the sub-optimality over $\sample_{t_e,t}$ to the sub-optimality over $\sample_{t_s,t}$, and then use Lemma \ref{lemma: subopt} to bound the sub-optimality.
\begin{align*}
&\expec{}{|\risk_{\sample_{t_e,t}}(\weight_{t-1}) - \risk_{\sample_{t_e,t}}(\weight_{t-1}'')|} \\
&\leq \expec{}{|\risk_{\sample_{t_e,t}}(\weight_{t-1}) - \risk_{\sample_{t_e,t}}^*|} + \expec{}{|\risk_{\sample_{t_e,t}}^* - \risk_{\sample_{t_e,t}}(\weight_{t-1}'')|} \\
&\leq \expec{}{\risk_{\sample_{t_e,t}}(\weight_{t-1}) - \risk_{\sample_{t_e,t}}^*} + \expec{}{\risk_{\sample_{t_s,t}}(\weight_{t-1}'') - \risk_{\sample_{t_s,t}}^*} + \frac{n_{t_e,t} - n_{t_s,t}}{n_{t_e,t}} \strisk(n_{t_s,t}). \\
&\leq (1 + o(1))\strisk(n_{t_e,t}) + (2 + o(1))\strisk(n_{t_s,t}).
\end{align*}

\end{proof}

\begin{lemma}
\label{lemma: noDrift-react}
In a recovered state, the probability $\theta$ of entering the reactive state at any time step $t$ is bounded by $(\strisk(m) + (2 + o(1))\strisk(n_{t_e,t}))/\delta + (\strisk(m) + (2 + o(1))\strisk(n_{t_e,t}) + (2 + o(1))\strisk(n_{t_s,t}))/\delta'$, where $|\sample_{t_e,t}| = n_{t_e,t}, |\sample_{t_s,t}| = n_{t_s,t}$, $\sample_{t_e,t}$ and $\sample_{t_s,t}$ are the stream segments that the predictive model and stable model of \dsurf are trained over, and $m$ is the batch size at each time step.
\end{lemma}

\begin{proof}
Using Lemma~\ref{lemma: noDrift-react-first} and Lemma~\ref{lemma: noDrift-react-second}.
\end{proof}

From here on, for simplicity of analysis, we use $\theta$ as a constant upper bound independent of $t$.

In Lemma \ref{lemma:risk-competitive-drift}, we establish the risk-competitiveness of the predictive model in the stable state. The following corollary analyzes how often \dsurf is in the stable state.
\begin{corollary}
\label{corollary:fraction-reactive-state}
In a recovered state, the limit of the expected fraction of time spent in the stable state is bounded below by $\frac{1}{1 + r\theta}$ as $t \rightarrow \infty$.
\end{corollary}

\begin{proof}
After $t$ time steps, we asymptotically bound the expected number of transitions to the reactive state via the hitting time in a Markov chain. The Markov chain has states $1, 2, \dotsc, t, \dotsc, t+r$, with initial state 1 and absorbing states $t, \dotsc, t+r$, where each state corresponds to a time step in the stable state. At each state $1 \leq i < t$, $i \rightarrow i + r + 1$ with probability at most $\theta$, and $ i \rightarrow i+1$ with probability at least $1 - \theta$.

The hitting time $h$ can be decomposed into the expected number of wins and losses $h = w + \ell$, where a win corresponds to a transition to the reactive state, and $t \leq w(r+1) + (h-w) < t + r$. From $h$ we can determine the number of wins $w \sim \frac{t-h}{r}$. Furthermore, for each win excluding the last one, $r$ time steps are spent in the reactive state. Thus, the fraction of time spent in the reactive state after $t$ time steps is asymptotically $wr/t \sim 1 - h/t$.

The hitting time can be found solving the following linear recurrence with constant coefficients, where $h = x_1$.
\begin{align*}
x_i &= (1-\theta)x_{i+1} + \theta x_{i+r+1} + 1 \\
x_{t}, x_{t+1}, \cdots, x_{t+r} &= 0.
\end{align*}
Solving the recurrence, we find $h \sim t/(1 + r\theta)$.

\end{proof}

\begin{lemma}
\label{lemma: noDrift-switch-general}
(Generalized statement of Lemma \ref{lemma: noDrift-switch}.)
In a recovered state, if \dsurf enters the reactive state, the probability of switching to the reactive model at the end of the reactive state is bounded by $be^{-(\frac{\beta}{r})}$, where $r$ is the length of the reactive state and $\beta$ is the number of time steps that the predictive model was around before entering the reactive state, i.e. $|\sample| = \beta\times m$.
\end{lemma}

\begin{proof}

The proof is through a series of reductions by rearranging of the order of the random sample of $\sample$ drawn from a single distribution in a recovered state. Recall $T$ be the set of samples that arrived during the reactive state. Let $\weight'_+$ be the reactive model with one additional iteration over $T$. Let $\widetilde{T}$ be a random sample of $mr$ elements from $\sample$, and let $\widetilde{\weight}$ be a model trained over $\widetilde{T}$, with one additional iteration.

By Equation \ref{eq:switch-cost}, $\subopt_T (\weight') \geq \frac{1}{\rho_{mr}} \expec{}{\subopt_T (\weight'_+)}$. By Markov's inequality, 
\begin{align*}
\prob{\risk_{T}(\weight) > \risk_{T}(\weight')} &= \prob{ \subopt_{T}(\weight) > \subopt_{T}(\weight') } \\
&\leq \prob{ \subopt_{T}(\weight) > \frac{1}{\rho_{mr}} \expec{}{\subopt_T (\weight'_+)} } \\
&\leq \rho_{mr} \frac{\expec{}{\subopt_T(\weight)}}{\expec{}{\subopt_T (\weight'_+)}} \\
&= \rho_{mr} \frac{\expec{}{\subopt_T(\weight)}}{\expec{}{\subopt_{\widetilde{T}} (\widetilde{\weight})}} \\
&\leq b \rho_{mr} \frac{\expec{}{\subopt_{\widetilde{T}}(\weight)}}{\expec{}{\subopt_{\widetilde{T}} (\widetilde{\weight})}} 
\end{align*}
where the last line follows from the assumption of bounded bias. Applying Equation \ref{eq:switch-cost}, $\expec{}{\subopt_{\widetilde{T}}(\weight)} \leq \rho_{mr}^{2m\beta - 1} \expec{}{\subopt_{\widetilde{T}}(\widetilde{\weight})}.$
Under the assumption that $m > L/(\mu r)$, and using $(1-\frac{1}{mr})^{mr} \leq 1/e$, the probability is bounded as $b \exp(-\beta/r)$.
\end{proof}

We can now prove Corollary \ref{corollary:noDrift-ave age-general}.
\begin{corollary}
\label{corollary:noDrift-ave age-general}
(Generalized statement of Corollary \ref{corollary:noDrift-ave age}.)
Let $t_r$ be the time step \dsurf enters a recovered state after a drift at time $t_{d_i}$. With probability $1 - \epsilon$, the size of the sample set $\sample$ for the predictive model in the stable state is larger than $n_{t_r,t}/2$ at any time step $t_r + 2r \ln \left(\frac{2(br\theta)^2}{\epsilon-br\theta} \right) \leq t < t_{d_{i+1}}$, where $n_{t_r,t}$ is the total number of data points that arrived from time $t_r$ until time $t$ and $r$ is the length of the reactive state.
\end{corollary}

\begin{proof}
Let $t' = t - t_r$. For $t' < t_{d_{i+1}}$, \dsurf is in a recovered state and we can apply Lemma \ref{lemma: noDrift-switch-general} as follows. $\Pr[|\sample|>\frac{mt'}{k}]$ 
\begin{align*}
    &\geq \Pr[\text{not switching in the last $t'/k$ time steps}]\\
    &\geq \Pr[\text{not switching in the last $t'/k - i$ time steps}| \beta_{(\frac{k-1}{k})t'+i} > i] \times \Pr[\beta_{(\frac{k-1}{k})t'+i} > i]\\
    &= \Pr[\beta_{(\frac{k-1}{k})t'+i} > i] \times \prod_{j \in (t' - t'/k + i, t']} \Pr[\text{not switching at time step $j$} | \beta_j > j - (\frac{k-1}{k})t'] \\
    &\geq \Pr[\beta_{(\frac{k-1}{k})t'+i} > i] \times \left(1 - \sum_{j\in(t'-t'/k+i,t']} \Pr[\text{switching at time $j$}|\beta_j > j - (\frac{k-1}{k})t']\right)
\end{align*}
where $\beta_j$ is the age of the expert at time step $j$ and last inequality holds following Weierstrass inequality. Let $ \gamma_i = \Pr[\beta_{(\frac{k-1}{k})t'+i} \geq i] $, thus, 
\begin{align*}
   \Pr[|\sample|>\frac{mt'}{k}] &\geq \gamma_i \times \left(1 - \sum_{j\in(t'-t'/k+i,t']} \Pr[\text{switching at time $j$}|\beta_j > j - (\frac{k-1}{k})t']\right)\\
    &\geq \gamma_i \times \left(1 -\sum_{\beta_j=i+1}^{\frac{t'}{k}} b\theta \exp\left(-\frac{\beta_j}{r}\right)\right) = \gamma_i \times \left(1 - b\theta \sum_{\beta_j=i+1}^{\frac{t'}{k}} \exp\left(-\frac{\beta_j}{r}\right)\right)
\end{align*}
The sum is the lower Riemann sum of the decreasing function $\exp(-\beta_j/r)$ over the partitioning of the interval $I$ = ($i,t'/k$] into unit subintervals. The lower Riemann sum is upper bounded by the area under the curve of $\exp(-\beta_j/r)$ over $I$. Continuing,
\begin{align*}
\Pr[|\sample|>\frac{mt'}{k}] &\geq \gamma_i \times \left(1 - b\theta \int_{i}^{\frac{t'}{k}} \exp(-\beta_j/r) \mathrm{~d}\beta_j \right) 
\end{align*}
On the other hand, we know that at any time step $j > r$ we have $\beta_j > r$. Thus,
\begin{align*}
    \gamma_i &= \Pr[\text{not switching between $(\frac{k-1}{k})t' + r$ and $(\frac{k-1}{k})t'+i$ time steps}] \\
    &= \prod_{j \in ((\frac{k-1}{k})t' + r, (\frac{k-1}{k})t'+i]} \Pr[\text{not switching at time step $j$} | \beta_j > j - (\frac{k-1}{k})t']\\
    &\geq 1 - \sum_{j \in (t'-t'/k + r, t'-t'/k+i]} \Pr[\text{switching at time $j$}| \beta_j > j - (\frac{k-1}{k})t']\\
    &\geq \left(1 - b \theta \int_{r}^{i} \exp(-\beta_j/r) \mathrm{~d}\beta_j \right)
\end{align*}
Therefore,
\begin{align*}
    \Pr[|\sample|>\frac{mt'}{k}] &\geq \left(1 - br\theta - (br\theta)^2\exp(-\frac{t'}{kr}) - (br\theta)^2\exp(-\frac{2i}{r})\right)
\end{align*}
For the choice $k = 2$ and $i = \frac{t'}{2k}$, we get $\Pr[|\sample|>\frac{mt'}{2}] \geq  (1 - \epsilon)$ provided $t' \geq 2r \ln \left(\frac{2(br\theta)^2}{\epsilon-br\theta} \right)$.
\end{proof}

\subsection{In Presence of Abrupt Drifts}
\label{sec:abrupt-drift}

For the case of abrupt drift, we first bound the recovery time for \dsurf through Lemmas \ref{lemma:NumOfReactiveState} and \ref{lemma:recoverytime}, and then establish risk-competitiveness after recovery in Lemma \ref{lemma:risk-competitive-drift}. To simplify the analysis, we use two parameters, $p$ and $p^*$, where $p$ represents the probability that \dsurf enters the reactive state at a time step after the drift and $p^*$ represents the probability that \dsurf switches to the reactive model at the end of the reactive state. These two parameters are hard to analytically estimate because they depend on the magnitudes and frequencies at which concept drifts occur (which in turn impact the age of the stable model and consequently its risk), for which there is no agreed upon model. Additionally, the reactive and stable models are trained over points drawn from different distributions, and so $p^*$, which is determined by the difference in risks of the two models, may depend on the different \emph{approximation errors} for each distribution (the approximation error is the optimal expected risk within the function class $\funcclass$). By parameterizing $p, p^*$, we are able to show the general results in this section without making too simplistic assumptions about concept drifts.  (Note that such concerns did not arise in the simpler setting of a stationary environment in Section~\ref{sec:stationary}.)

\begin{lemma}
\label{lemma:NumOfReactiveState}
With probability $1 - \epsilon$, the number of times \dsurf enters the reactive state before recovering from a drift is less than $\frac{1}{p^*} + \sqrt{\frac{1-\epsilon}{\epsilon} (\frac{1-p^*}{{p^*}^2})}$.
\end{lemma}
\begin{proof}
Let $X$ be a random variable denoting the number of times \dsurf enters the reactive state after a drift and before recovering from it. Using Cantelli's inequality for any real number $\lambda> 0$, we have:
\[
\Pr[X - \mu \geq \lambda] \leq \frac{\sigma^2}{\sigma^2+\lambda^2}
\]
where $\mu = \expec{}{X} = \frac{1}{p^*}$ and $\sigma^2 = \var{X} = \frac{1-p^*}{{p^*}^2}$. Let $\lambda = \frac{k}{p^*}$, therefore,
\begin{align*}
    \Pr[X \geq \frac{(k+1)}{p^*}] \leq \frac{1}{1+\frac{k^2}{1-p^*}} \leq \epsilon.
\end{align*}
\end{proof}

Using Lemma~\ref{lemma:NumOfReactiveState}, w.h.p. we can estimate the recovery time of \dsurf as follows:

\begin{lemma}
\label{lemma:GeometricBound}
Let $X = \sum_{i=1}^{k}X_i$, where $k \geq 1$ and $X_i$ for $i = 1, ..., k$, are independent geometric random variables distributed $X_i \sim Ge(p)$ and $\expec{}{X} = \frac{k}{p}$. For any $\lambda \geq 1$, we have:
\begin{align*}
    \Pr\left[X \geq \frac{\lambda k}{p}\right] \leq e^{-k(\frac{\lambda}{2} - \ln{2})}
\end{align*}
\end{lemma}

\begin{proof}
Similar to the proof of Theorem 2.1 in \cite{janson2018tail} and by setting parameter $t$ (defined in their proof) to $\frac{p}{2}$.
\end{proof}

We can now prove Lemma~\ref{lemma:recoverytime} from Section~\ref{sec:analysis-abrupt-drift}.
\begin{replemma}{lemma:recoverytime}
With probability $1-\epsilon$, the recovery time of \dsurf is bounded by $kr + \frac{2}{p}(\ln{\frac{1}{\epsilon_1}} + k\ln{2})$, where $k < \frac{1}{p^*} + \sqrt{\frac{1-\epsilon_2}{\epsilon_2} (\frac{1-p^*}{{p^*}^2})}$ is the number of times \dsurf enters the reactive state before recovering from drift, and $\epsilon = \epsilon_1 + \epsilon_2$.
\end{replemma}

\begin{proof}
Let $X = \sum_{i=1}^{k}X_i$, where $k \geq 1$ and $X_i$ for $i = 1, ..., k$, are independent geometric random variables with distributions: $X_i \sim Ge(p)$. 
Using Lemma \ref{lemma:GeometricBound} for $\lambda = 1$ we have:
\[
\Pr\left[X \geq \frac{k}{p}\right] \leq e^{-k(\frac{1}{2} - \ln{2})}
\]
Therefore, with probability $ 1 - \epsilon_1$, we have $ X < \frac{2}{p}(\ln{\frac{1}{\epsilon_1}} + k\ln{2})$. Consequently, w.h.p. the total number of time steps before recovering from the drift will be less than $kr + \frac{2}{p}(\ln{\frac{1}{\epsilon_1}} + k\ln{2})$. Besides, using Lemma~\ref{lemma:NumOfReactiveState} we have with probability $1 - \epsilon_2$, $k < \frac{1}{p^*} + \sqrt{\frac{1-\epsilon_2}{\epsilon_2} (\frac{1-p^*}{{p^*}^2})}$.
\end{proof}

With the preceding lemmas, we can now establish the risk-competitiveness of \dsurf in the stationary period between abrupt drifts at times $t_{d_i}$ and $t_{d_{i+1}}$. The full proof is given in Section \ref{sec:analysis}. Note that if two drifts occur rapidly in succession, the condition in Lemma \ref{lemma:risk-competitive-drift} of $t_{d_i} + l < t < t_{d_{i+1}}$ may correspond to an empty domain if the recovery time bound of \dsurf exceeds the gap between the drifts.

\begin{replemma}{lemma:risk-competitive-drift}
With probability $1 - \epsilon$, the predictive model of \dsurf in the stable state is $\frac{7}{4^{1-\alpha}}$-risk-competitive with \aware at any time step $t_{d_i} + l + \max\left(l, 2r \ln \left(\frac{2(br\theta)^2}{\epsilon_3-br\theta} \right) \right) \leq t < t_{d_{i+1}}$, where $t_{d_i}$ is the time step of the most recent drift, $l = kr + \frac{2}{p}(\ln{\frac{1}{\epsilon_1}} + k\ln{2})$ where $k < \frac{1}{p^*}+ \sqrt{\frac{1-\epsilon_2}{\epsilon_2} (\frac{1-p^*}{{p^*}^2})}$, and $\epsilon = \epsilon_1 + \epsilon_2 + \epsilon_3$.\end{replemma}
\begin{proof}
Based on Lemma~\ref{lemma:recoverytime}, with probability $1 - \epsilon$, \dsurf recovers from drift after $l =  kT' + \frac{2}{p}(\ln{\frac{1}{\epsilon_1}} + k\ln{2})$ time steps, where $k < \frac{1}{p^*} + \sqrt{\frac{1-\epsilon_2}{\epsilon_2} (\frac{1-p^*}{{p^*}^2})}$, and $\epsilon = \epsilon_1 + \epsilon_2$. After recovering from the drift, the situation is similar to the stationary case. Let $t_r$ be the time step that \dsurf recovers from the most recent drift at time $t_d = t_{d_i}$. Also, let $t_e$ be the time step that the current predictive model was initialized.

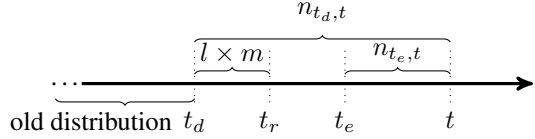
\begin{figure}[H]
\centering
\begin{tikzpicture}[
    scale=1,
    axis/.style={very thick, ->, >=stealth'},
    dashed line/.style={dashed, very thick},
    ]

\draw[axis] (0,0) -- (6,0) ;
\draw[dotted] (1.5,-0.25) -- (1.5,0.5);
\draw[dotted] (2.5,-0.25) -- (2.5,0.5);
\draw[dotted] (3.5,-0.25) -- (3.5,0.5);
\draw[dotted] (4.9,-0.25) -- (4.9,0.5);

\draw
    (-0.2,0) node[auto=false]{\ldots} (0,0)
    (1.5,-0.25) node[anchor=north] {$t_d$}
    (2.5,-0.25) node[anchor=north] {$t_r$}
    (3.5,-0.25) node[anchor=north] {$t_e$}
    (4.9,-0.25) node[anchor=north] {$t$}

	[decoration={brace, mirror, raise=0.2cm}, decorate] (-0.35,0) -- (1.5,0)
    (0.125, -0.25) node[anchor=north]{old distribution};
    
\draw 
    [decoration={brace}, decorate] (1.5,0.125)--(2.5,0.125)
    (2,0.15) node[anchor=south] {$l\times m$};

\draw
    [decoration={brace}, decorate] (3.5,0.125)--(4.9,0.125)
    (4.2,0.15) node[anchor=south] {$n_{t_e,t}$};
    
\draw 
    [decoration={brace, raise=0.2cm}, decorate] (1.5,0.35)--(4.9,0.35)
    (3.2,0.7) node[anchor=south] {$n_{t_d,t}$};

\end{tikzpicture}
\caption{A drift happens at time $t_d$. \dsurf recovers by time $t_r$. The current predictive model is initialized at time $t_e$.}
\label{fig:drift}
\end{figure}

To show \dsurf is $\frac{7}{4^{1-\alpha}}$-risk-competitive to \aware, we want to show $n_{t_e,t} \geq \frac{n_{t_d,t}}{4}$. 
Using Corollary~\ref{corollary:noDrift-ave age-general}, w.p. $1 - \epsilon_3$ we have $n_{t_e,t} \geq \frac{n_{t_r,t}}{2}$ at any time step $t$ such that $t_r + 2r \ln \left(\frac{2(br\theta)^2}{\epsilon_3-br\theta} \right) \leq t < t_{d_{i+1}}$, where $r$ is the length of reactive state and $\theta$ is the false positive rate of entering the reactive state. Therefore, $n_{t_e,t} \geq n_{t_r,t_e}$. On the other hand, we have
\begin{align*}
    n_{t_e,t} &= n_{t_d,t} - n_{t_d,t_r} - n_{t_r,t_e} \\
    &= n_{t_d,t} - l \times m - n_{t_r,t_e} \geq n_{t_d,t} - l \times m - n_{t_e,t}.
\end{align*}
Also, at any time step $t$ such that $ t - t_d \geq l + \max\left(l,2r \ln \left(\frac{2(br\theta)^2}{\epsilon_3-br\theta} \right) \right)$, we have $t - t_d \geq 2l$. Therefore, 
\begin{align*}
    2n_{t_e,t} \geq n_{t_d,t} - l\times m \geq \frac{n_{t_d,t}}{2}.
\end{align*}
It remains to bound the expected sub-optimality over $\sample_{t_d,t}$. Lemma~\ref{lemma: subopt} bounds the expected sub-optimality over $\sample_{t_e,t}$ as $(1 + o(1))\strisk(n_{t_e,t})$, and Equation~\ref{eq:switch-cost} relates the expected sub-optimality over $\sample_{t_e,t}$ to the expected sub-optimality over $\sample_{t_d,t}$:
\begin{align*}
    \expec{}{\subopt_{\sample_{t_d,t}}(\dsurf)} &\leq 
    (1 + o(1))\strisk\left(\frac{n_{t_d,t}}{4}\right) +  \frac{n_{t_d,t} - n_{t_d,t}/4}{n_{t_d,t}} \strisk\left(\frac{n_{t_d,t}}{4}\right) \\
    &\leq \frac{7}{4^{1-\alpha}} (1 + o(1))\strisk(n_{t_d,t}). \qedhere
\end{align*}
\end{proof}

Corollary \ref{corollary:noDrift-worst age} guarantees a minimum risk-competitiveness.

\begin{corollary}
\label{corollary:noDrift-worst age}
At any time step $t > r$, the size of sample set $\sample$ for the predictive model in the stable state is larger than $r\times m$, where $r$ is the length of the reactive state. Therefore, \dsurf is at worst $2(\frac{t}{r})^\alpha$-risk-competitive with \aware.
\end{corollary}

\begin{proof}
The proof is similar to the proof of Lemma \ref{lemma:risk-competitive-drift} and is a consequence of the algorithm's design as \dsurf may change its predictive model only at the end of the reactive state, which lasts $r$ time steps.
\end{proof}

%% file: appendix-expt-setup.tex
\section{Additional Details on the Experimental Setup}

This section contains additional details on the algorithms, datasets, and training for the experimental evaluation. 

\subsection{Algorithms Evaluated}
\label{sec:expt-setup-algs}
In our experimental evaluation, we compare our algorithm \dsurf to MDDM \cite{pesaranghader2018mcdiarmid} and AUE \cite{brzezinski2013reacting}, as representatives of state-of-the-art drift-detection-based and ensemble-based algorithms, respectively. The MDDM algorithm maintains a sliding window over the prediction results, which is a binary series indicating for each data point whether the model's predicted label matches the true label. MDDM signals a drift whenever a weighted mean over the sliding window is worse than the best observed weighted mean so far by a specified threshold. Upon signaling a drift, the current model is discarded and a new model is initialized starting at the current time step. Pesaranghader et al. offer three variants of their algorithm, MDDM-A, MDDM-G, and MDDM-E, differing in the weighting scheme applied over the sliding window. Pesaranghader et al. remark that ``all three variants had comparable levels of accuracy'' across each dataset they tested and that ``the optimal shape for the weighting function is data, context and application dependent'' \cite{pesaranghader2018mcdiarmid}. Generally, we do not know the type of drifts that will occur in advance, and so in our experiments, we used the intermediate choice MDDM-G, corresponding to a geometric weighting. (We also perform a sensitivity study among all three variants.) We reused the source code for MDDM-G available in the Tornado framework from Pesaranghader et al., and we used their default parameters for their algorithm: the window size $n = 100$, the confidence level $\delta_w = 10^{-6}$, and the geometric weighting factor $r = 1.01$.

The AUE algorithm (sometimes called AUE2 to distinguish from a preliminary published version of the algorithm) manages an ensemble of $k$ experts that are incrementally trained over the stream. After each batch of arrivals, AUE updates the weight of each expert based on its prediction error, and drops the lowest weighted expert to introduce a new expert. The prediction output from the ensemble is a weighted vote by its experts. We used the parameter $k=10$ as the limit on the total number of experts, following the choice made by Brzezinski and Stefanowski in their experimental evaluation \cite{brzezinski2013reacting}.

For the implementation of our algorithm \dsurf, we used the following parameters. The length of the reactive state $r = 4$. Regarding the conditions to enter the reactive state described in Section \ref{sec:algorithm}, the threshold for condition \ref{eq:enter_reactive_first} is $\delta = 0.1$, and the threshold for condition \ref{eq:enter_reactive_second} is $\delta' = \delta/2$.

In our main experiment, on each dataset discussed below, we evaluate
\dsurf, MDDM (the MDDM-G variant), AUE, and the \aware algorithm with
oracle access to when drifts occur (discussed in
Section~\ref{sec:analysis}). We also run additional experiments for
MDDM-A, MDDM-E, single-pass SGD, and an oblivious algorithm, which
maintains a single model updated with \incsaga. The version of
\incsaga in the oblivious algorithm samples uniformly from its sample
set at each iteration and has no bias towards sampling more recent
data arrivals.

When using \incsaga or any other SGD-style optimization, we consider a
parameter $\rho$ that dictates the number of update steps
(specifically, gradient computations) that are available to train the
model. The four adaptive learning algorithms maintain a different
number of models---\dsurf uses 2, \aware and MDDM use 1, and AUE uses
10. This leads us to consider two different possibilities for training
at each time: (1) each algorithm can use $\rho$ steps per model; or
(2) each algorithm has $\rho$ steps in total that are divided equally
across its models. The second approach accounts for the varying
computational efficiency of each algorithm and lets us examine the
accuracy achieved when enforcing equal processing time.

For our evaluation under equal processing time, we also evaluate another ensemble method, Candor \cite{zhao2020handling}. Candor is a more computationally efficient ensemble method than AUE because it only trains one newly added expert at a time. Candor manages a total of $K$ experts, for which weights are updated based on observed losses with exponential decay factor $\eta$, and the prediction output is a weighted vote. After each epoch of Candor, a new model is added (deleting the oldest if the total exceeds $K$) to minimize the loss over the previous epoch plus an added biased regularization term $\frac{\mu}{2} ||\weight - \weight_p||^2$, where $\weight_p$ is the weighted linear combination of the ensemble's experts. In adapting the original point-wise Candor algorithm to our batch setting, we redefine an epoch to be the batch size of the stream for consistent comparison. We set $K=25, \eta = 0.75$ following the choice made by the authors in their experimental evaluation. Finally, we set $\mu$ to be the same regularization constant per dataset we use for L2-regularization in training the models of the other evaluated algorithms.

\begin{table}[h!]
\caption{Basic statistics of datasets }
\label{table:dataset}
\vskip 0.15in
\begin{center}
\begin{small}
\begin{sc}
\begin{tabular}{llll}
\toprule
 & Dataset & \# instance & \# Dim  \\
\midrule
\multirow{5}{*}{Synthetic}  
& \multicolumn{1}{l}{SEA} & \multicolumn{1}{l}{100000} & \multicolumn{1}{l}{3} \\\cline{2-4}
& \multicolumn{1}{l}{HyperPlane} & \multicolumn{1}{l}{100000} & \multicolumn{1}{l}{10} \\\cline{2-4}
& \multicolumn{1}{l}{SINE1} & \multicolumn{1}{l}{10000} & \multicolumn{1}{l}{2} \\\cline{2-4}
& \multicolumn{1}{l}{Mixed} & \multicolumn{1}{l}{100000} & \multicolumn{1}{l}{4} \\\cline{2-4}
& \multicolumn{1}{l}{Circles} & \multicolumn{1}{l}{10000} & \multicolumn{1}{l}{2} \\\hline

\multirow{2}{*}{\makecell{semi-\\synthetic}}  
& \multicolumn{1}{l}{RCV1} & \multicolumn{1}{l}{20242} & \multicolumn{1}{l}{47235} \\\cline{2-4}
& \multicolumn{1}{l}{Covertype} & \multicolumn{1}{l}{581012} & \multicolumn{1}{l}{54} \\\hline

\multirow{3}{*}{Real}  
& \multicolumn{1}{l}{Airline} & \multicolumn{1}{l}{5810462} & \multicolumn{1}{l}{13} \\\cline{2-4}
& \multicolumn{1}{l}{Electricity} & \multicolumn{1}{l}{45312} & \multicolumn{1}{l}{13} \\\cline{2-4}
& \multicolumn{1}{l}{PowerSupply} & \multicolumn{1}{l}{29928} & \multicolumn{1}{l}{2}  \\

\bottomrule
\end{tabular}
\end{sc}
\end{small}
\end{center}
\vskip -0.1in
\end{table}

\subsection{Datasets}
\label{sec:expt-setup-datasets}
Our experiments use the 5 synthetic, 2 semi-synthetic and 3 real-world datasets shown in Table~\ref{table:dataset} and described below. The selection of datasets included all datasets for binary classification used in the experimental evaluations by Pesaranghader et al.~on their MDDM algorithm (namely, SINE1 and Electricity) and Brzezinski and Stefanowski on their AUE algorithm (SEA10, Hyperplane-Slow, Hyperplane-Fast, Electricity, and Airlines).
\begin{itemize}
    \item SEA \cite{bifet2010moa}: This dataset is generated using the Massive Online Analysis (MOA) framework. There are three attributes in $[0, 10]$. The label is determined by $x_1 + x_2 \leq \theta_j$ where $j$ corresponds to 4 different concepts, $\theta_1 = 9, \theta_2 = 8, \theta_3 = 7, \theta_4 = 9.5$ (the third attribute $x_3$ is not correlated with the label). We synthetically generated 25000 points from each concept in the order 3, 2, 4, 1, following the example from the MOA manual. We experimented on four different datasets varying the amount of noise, SEA0, SEA10, SEA20, SEA30, corresponding to 0\%, 10\%, 20\%, and 30\% of the labels being swapped during the generation of the dataset.  SEA-gradual is generated by generating samples from two concepts (the first two concepts discussed above) during the drift period. 
        
    \item Hyperplane \cite{bifet2010moa}: This dataset is generated using the MOA framework. For each data point, the label corresponds to its half space for an underlying hyperplane, where each coordinate of the hyperplane changes by some magnitude for each point in the stream, representing a continually gradually drifting concept.  We experimented on two variations, Hyperplane-Slow and Hyperplane-Fast, corresponding to a 0.001 and a 0.1 magnitude of change. In each case, at each point in the stream, there is a 10\% probability that the direction of the change is reversed.
    
    \item SINE1~\cite{pesaranghader2016framework}: This dataset contains two attributes $(x_1, x_2)$, uniformly distributed in $[0, 1]$. Label of each data is determined using a sine curve as follows: $x_2 \leq sin(x_1)$. Labels are reversed at drift points.
    
    \item Mixed~\cite{pesaranghader2016framework}: This dataset contains four attributes $(x_1, x_2, x_3, x_4)$, where $x_1$ and $x_2$ are boolean and $x_3, x_4$ are uniformly distributed in $[0, 1]$. Label of each data is determined to be positive if two of $x_1, x_2$, and $x_4 < 0.5 + 0.3 \times \sin(3\pi x_3)$ hold.  Labels are reversed at drift points.
     
      \item Circles~\cite{pesaranghader2016framework}: This dataset contains two attributes $(x_1, x_2)$, uniformly distributed in $[0, 1]$. Label of each data is determined using a circle as the decision boundary as follows: $(x_1 - c_1)^2 + (x_2 - c_2)^2 <= r$, where $(c_1, c_2)$ and $r$ are (respectively) center and radius of the circle. Drift happens in a gradual manner where the center and radius of decision boundary changes over a period of time. We experimented on a generated dataset with $3$ gradual drift introduced at time $25, 50$, and $75$, where the transition period for each drift is $5$ time steps. 
    
    \item RCV1~\cite{lewis2004rcv1}: This real world data set contains  manually categorized newswire stories. The original order of the data set we used was randomly permuted before inserting drift. At drift points, we introduce a sharp abrupt drift by swapping each label. 
    
    \item Covertype~\cite{UCI-ML}: This real world data set contains observation of a forest area obtained from US Forest Service (USFS) Region 2 Resource Information System (RIS). Binary class labels are involved to represent the corresponding forest cover type. The original order of the data set we used was randomly permuted before inserting drift. At drift points, we introduce an abrupt drift by rotating each data point by $180^\circ$ along the 1st and 8th attributes. This particular rotation was chosen because it resulted in approximately 40\% misclassification rate with respect to the current predictive model.
    
    \item Airline(2008)~\cite{ElenaIko20:online}: This real world data set contains records of flight schedules. Binary class labels are involved to represent if a flight is delayed or not. Concept drift could appear as the result of changes in the flights schedules, e.g. changes in day, time, and the length of flights. In our experiments, we used the first 58100 points of the data set, and pre-processed the data by using one-hot encoding for categorical features and scaling numerical features to be in the range $[0, 1]$. The original dataset contains 13 features. But, after using one-hot encoding the dimension increases to $679$.

    \item Electricity~\cite{Harries99splice-2comparative}: This real world data set contains records of the New South Wales Electricity Market in Australia. Binary class labels are involved to represent the change of the price (i.e., up and down).  The concept drift may result from changes in consumption habits or unexpected events.

    \item Power Supply~\cite{dau2019ucr}: This real world data set contains records of hourly power supply of an Italy electricity company which records the power from two sources: power supply from main grid and power transformed from other grids. Binary class labels are involved to represent which time of day the current power supply belongs to (i.e. am or pm). The concept drifting in this stream may results from the change in season, weather or the differences between working days and weekend.
    
\end{itemize}

The type of drift in each dataset is detailed in Table
\ref{table:dataset-drift}. When working with real datasets, precisely
determining the time drift occurs is somewhat guesswork. Brzezinski
and Stefanowski remarked they ``cannot unequivocally state when drifts
occur or if there is any drift'' on the real datasets they
considered~\cite{brzezinski2013reacting}. Still, we had to mark the
drift times for the implementation of \aware, which resets the model
whenever drifts occur.  We chose these times by observing the
misclassification rates of an oblivious algorithm that is not designed
to adapt to drift, and noting for which time steps there was a
significant increase in misclassifications on the newly arrived batch.

\begin{table}[h!]
\caption{Details of drifts in datasets }
\label{table:dataset-drift}
\vskip 0.15in
\begin{center}
\begin{small}
\begin{sc}
\begin{tabular}{llll}
\toprule
 & Dataset & Drift Type & Drift times  \\
\midrule
\multirow{5}{*}{Synthetic} 
& \multirow{2}{*}{SEA} & Abrupt & [25, 50, 75] \\ \cline{3-4}
& & Gradual & [40-60] \\\cline{2-4}
& HyperPlane & Gradual & -\\\cline{2-4}

& SINE1 & Abrupt & [20, 40, 60, 80]  \\\cline{2-4}
& Mixed & Abrupt & [20, 40, 60, 80]  \\\cline{2-4}
& Circles & Gradual & [25-30, 50-55, 75-80] \\\hline

\multirow{2}{*}{\makecell{semi-\\synthetic}}  
& \multicolumn{1}{l}{RCV1} & \multicolumn{1}{l}{Abrupt} & \multicolumn{1}{l}{[30, 60]}\\\cline{2-4}
& \multicolumn{1}{l}{Covertype} & \multicolumn{1}{l}{Abrupt} & \multicolumn{1}{l}{[30, 60]}\\\hline

\multirow{3}{*}{Real}  
& \multicolumn{1}{l}{Airline} &\multicolumn{1}{l}{-} & \multicolumn{1}{l}{[31, 67]}\\\cline{2-4}
& \multicolumn{1}{l}{Electricity} & \multicolumn{1}{l}{-} & \multicolumn{1}{l}{[20]} \\\cline{2-4}
& \multicolumn{1}{l}{PowerSupply} & \multicolumn{1}{l}{-} & \multicolumn{1}{l}{[17, 47, 76]}\\

\bottomrule
\end{tabular}
\end{sc}
\end{small}
\end{center}
\vskip -0.1in
\end{table}

\subsection{Training and Hyperparameters}
\label{sec:expt-setup-training}
On each dataset, the prediction task is binary classification. Each model $\weight$ trained is a linear model, using \incsaga to optimize the L2-regularized logistic loss over the relevant stream segment. For a data point $(x, y)$, the corresponding loss function is $f_{(x, y)}(\weight) = \log(1 + \exp(-y \weight^T x)) + \frac{\mu}{2}||\weight||_2^2$.

There are two hyperparameters used by \incsaga, the regularization factor $\mu$ and the constant step size $\eta$. To set them, we first took each dataset in static form (opposed to streaming) and applied a random permutation, partitioning an 80\% split for training and 20\% for validation. (For the case of the semi-synthetic datasets where we introduced our own drift, the hyperparameter selection was done prior to modifying the data.) We used grid search to determine the values of $\mu$ and $\eta$ that optimized the validation set error after running \incsaga over the static training set for a number of iterations equal to two times the number of data points. We searched for $\mu$ of the form $10^{-a}$ for $1 \leq a \leq 7$ and $\eta$ of the form $b\times 10^{-c}$ for $b \in \{1, 2, 5\}$ and $1 \leq c \leq 5$. The parameters we chose are given in Table \ref{table:hyperparameters}. In experiments where we used SGD for training, we used the same constant step size $\eta$.

\begin{table}[h!]
\caption{Hyperparameters and batch sizes}
\label{table:hyperparameters}
\vskip 0.15in
\begin{center}
\begin{small}
\begin{sc}
\begin{tabular}{l*{3}c}
\toprule
 Dataset            & Regularization $\mu$  & Step size $\eta$ & Batch size $m$ \\
\midrule
SEA (all)   & $10^{-2}$ & $1 \times 10^{-3}$ & 1000\\
Hyper-slow  & $10^{-3}$ & $1 \times 10^{-1}$ & 1000\\
Hyper-fast  & $10^{-3}$ & $1 \times 10^{-2}$ & 1000\\
SINE1       & $10^{-3}$ & $2 \times 10^{-1}$ & 100\\
Mixed       & $10^{-3}$ & $1 \times 10^{-1}$ & 1000\\
Circles       & $10^{-3}$ & $1 \times 10^{-1}$ & 100\\
RCV1        & $10^{-5}$ & $5 \times 10^{-1}$ & 202\\
CoverType   & $10^{-4}$ & $5 \times 10^{-3}$ & 5810\\
Airline     & $10^{-3}$ & $2 \times 10^{-2}$ & 581\\
Electricity & $10^{-4}$ & $1 \times 10^{-1}$ & 1333\\
PowerSupply & $10^{-3}$ & $1 \times 10^{-1}$ & 299\\
\bottomrule
\end{tabular}
\end{sc}
\end{small}
\end{center}
\end{table}

In the streaming data setting studied in this paper
(Section~\ref{prelim}), the batch size is determined by the rate of
arrival of new data points, and hence not a hyperparameter to be
tuned.  For simplicity, we assume that data arrive over the course of
$b$ time steps in equally-sized batches containing $m = (\text{dataset
  size})/b$ points, where $b=100$ for all datasets other than
Electricity.  For the case of Electricity, we defined the number of
time steps $b = 34$ so that one time step corresponds to 28 days of
the collected data, and was a scale where we could visually observe
drift in the results.  The resulting batch sizes are shown in the
last column of Table~\ref{table:hyperparameters}.

%% file: appendix-expt-results.tex
\section{Additional Experimental Results}
\label{sec:expt-results}

This section contains experimental results under both training
strategies of equal computational power for each model and equal
computational power for each algorithm, which is divided among its
models. Additionally, we report results for single-pass SGD and an
oblivious algorithm using \incsaga, results for \dsurf without the
greedy approach during the reactive state, and results for each
algorithm when SGD is used as the update process instead of \incsaga.

\subsection{Equal Computational Power for each Model}
\label{sec:expt-results-equal-model}

We present the misclassification rates at each time step over the new batch in Figure \ref{fig:Misclassification rate over time}, and the average misclassification rate over all time steps is summarized in Table \ref{table:ave-misclassification-app}. (These results are a superset of those presented in Figures \ref{fig:power} and \ref{fig:covtype} and Table \ref{table:ave-misclassification-compact} from Section \ref{sec:expt}). Here, we used the training strategy where at every time step, each algorithm uses $\rho=2m$ update steps for each of its models. Let us note a few general trends. The advantage of \dsurf over MDDM is most evident on the noisy versions of SEA (also shown in Figure \ref{fig:sea-noise}), and on CoverType and PowerSupply. The drift detection method MDDM encounters false positives that lead to unnecessary resetting of the predictive model, while \dsurf avoids the performance loss after most of the false positives by catching them via the reactive state. In particular, the CoverType dataset was especially problematic for MDDM, which continually signaled a drift. 

For true drifts when immediately switching to a new model is desirable, we observe, most evident on SINE1 and RCV1, that MDDM is the fastest to adapt, followed shortly by \dsurf, and with AUE lagging behind. CoverType also is a clear example where \dsurf adapts faster than AUE (but MDDM suffered as previously mentioned). For these drifts, MDDM naturally leads because it is using a new model when it accurately detects the drift, while \dsurf always takes at least one time step to switch because it waits until it sees a batch where the new (reactive) model outperforms the older (stable) model. Finally, AUE also takes at least one time step, because its ensemble members are weighted based on the previous performance, but it can take longer, because even if the older, inaccurate models are low-weighted, they are not weighted zero, and shortly after a drift, most of the models in the ensemble are trained on old data and can still negatively impact the predictions.

There are two major advantages of \dsurf and AUE not immediately switching to the latest model: (i) there are drifts for which switching to a new model is not desired because the older model can still provide good accuracy, and (ii) delaying the switch to a new model can be desired if the new model has poor accuracy immediately after the drift while it warms up. Regarding the first point, observe the drift in SEA10 at $t=25$ and the drift in Electricity. There is a notable degradation in accuracy of each algorithm at the time of the drift, but resetting the model as \aware does is a poor choice. We even observe that the oblivious algorithm (OBL) (which trains a model from the beginning of time and is not designed to adapt to drifts) outperforms \aware on these datasets. Despite the initial degradation in accuracy at the time of drift, we find that the older model is able to converge again after the drift, even while the older model is trained on data from both before and after the drift. Meanwhile, training a new model from scratch as \aware does is not worth the initial start-up cost when the older model performs well.

The reader may be skeptical specifically of \aware's reset to a random model for predictions at the time step drift occurs---practically, wouldn't it be preferable to use the previously-learned model for the first time step, and then switch to the new model? We considered this alternative implementation of \aware, and observed that across each dataset, the average misclassification rate of the alternative \aware was better by at most 1.1 percentage points than the version of \aware reported in Table \ref{table:ave-misclassification-app}, and was worse on SINE1 and RCV1. There was no case where the alternative \aware outperformed any algorithm in the table that \aware did not already outperform.

The second advantage previously mentioned, of delaying the switch to the new model, is best exemplified on Airline. Immediately after the two drifts, \dsurf is the best performer, followed by AUE, and then MDDM and \aware. Immediately after the drift, \dsurf continues to use the older, stable model, which outperforms a newly created model (compare \dsurf to \aware), because a new model needs a few time steps to train before it is a better choice, and then \dsurf switches later. AUE is of intermediate error in the time steps immediately after the drift, because it does place greater weight on the better performing, older models, but is still worse than placing unit weight on an old model.

Finally, the Hyperplane-slow and Hyperplane-fast warrant their own discussion. These two datasets represent a continually drifting concept throughout the entire stream. For Hyperplane-slow, AUE is the best performing algorithm, while for Hyperplane-fast, MDDM is the best performing. The advantage that AUE and MDDM have over \dsurf in these datasets is that AUE adds a new model at every time step, and MDDM has the capability of switching to a new model at any time step, and therefore, they can better fit the most recent data in the stream. On the other hand, \dsurf is only able to create a new model upon transitioning to the reactive state, so \dsurf does not have the capability of creating new models at time steps during its reactive state. \dsurf is not designed for the setting where creating a new model at every time step is desirable, but nonetheless, the accuracy of \dsurf is still comparable.
Furthermore, on the remaining datasets with gradual drift, SEA-gradual and Circles, that contain stationary periods and drift periods instead of the continual drift of Hyperplane, \dsurf is the best performer.

\begin{table}[ht!]
\caption{Total average of misclassification rate ($\rho=2m$ for each model)}
\label{table:ave-misclassification-app}
\begin{center}
\begin{small}
\begin{sc}
\begin{tabular}{l*{8}c}
\toprule
 Dataset    	& \aware & \dsurf& MDDM-G & MDDM-A & MDDM-E & AUE & 1PASS-SGD & OBL  \\
\midrule
SEA0        	& 0.137 & \textbf{0.088} & \textbf{0.088} & 0.090 & 0.087 & 0.094 & 0.131 & 0.110 \\
SEA10       	& 0.197 & \textbf{0.156} & 0.180 & 0.166 & 0.172 &  0.163 & 0.188 & 0.176 \\
SEA20       	& 0.264 & \textbf{0.246} & 0.293 & 0.278 & 0.289 & 0.247 & 0.267 & 0.254 \\
SEA30       	& 0.350 & \textbf{0.336} & 0.357 & 0.358 & 0.352 & 0.337 & 0.348 & 0.338 \\
SEA-gradual  	& 0.177 & \textbf{0.159} & 0.177 & 0.167 & 0.174 & 0.163 & 0.196 & 0.173 \\
Hyper-slow  	& 0.116 & 0.117 & 0.117 & 0.117 & 0.116 & \textbf{0.112} & 0.139 & 0.170 \\
Hyper-fast  	& 0.191 & 0.174 & \textbf{0.162} & 0.163 & 0.164 & 0.179 & 0.177 & 0.280 \\
SINE1       	& 0.171 & 0.197 & 0.179 & \textbf{0.175} & 0.178 & 0.211 & 0.223 & 0.477\\
Mixed       		& 0.192 & 0.204 & 0.204 & 0.204 & \textbf{0.203} & 0.211 & 0.208 & 0.455\\
Circles       	& 0.368 & \textbf{0.371} & 0.376 & 0.375 & 0.372 & 0.380 & 0.385 & 0.508\\
RCV1        	& 0.121 & 0.135 & \textbf{0.127} & 0.130 & 0.130 & 0.167 & 0.276 & 0.468 \\
CoverType   	& 0.267 & \textbf{0.268} & 0.313 & 0.311 & 0.313 & 0.278 & 0.298 & 0.321 \\
Airline     		& 0.338 & \textbf{0.332} & 0.348 & 0.346 & 0.348 & 0.333 & 0.340 & 0.359 \\
Electricity 		& 0.315 & 0.306 & 0.341 & 0.339 & 0.341 & 0.303 & 0.347 &  \textbf{0.302} \\
PowerSupply 	& 0.309 & 0.300 & 0.323 & 0.315 & 0.329 & \textbf{0.299} & 0.307 & 0.312  \\
\bottomrule
\end{tabular}
\end{sc}
\end{small}
\end{center}
\end{table}

\begin{figure*}[h!]
    \begin{center}
        \begin{subfigure}[t]{0.32\textwidth}
            \includegraphics[width=\columnwidth]{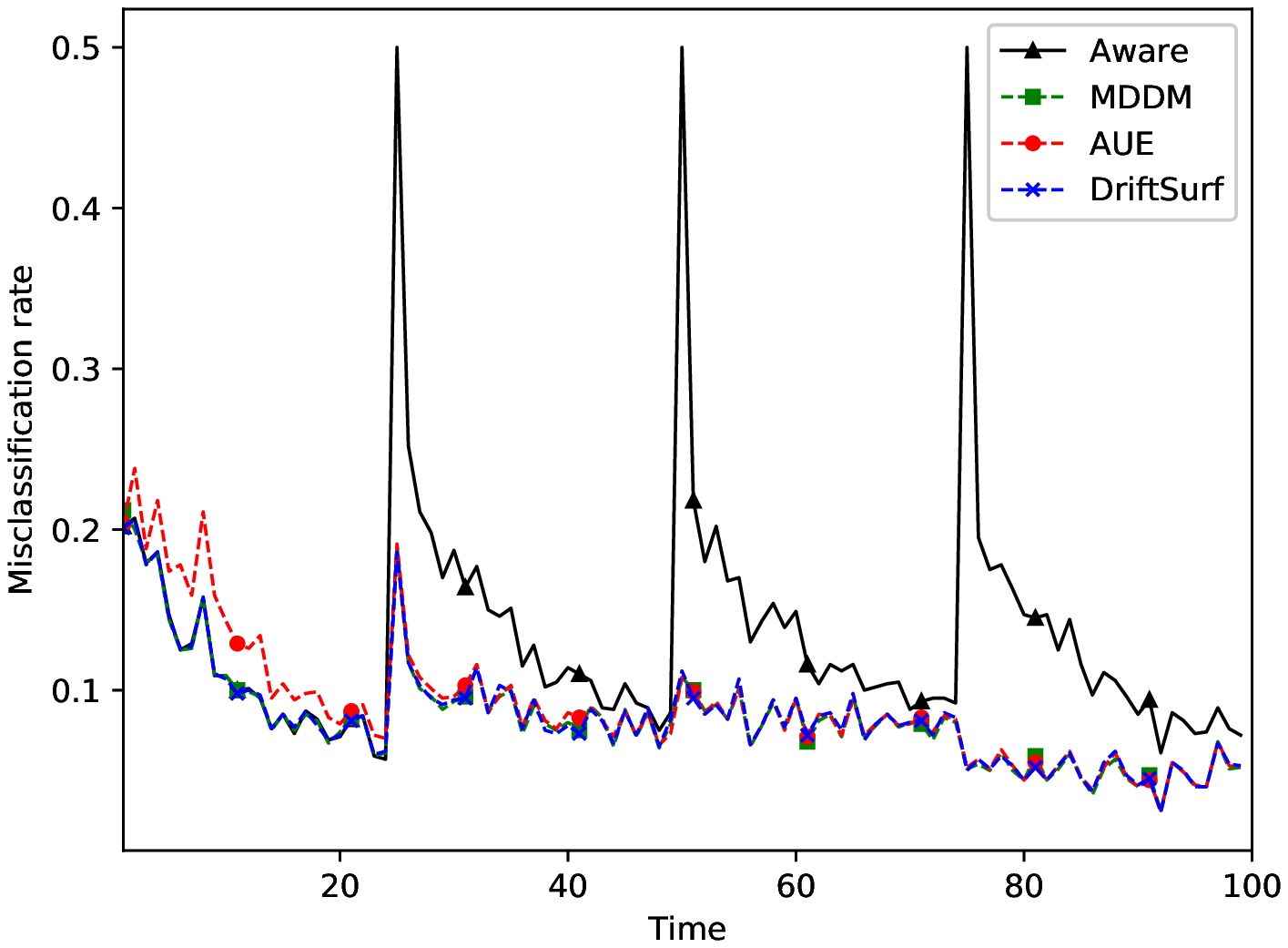}
            \caption{SEA0}
            \label{fig:sea0-unlimited}
        \end{subfigure}\hfill
        \begin{subfigure}[t]{0.32\textwidth}
            \includegraphics[width=\columnwidth]{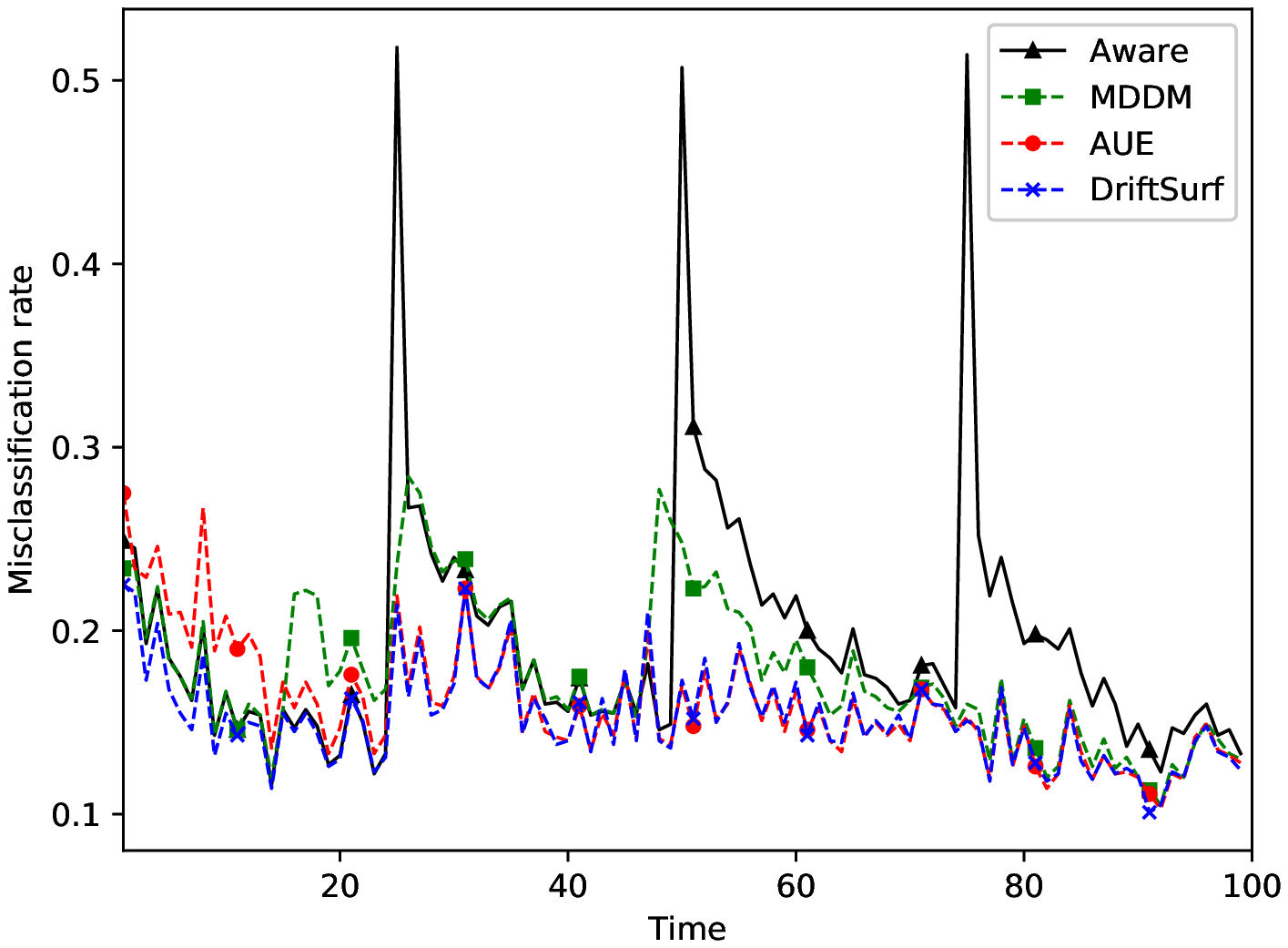}
            \caption{SEA10}
            \label{fig:sea10-unlimited}
        \end{subfigure}\hfill
        \begin{subfigure}[t]{0.32\textwidth}
            \includegraphics[width=\columnwidth]{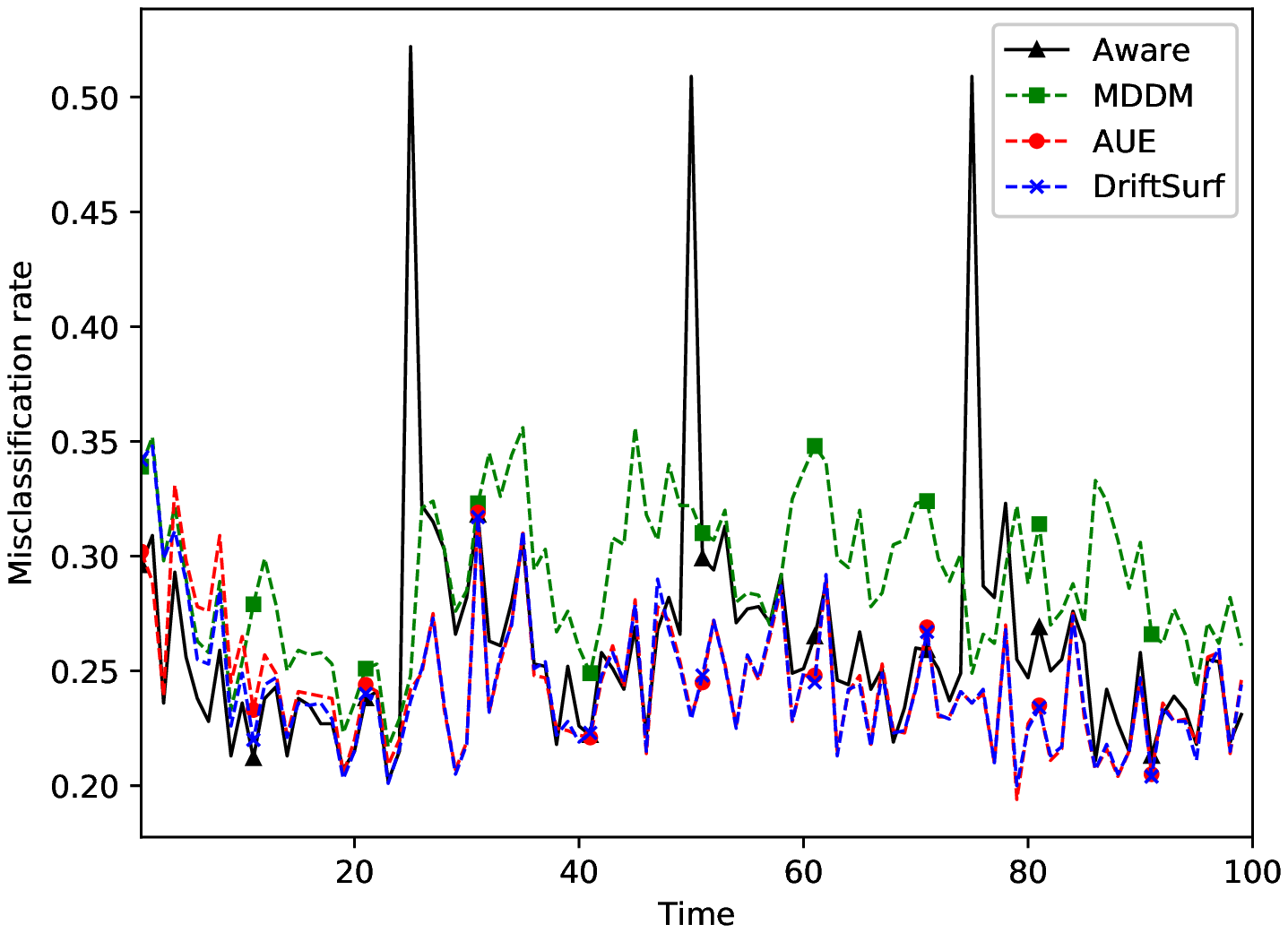}
            \caption{SEA20}
            \label{fig:sea20-unlimited}
        \end{subfigure}\hfill
        
        \begin{subfigure}[t]{0.32\textwidth}
            \includegraphics[width=\columnwidth]{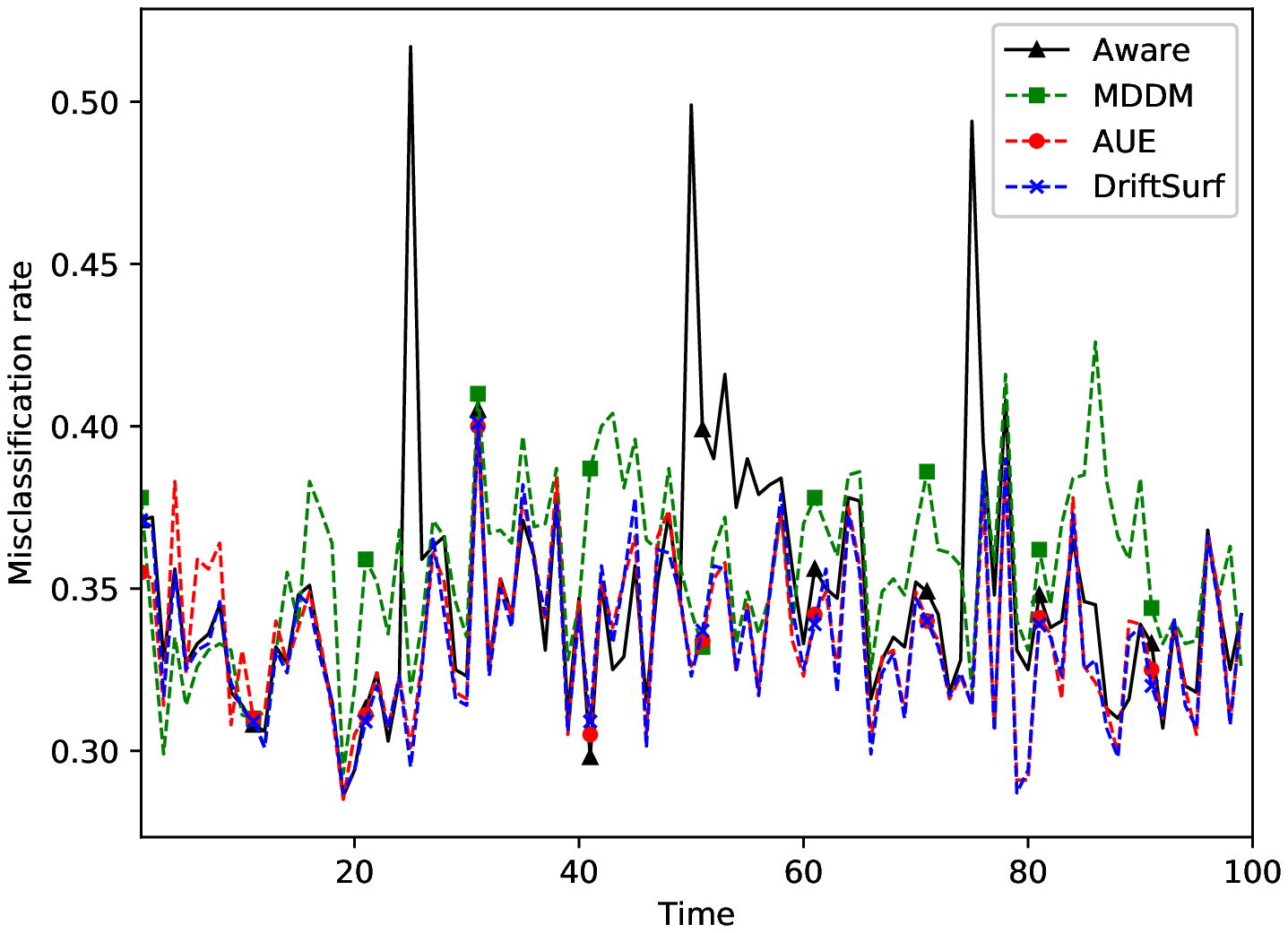}
            \caption{SEA30}
            \label{fig:sea30-unlimited}
        \end{subfigure}\hfill
        \begin{subfigure}[t]{0.32\textwidth}
            \includegraphics[width=\columnwidth]{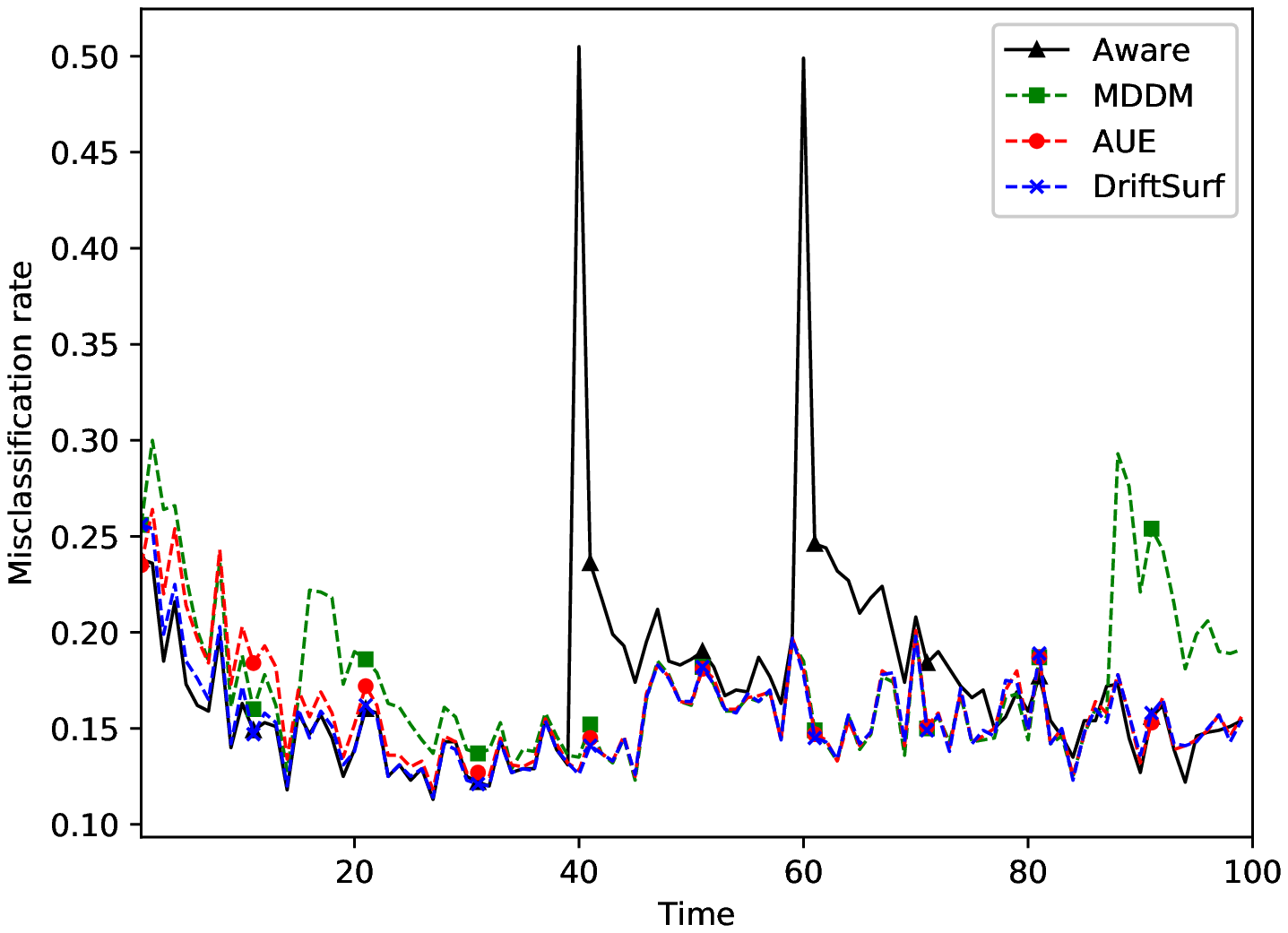}
            \caption{SEA-gradual}
            \label{fig:sea_gradual-unlimited}
        \end{subfigure}\hfill
	\begin{subfigure}[t]{0.32\textwidth}
            \includegraphics[width=\columnwidth]{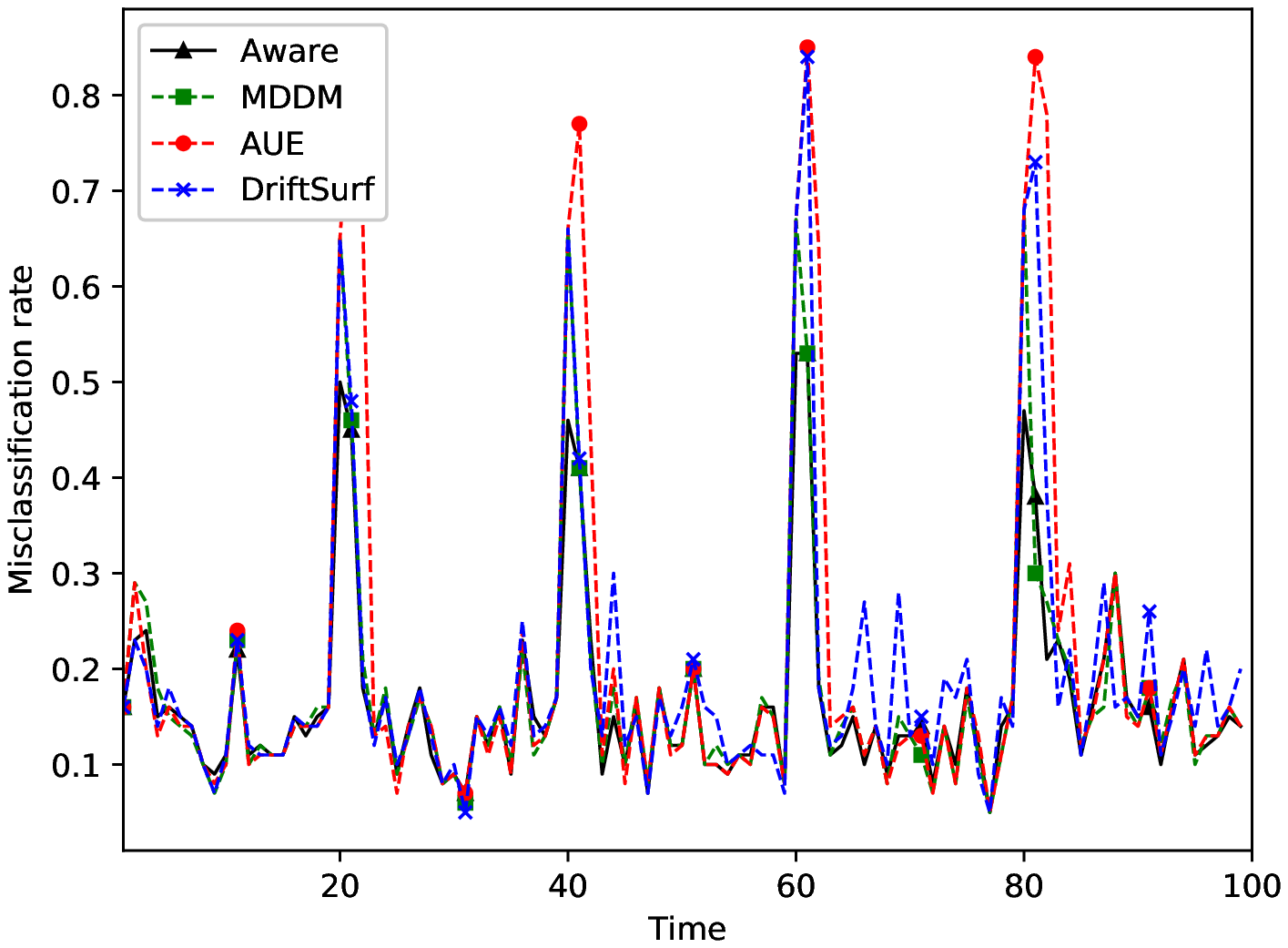}
            \caption{SINE1}
            \label{fig:sine1-unlimited}
        \end{subfigure}\hfill

        \begin{subfigure}[t]{0.32\textwidth}
            \includegraphics[width=\columnwidth]{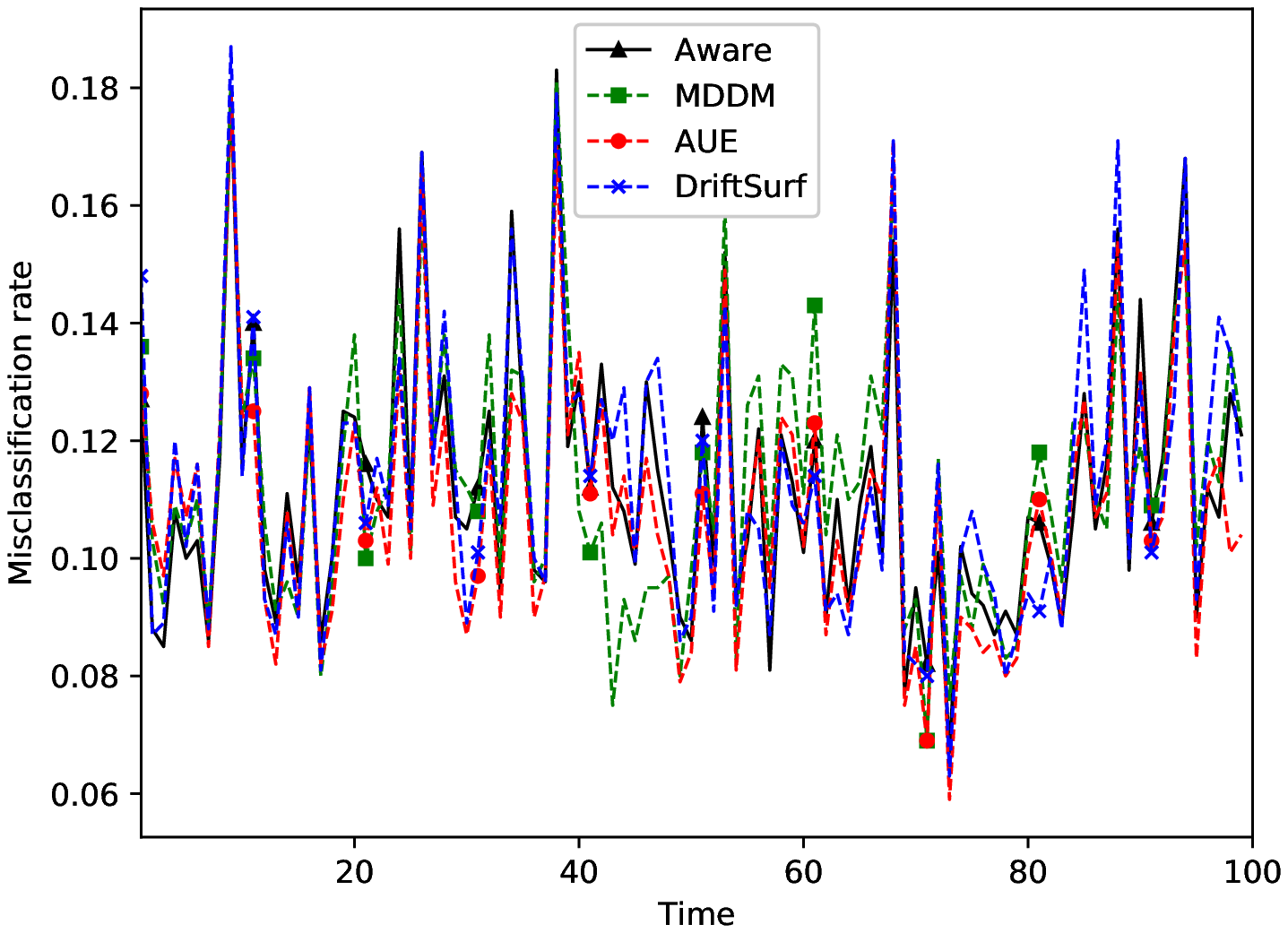}
            \caption{HyperPlane-slow}
            \label{fig:hyperplane-slow-unlimited}
        \end{subfigure}\hfill
        \begin{subfigure}[t]{0.32\textwidth}
            \includegraphics[width=\columnwidth]{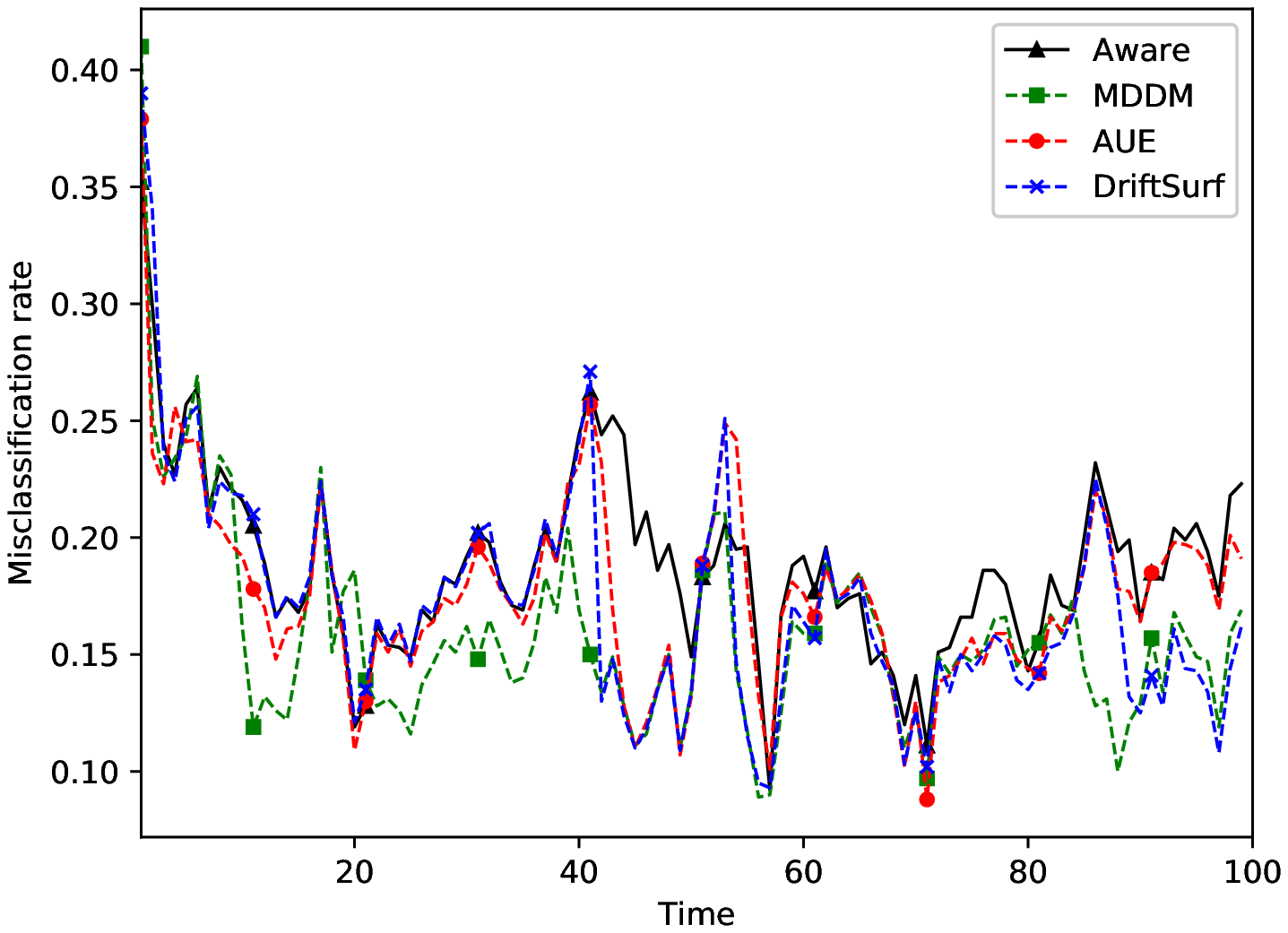}
            \caption{HyperPlane-fast}
            \label{fig:hyperplane-fast-unlimited}
        \end{subfigure}\hfill
        \begin{subfigure}[t]{0.32\textwidth}
            \includegraphics[width=\columnwidth]{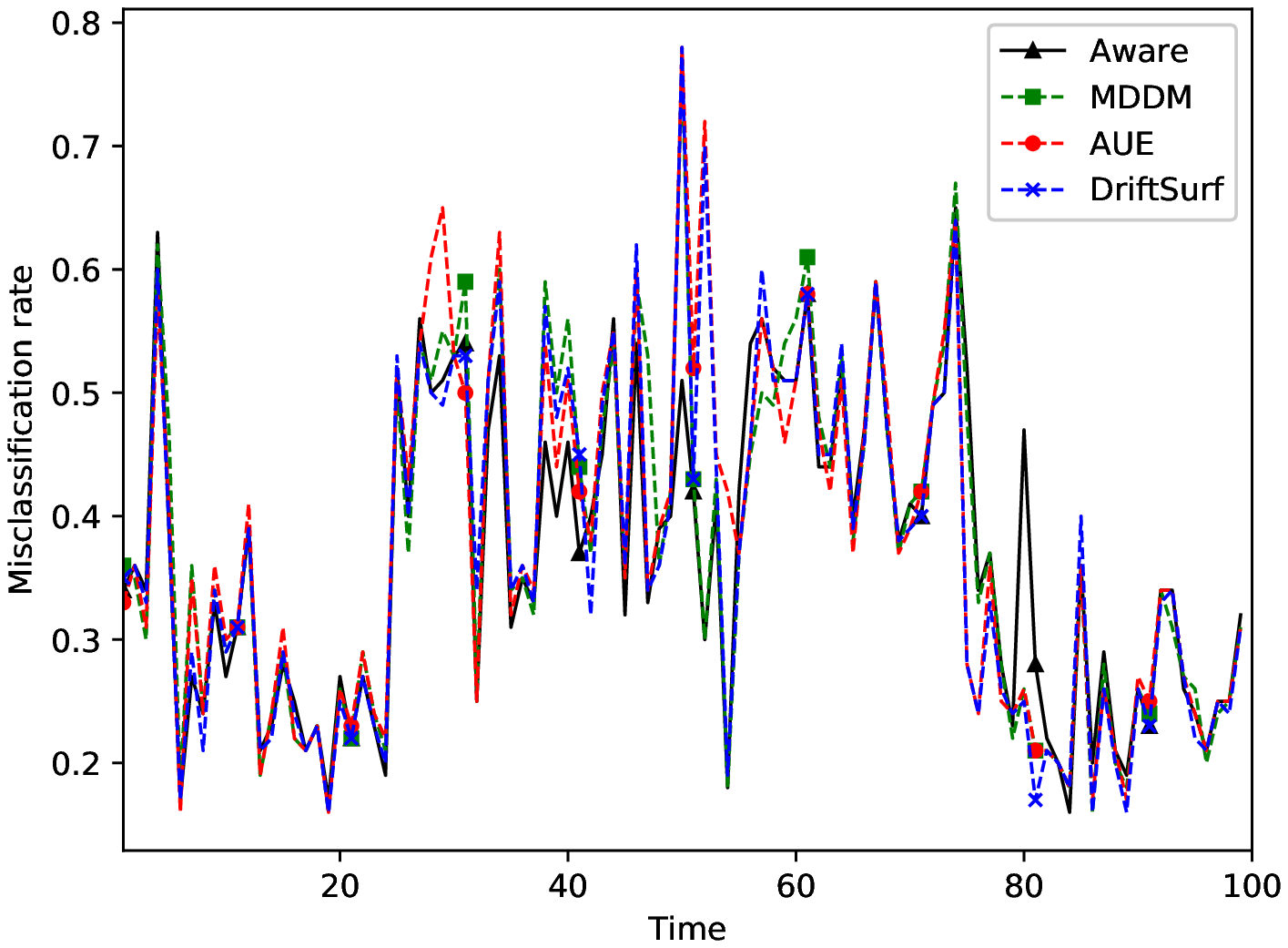}
            \caption{CIRCLES}
            \label{fig:circles-unlimited}
        \end{subfigure}\hfill
        
        \begin{subfigure}[t]{0.32\textwidth}
            \includegraphics[width=\columnwidth]{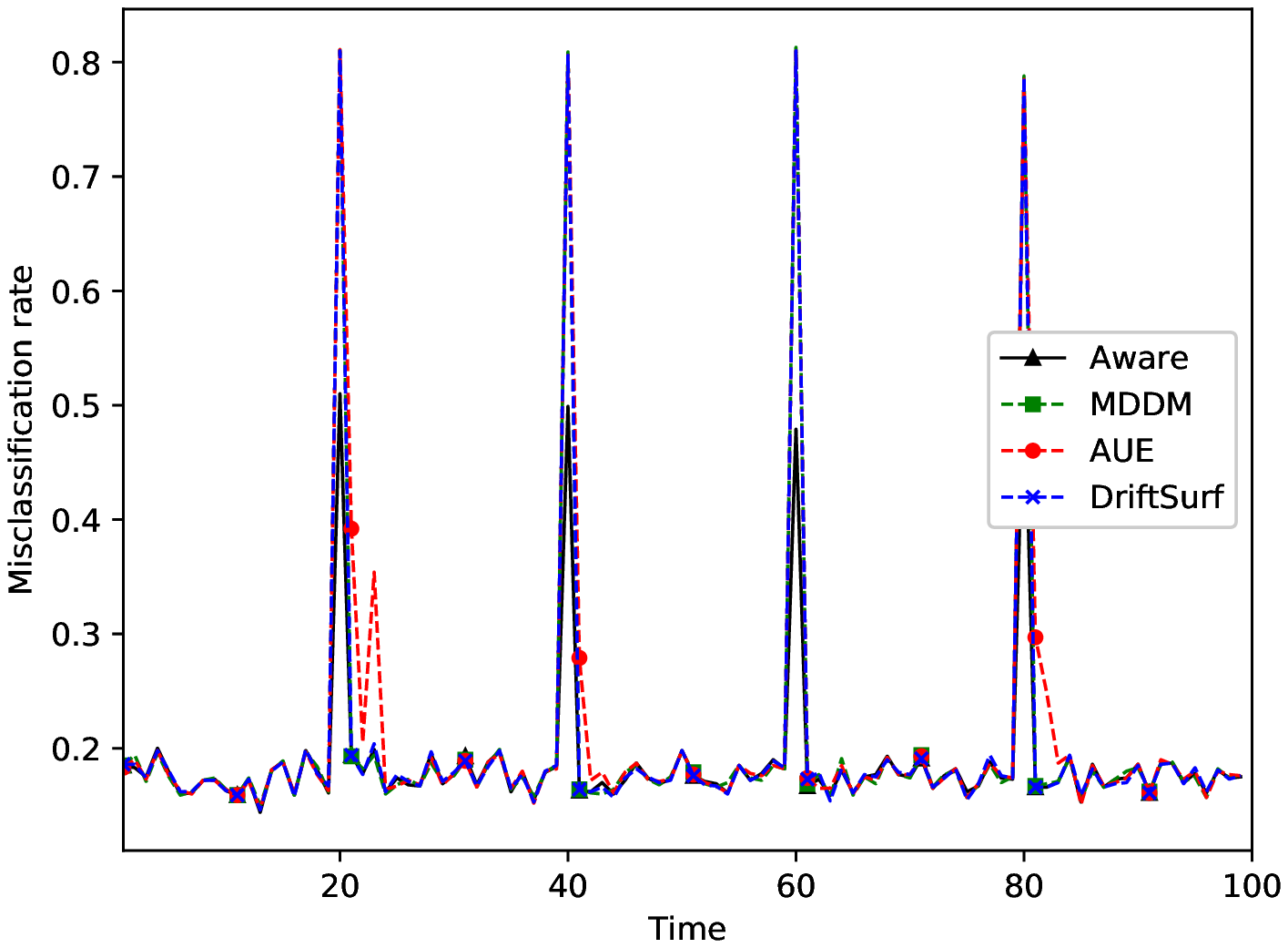}
            \caption{MIXED}
            \label{fig:mixed-unlimited}
        \end{subfigure}\hfill
        \begin{subfigure}[t]{0.32\textwidth}
            \includegraphics[width=\columnwidth]{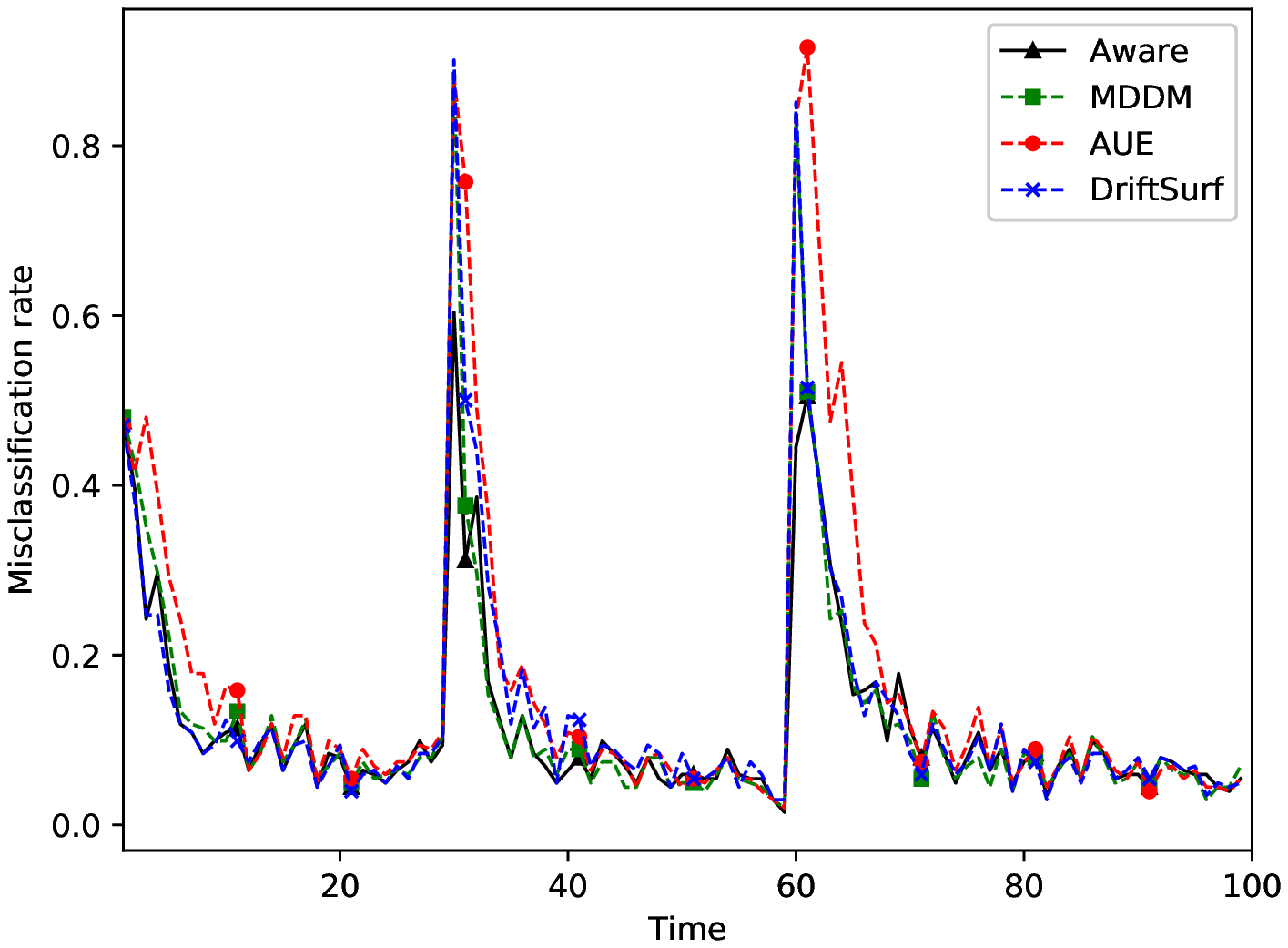}
            \caption{RCV1}
            \label{fig:rcv-unlimited}
        \end{subfigure}\hfill
        \begin{subfigure}[t]{0.32\textwidth}
            \includegraphics[width=\columnwidth]{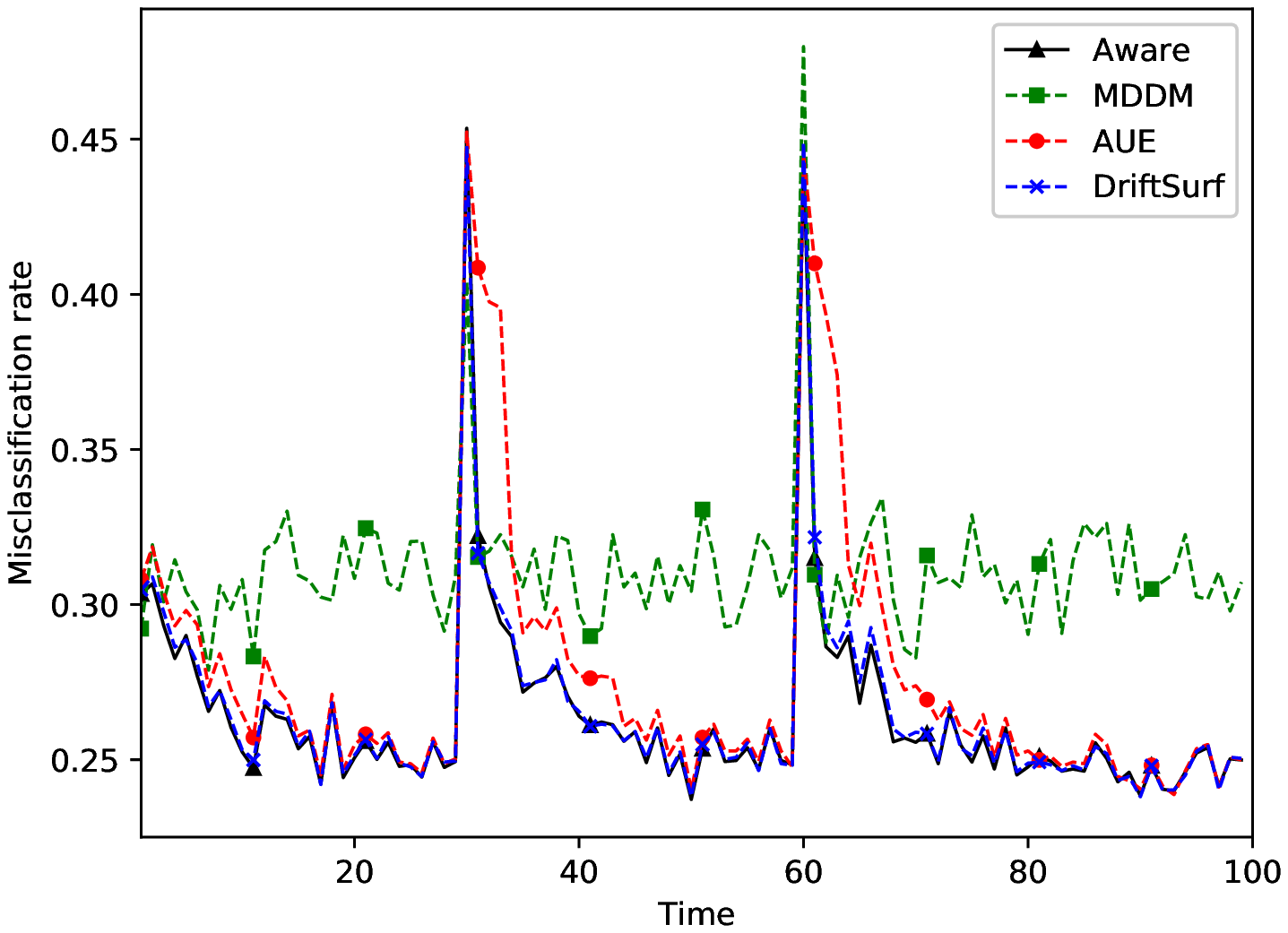}
            \caption{CoverType}
            \label{fig:covtype-unlimited}
        \end{subfigure}\hfill
        
        \begin{subfigure}[t]{0.32\textwidth}
            \includegraphics[width=\columnwidth]{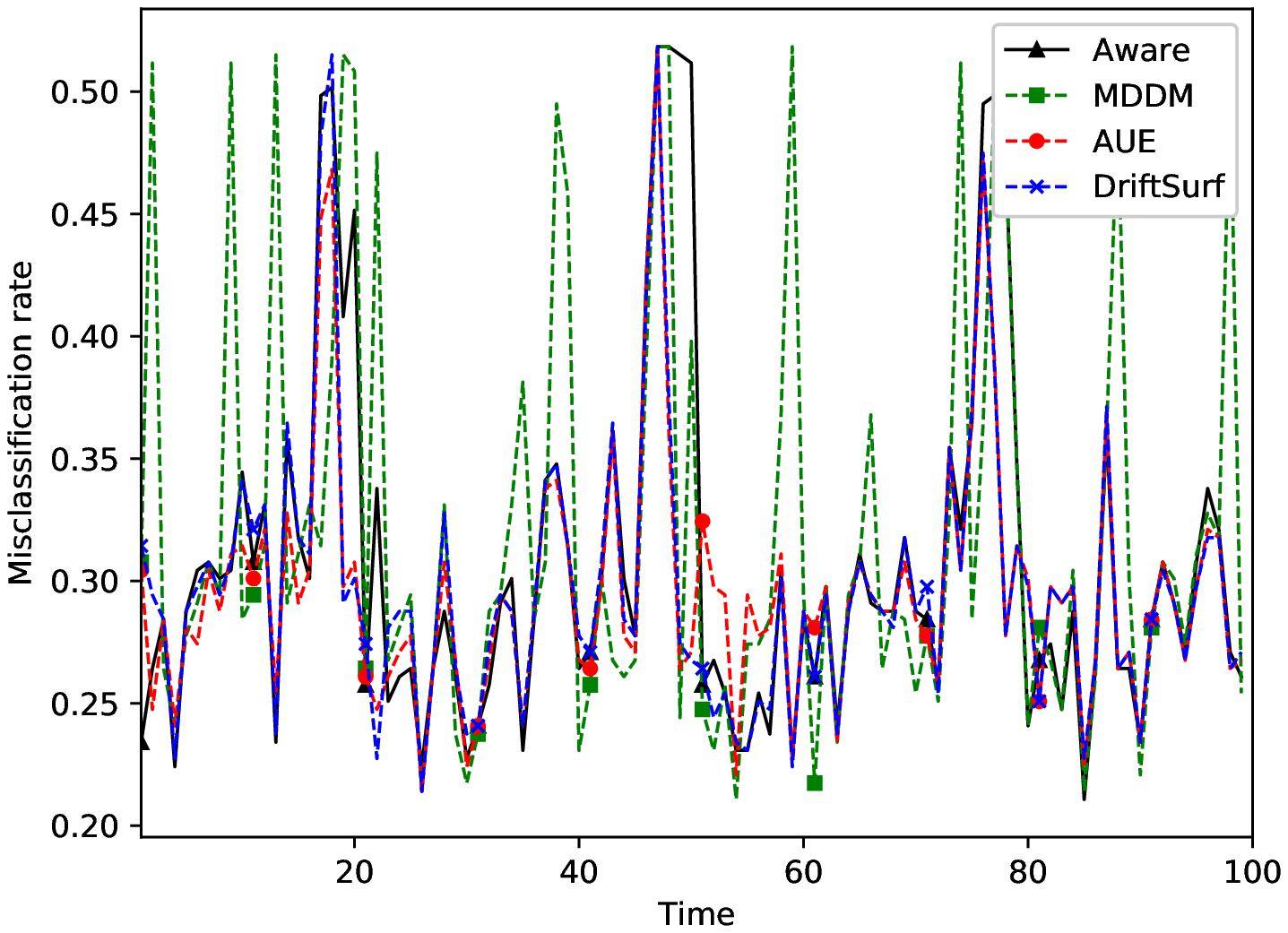}
            \caption{PowerSupply}
            \label{fig:powersupply-unlimited}
        \end{subfigure}\hfill
        \begin{subfigure}[t]{0.32\textwidth}
            \includegraphics[width=\columnwidth]{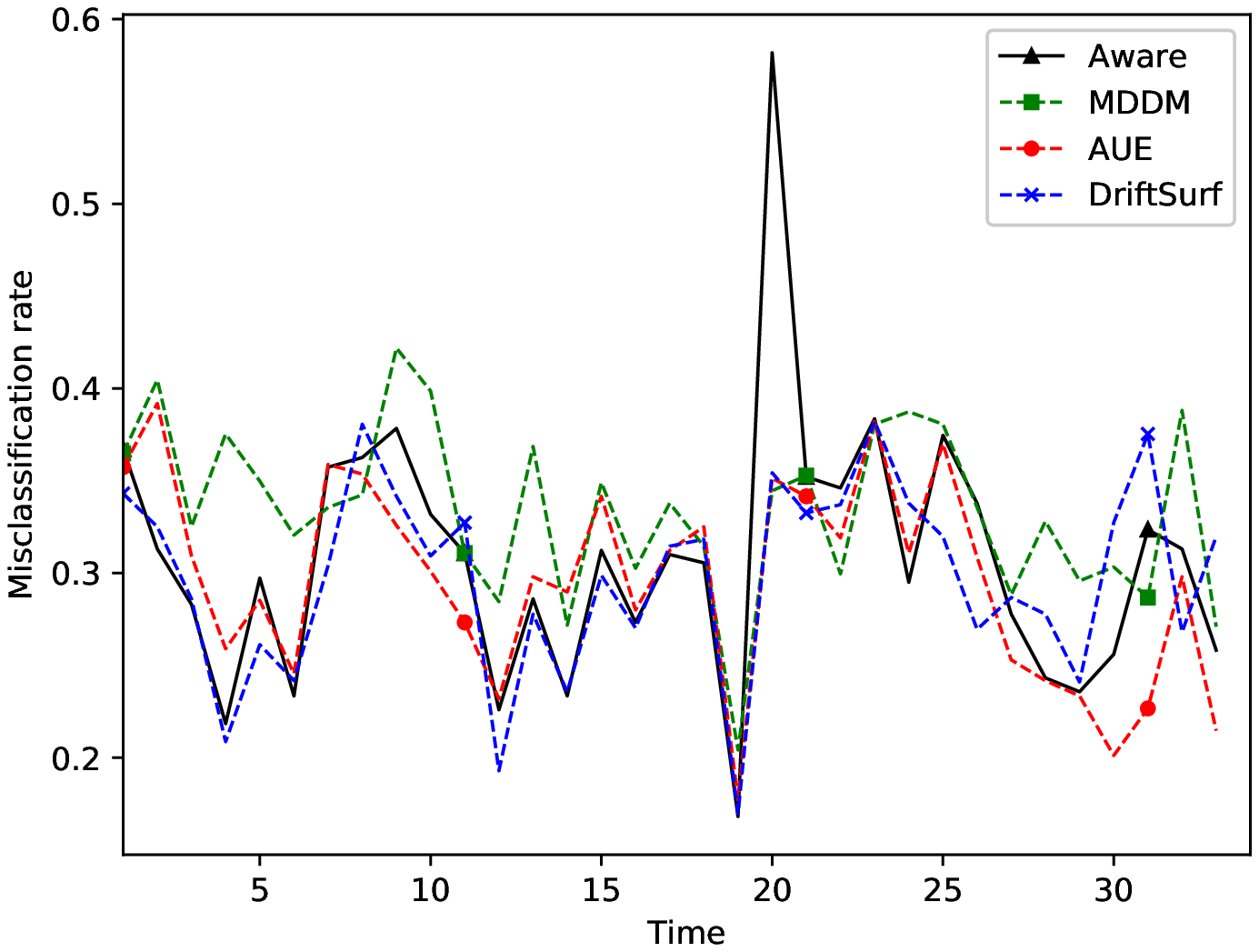}
            \caption{Electricity}
            \label{fig:elec-unlimited}
        \end{subfigure}\hfill
        \begin{subfigure}[t]{0.32\textwidth}
            \includegraphics[width=\columnwidth]{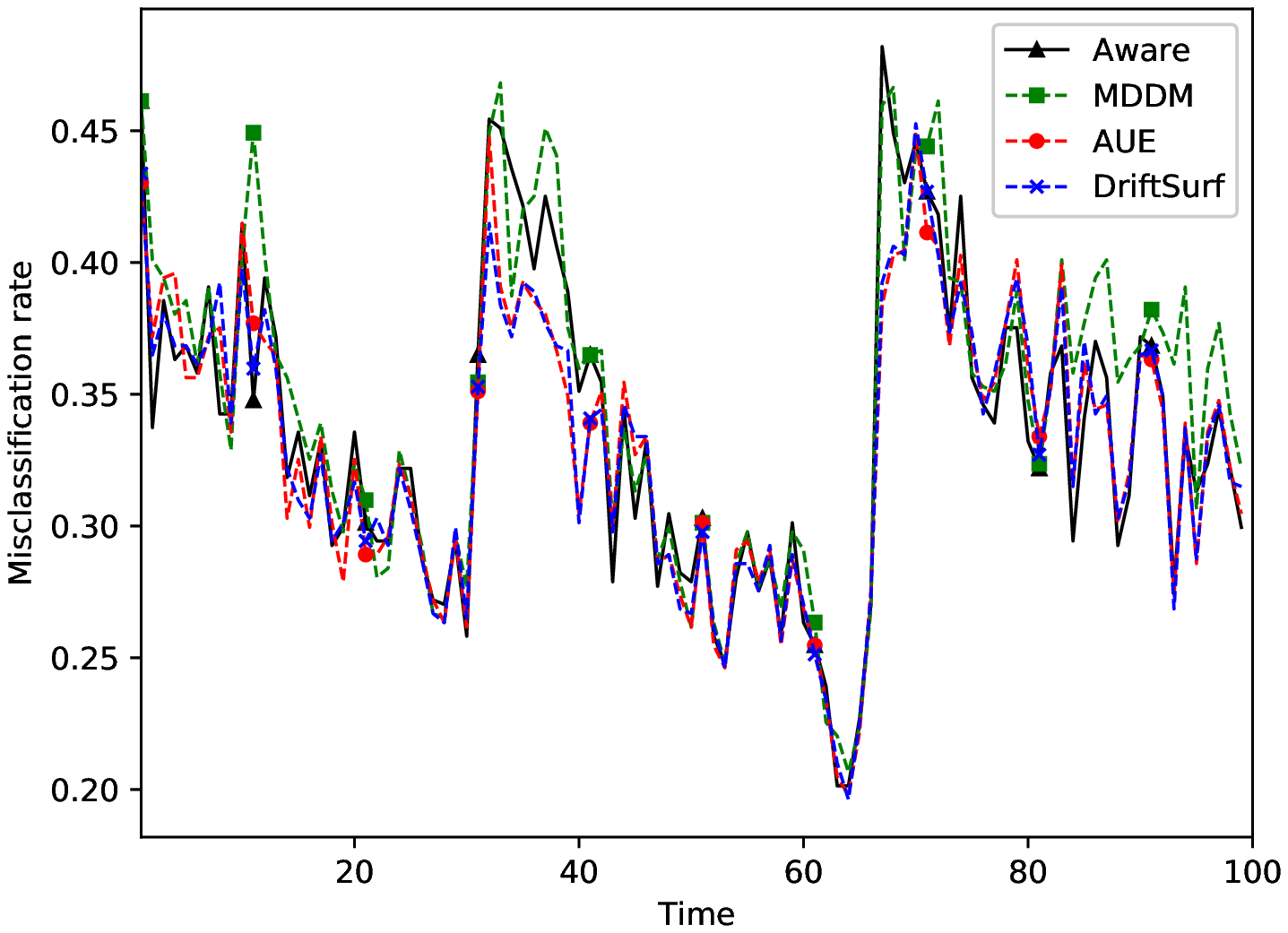}
            \caption{Airline}
            \label{fig:airline-unlimited}
        \end{subfigure}\hfill
    \end{center}
    \caption{Misclassification rate over time ($\rho=2m$ for each model)}
    \label{fig:Misclassification rate over time}
\end{figure*}

\begin{figure}[h!]
\begin{center}
\centerline{\includegraphics[width=0.32\columnwidth]{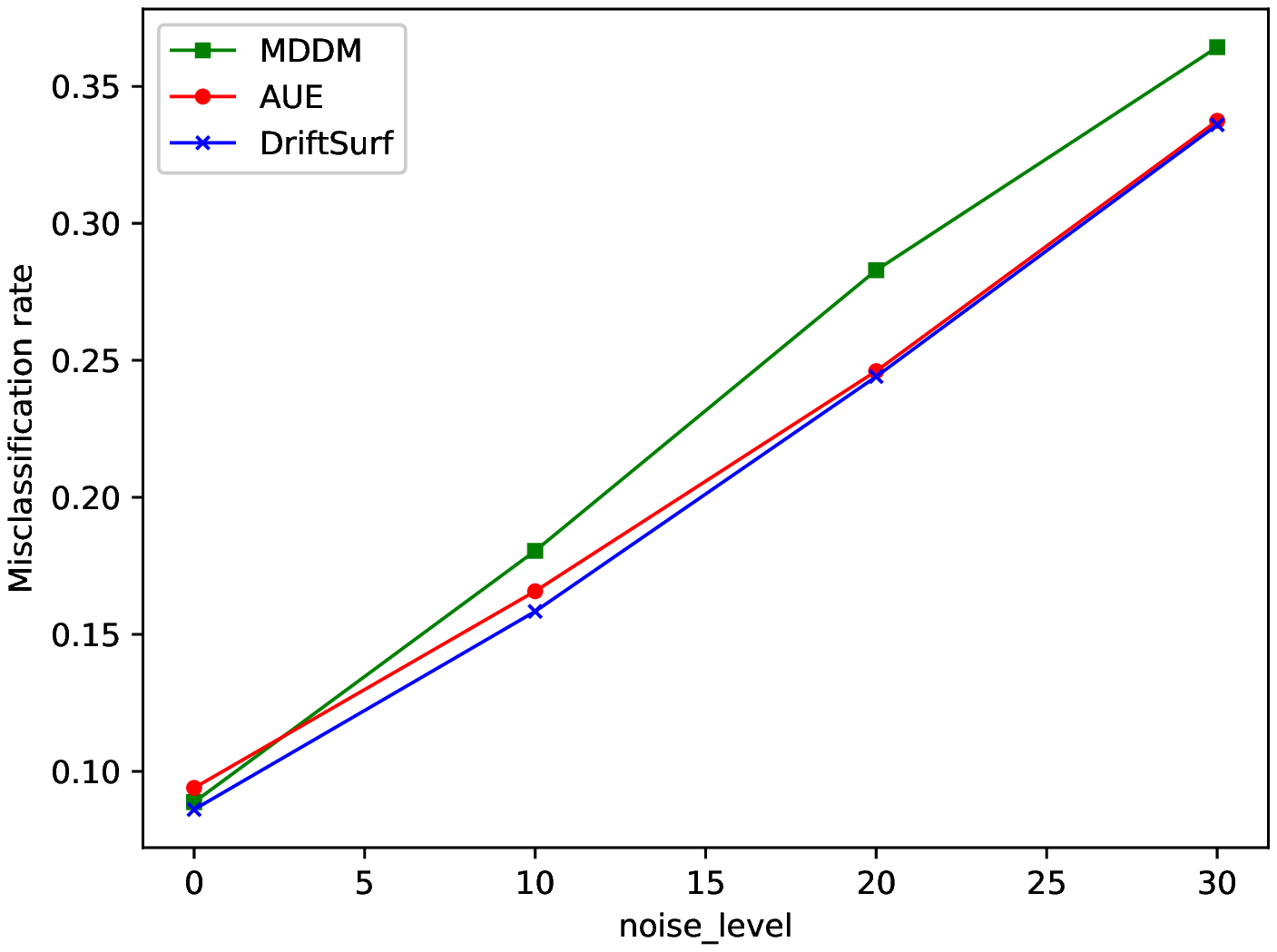}}
\caption{Total average of misclassification rate for SEA dataset with different levels of noise}
\label{fig:sea-noise}
\end{center}
\end{figure}

Table \ref{table:ave-misclassification-app} includes results for MDDM-G (what we use generally for MDDM), as well as two other MDDM variants, MDDM-A and MDDM-E, for a more thorough comparison. The average misclassification rates were similar across each dataset, with no single MDDM variant that consistently outperformed the others. Given the poor peformance of MDDM on CoverType, we re-did the experiment on CoverType with two other drift detection methods, DDM \cite{gama2004learning} and EDDM \cite{baena2006early} to investigate further. In Figure \ref{fig:covtype-MDDM-drift detectors}, we observed DDM accurately detected the two drifts, but EDDM also suffered with continual false positives.

\begin{figure*}[h!]
    \begin{center}
        \begin{subfigure}[t]{0.25\textwidth}
            \includegraphics[width=\columnwidth]{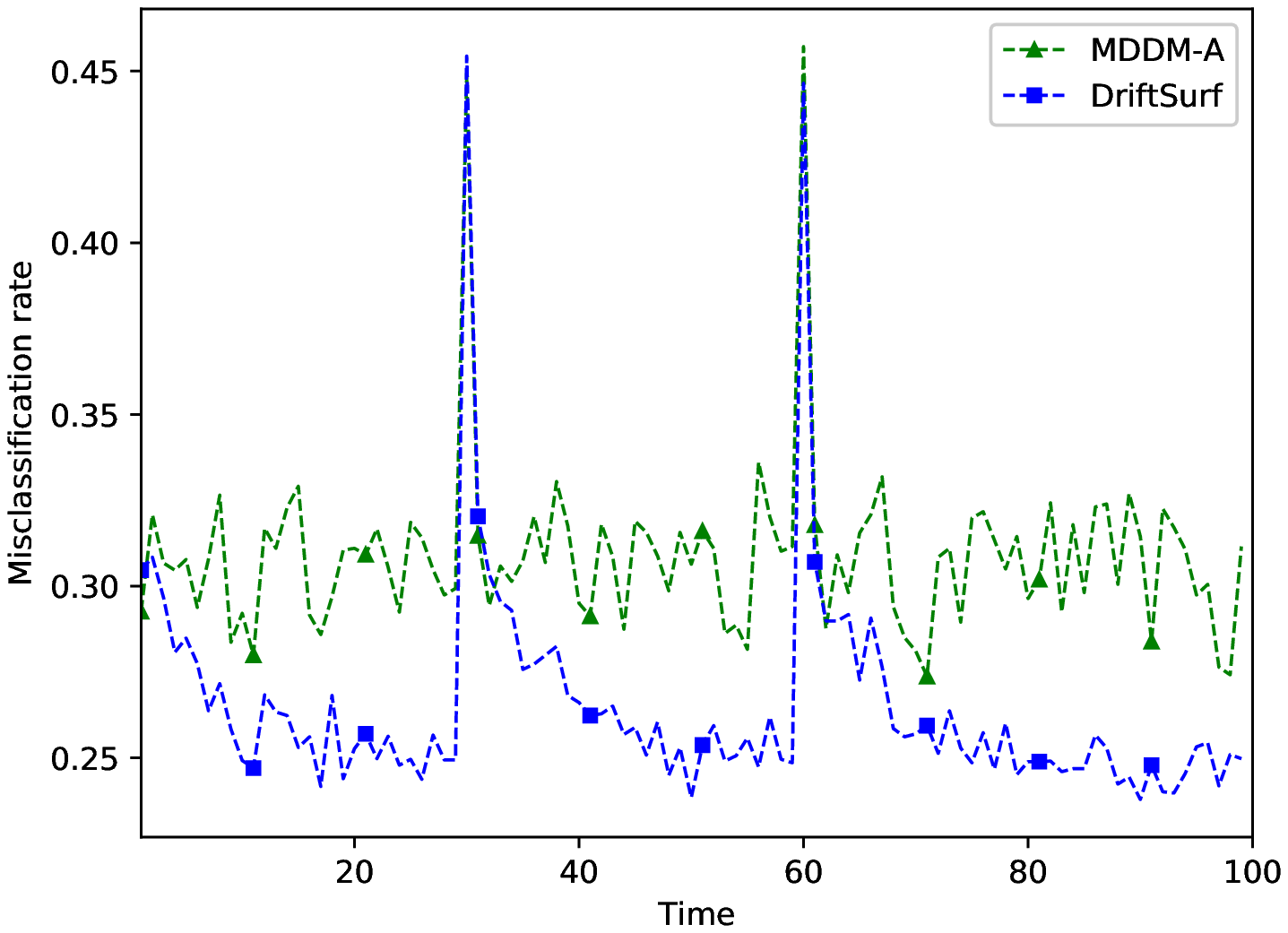}
            \caption{MDDM-A}
            \label{fig:MDDM-A}
        \end{subfigure}\hfill
        \begin{subfigure}[t]{0.25\textwidth}
            \includegraphics[width=\columnwidth]{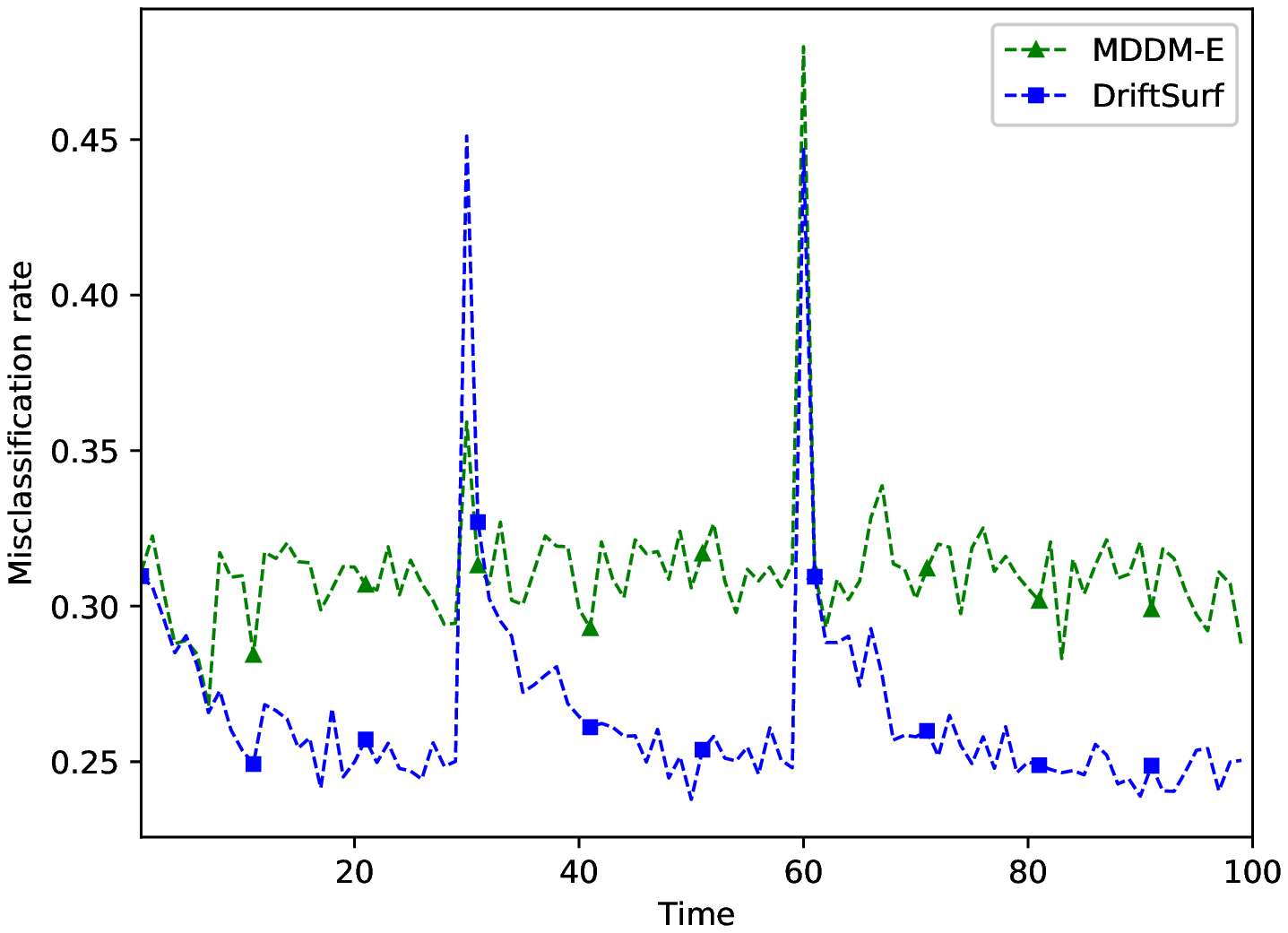}
            \caption{MDDM-E}
            \label{fig:MDDM-E}
        \end{subfigure}\hfill
        \begin{subfigure}[t]{0.25\textwidth}
            \includegraphics[width=\columnwidth]{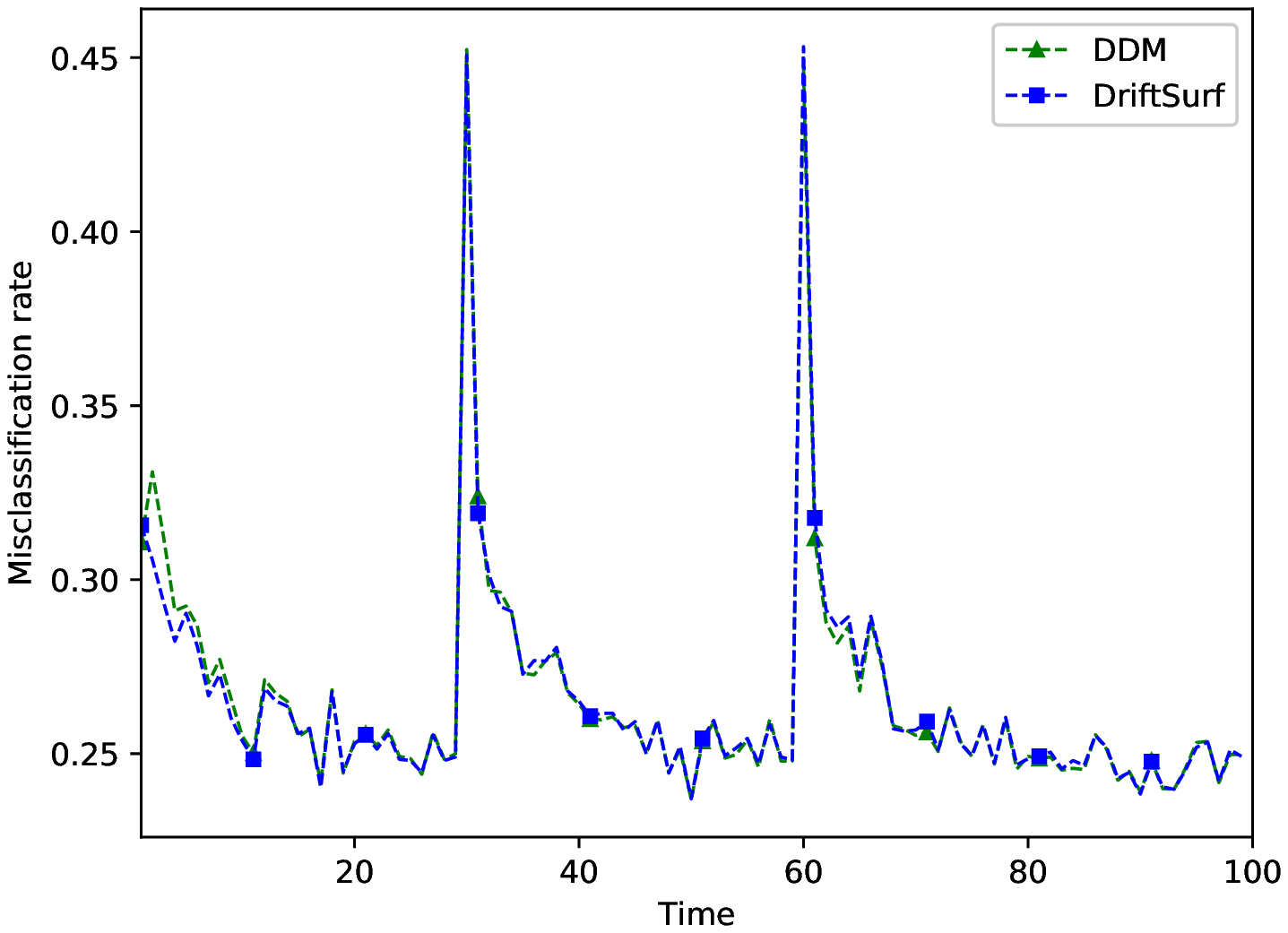}
            \caption{DDM}
            \label{DDM}
        \end{subfigure}\hfill
        \begin{subfigure}[t]{0.25\textwidth}
            \includegraphics[width=\columnwidth]{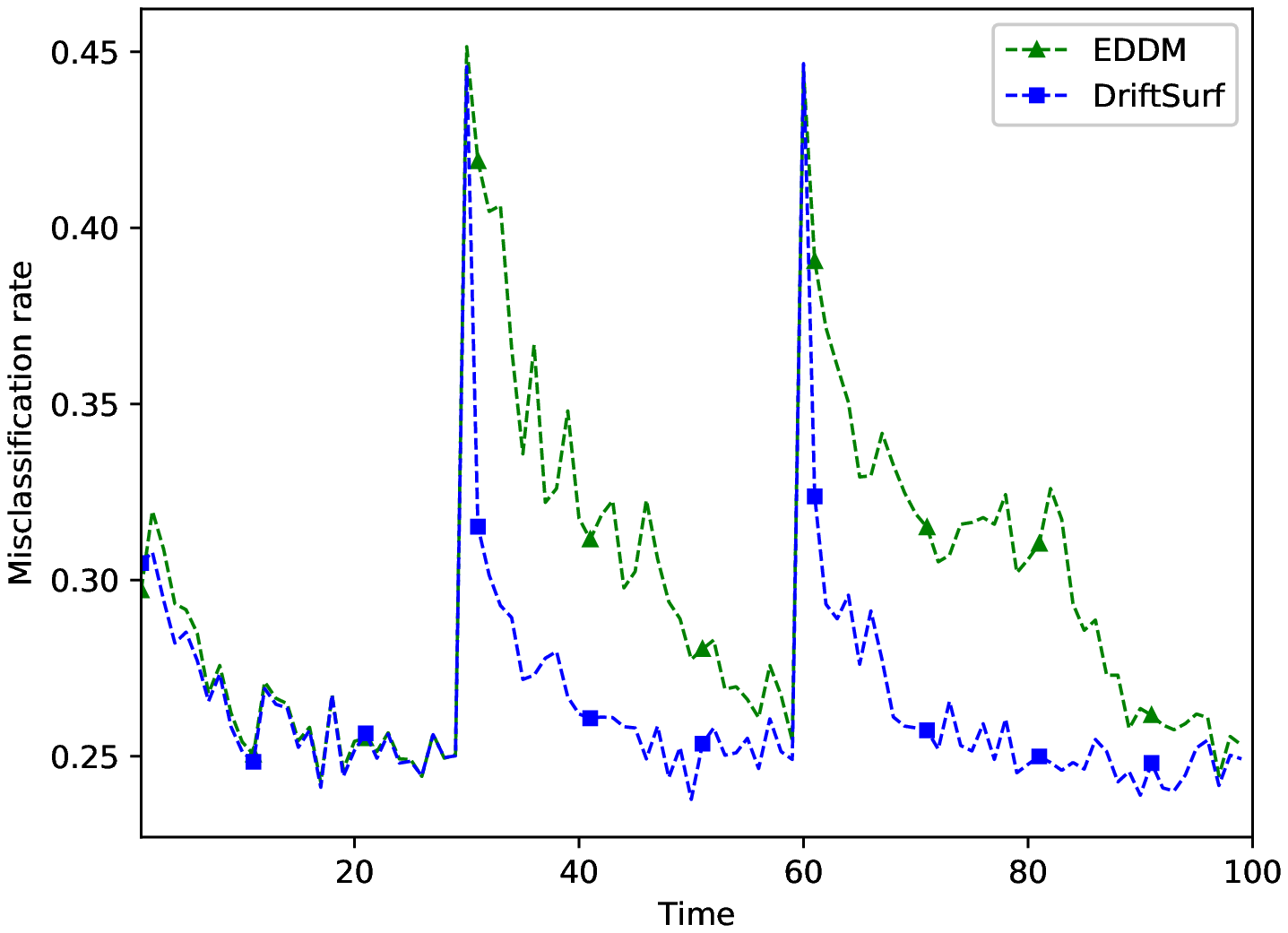}
            \caption{EDDM}
            \label{EDDM}
        \end{subfigure}\hfill
    \end{center}
    \caption{Covertype dataset, different drift detectors ($\rho=2m$ for each model)}
    \label{fig:covtype-MDDM-drift detectors}
\end{figure*}

\subsection{Equal Computational Power for each Algorithm}
\label{sec:expt-results-equal-alg}

Next, we present results for the training strategy where each algorithm has access to $\rho$ update steps in total that are divided among all its models so that the computation time of each algorithm is identical. For the case $\rho=4m$, the misclassification rate at each time step is shown in Figure \ref{fig:Misclassification rate over time-rho=4m} for the comparison of \dsurf, Aware, MDDM, and AUE and in Figure \ref{fig:Misclassification rate over time-rho=4m-candor} for the additional algorithmic comparisons against two ensemble methods, AUE ($k=2$) and Candor. The average over time is in Table \ref{table:ave-misclassification-rho=4m}. For the case $\rho=2m$, the misclassification rate at each time is shown in Figure \ref{fig:Misclassification rate over time-rho=2m}, and the average over time is in Table \ref{table:ave-misclassification-rho=2m}.

Let us discuss a few differences from the previous case where each model was trained with $\rho$ steps. We generally observe lower relative accuracy for AUE, and especially so after drifts. (The exceptions are on Circles and PowerSupply, where the extra training iterations do not matter as much; compare to the fast convergence of Aware after a reset.) This is because AUE is an ensemble of 10 models, and so each model is trained with only 1/5 of the steps that the models of \dsurf get, and only 1/10 of the models for MDDM and \aware. \dsurf now dominates AUE in average misclassification rate on each dataset except for PowerSupply.

We observe \dsurf compares favorably to MDDM on the same datasets as it did in the undivided $\rho$ case. However, MDDM's advantages are magnified on SINE1 and RCV1, the datasets with sharp drifts that were clear to detect, and when immediate switching to the new model was desired. On  PowerSupply, we observe that the false positives are not as punitive for MDDM as before, because its relative additional training per model means that its new models catch up faster. For Hyperplane, the relative additional training for MDDM was advantageous in the $\rho = 4m$ case, but in the $\rho=2m$ case, the advantage of MDDM on Hyperplane-slow vanished and it was comparable to \dsurf. We suspect that when fewer computational steps are available, it is no longer desirable to create new models (which take longer to warm up) so frequently as MDDM did in the $\rho=4m$ case where it outperformed \dsurf.

In Tables \ref{table:ave-misclassification-rho=4m} and \ref{table:ave-misclassification-rho=2m}, we present results for a variation on AUE that is limited to only two experts, which we refer to as AUE ($k=2$). In our comparison of each algorithm when enforcing equal computation time, dividing the $\rho$ steps equally among a total of ten experts in the original AUE is unsurprisingly detrimental to its performance. An alternative comparison is to reduce the total number of experts so that in AUE ($k=2$), each of the two experts is updated with $\rho=2m$ steps, identical to \dsurf. We observe that AUE ($k=2$) performs better than AUE on four datasets: Hyperplane-slow, Hyperplane-fast, SINE1, and Electricity. We previously mentioned that for Hyperplane, the continual drift means always using the latest available model works well, and we mentioned that for Electricity, the drift that does not require adaptation means always using the oldest available model works well. Therefore, on these datasets, the additional eight experts of the original AUE have little utility and AUE ($k=2$) performs better. The reason for improvement of AUE ($k=2$) on SINE1 is less clear, but we suspect that the additional experts of the original AUE penalize the accuracy immediately after the abrupt drifts when it is desirable to assign the most weight to the newest expert.

In Tables \ref{table:ave-misclassification-rho=4m} and \ref{table:ave-misclassification-rho=2m}, we present results for another ensemble method Candor, which is better suited for the setting studied in this section normalizing the computational power because it only requires training a single model at a time. Another distinctive feature of Candor is that it uses biased regularization during training to anchor the newest model closer to the weighted ensemble average from the previous time step. For these two factors, we expect that Candor is better at adapting to drift at the expense of stationary performance, which is exemplified by its high accuracy on the Mixed, powersupply and the continually drifting Hyperplane datasets and its relative improvement over AUE on some other datasets including Sine1, Circles, Airline, Electricity and Powersupply.

\begin{table}[ht!]
\caption{Total average of misclassification rate ($\rho=4m$ divided among all models of each algorithm)}
\label{table:ave-misclassification-rho=4m}
\begin{center}
\begin{small}
\begin{sc}
\begin{tabular}{l*{6}c}
\toprule
 Dataset    	& \aware  & \dsurf		& MDDM-G  	& AUE 		& AUE ($k$=2) & Candor\\
\midrule
SEA0        	& 0.120 	& \textbf{0.082} & 0.092 		& 0.179 		& 0.226 		& 0.192\\
SEA10       	& 0.179 	& 0.169 		& \textbf{0.160} & 0.218 		& 0.269 		& 0.234\\
SEA20       	& 0.256 	& \textbf{0.246} & 0.258 		& 0.280 		& 0.320  		& 0.283\\
SEA30       	& 0.334 	& \textbf{0.328} & 0.341 		& 0.342 		& 0.365  		& 0.338\\
SEA-gradual  	& 0.170 	& \textbf{0.157} & 0.160 		& 0.215 		& 0.267  		& 0.232\\
Hyper-slow  	& 0.145 	& 0.145 		& 0.132 		& 0.158 		& 0.120		& \textbf{0.103}\\
Hyper-fast  	& 0.222 	& 0.177 		& 0.154		 & 0.238  		& 0.154		& \textbf{0.144}\\
SINE1       	& 0.149 	& 0.194 		& \textbf{0.157} & 0.263  		& 0.181		& 0.159\\
Mixed 	      	& 0.188 	& 0.203 		& 0.200		 & 0.254  		& 0.203		& \textbf{0.182}\\
Circles       	& 0.345 	& 0.369 		& \textbf{0.341} & 0.372  		& 0.424		& 0.360\\
RCV1        	& 0.101 	& 0.127 		& \textbf{0.113} & 0.310 		& 0.404 		& 0.341\\
CoverType   	& 0.260 	& \textbf{0.266} & 0.302 		& 0.301  		&  0.314		& 0.303\\
Airline     		& 0.335 	& \textbf{0.331} & 0.337 		& 0.360 		& 0.366 		& 0.353\\
Electricity 		& 0.310 	& \textbf{0.289} & 0.324 		& 0.348 		& 0.326 		& 0.300\\
PowerSupply 	& 0.303 	& 0.305 		& 0.292 		& 0.284		& 0.393 		& \textbf{0.282}\\
\bottomrule
\end{tabular}
\end{sc}
\end{small}
\end{center}
\end{table}

\begin{figure*}[h!]
    \begin{center}
        \begin{subfigure}[t]{0.32\textwidth}
            \includegraphics[width=\columnwidth]{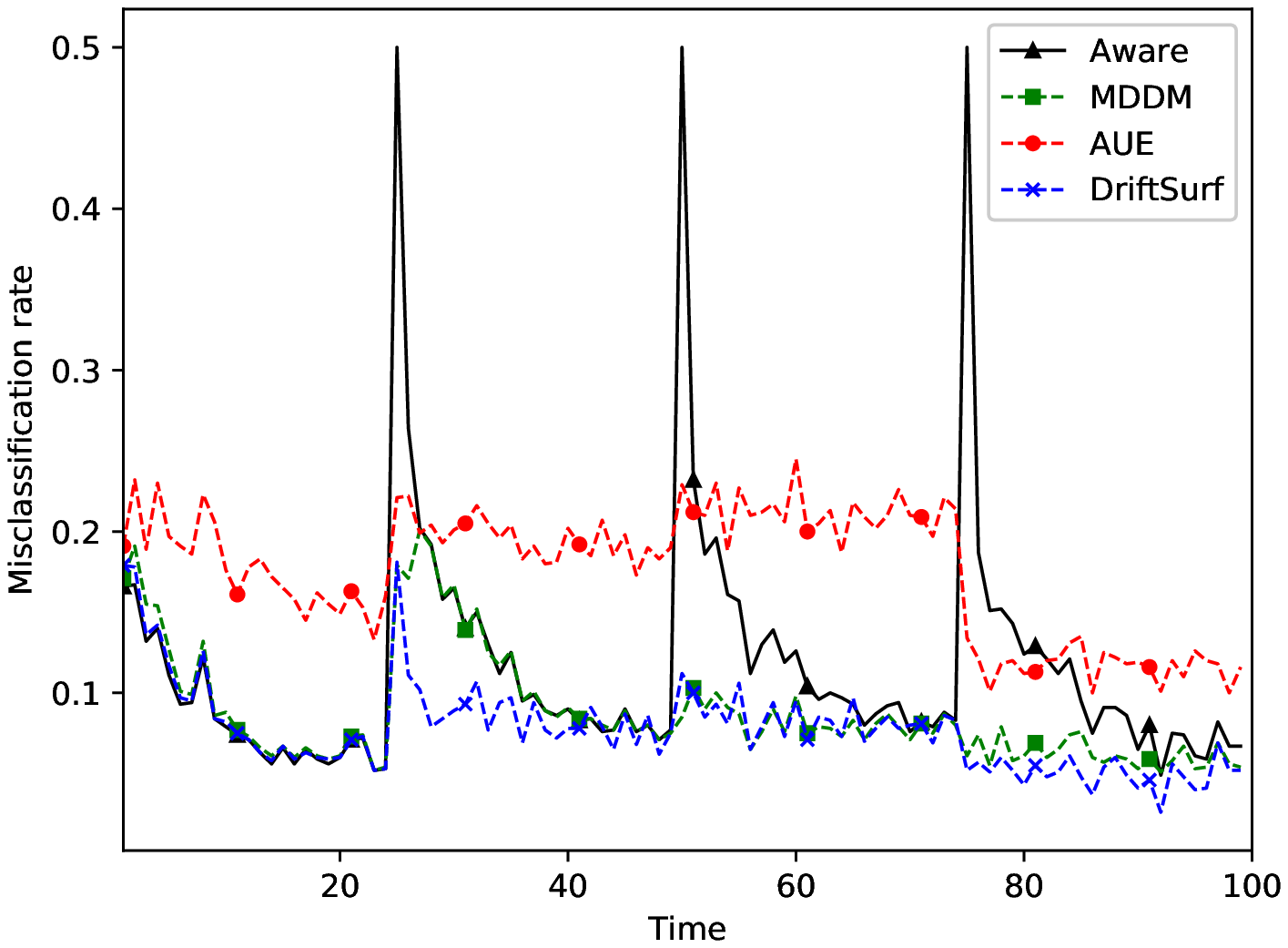}
            \caption{SEA0}
            \label{fig:sea0-limited-4}
        \end{subfigure}\hfill
        \begin{subfigure}[t]{0.32\textwidth}
            \includegraphics[width=\columnwidth]{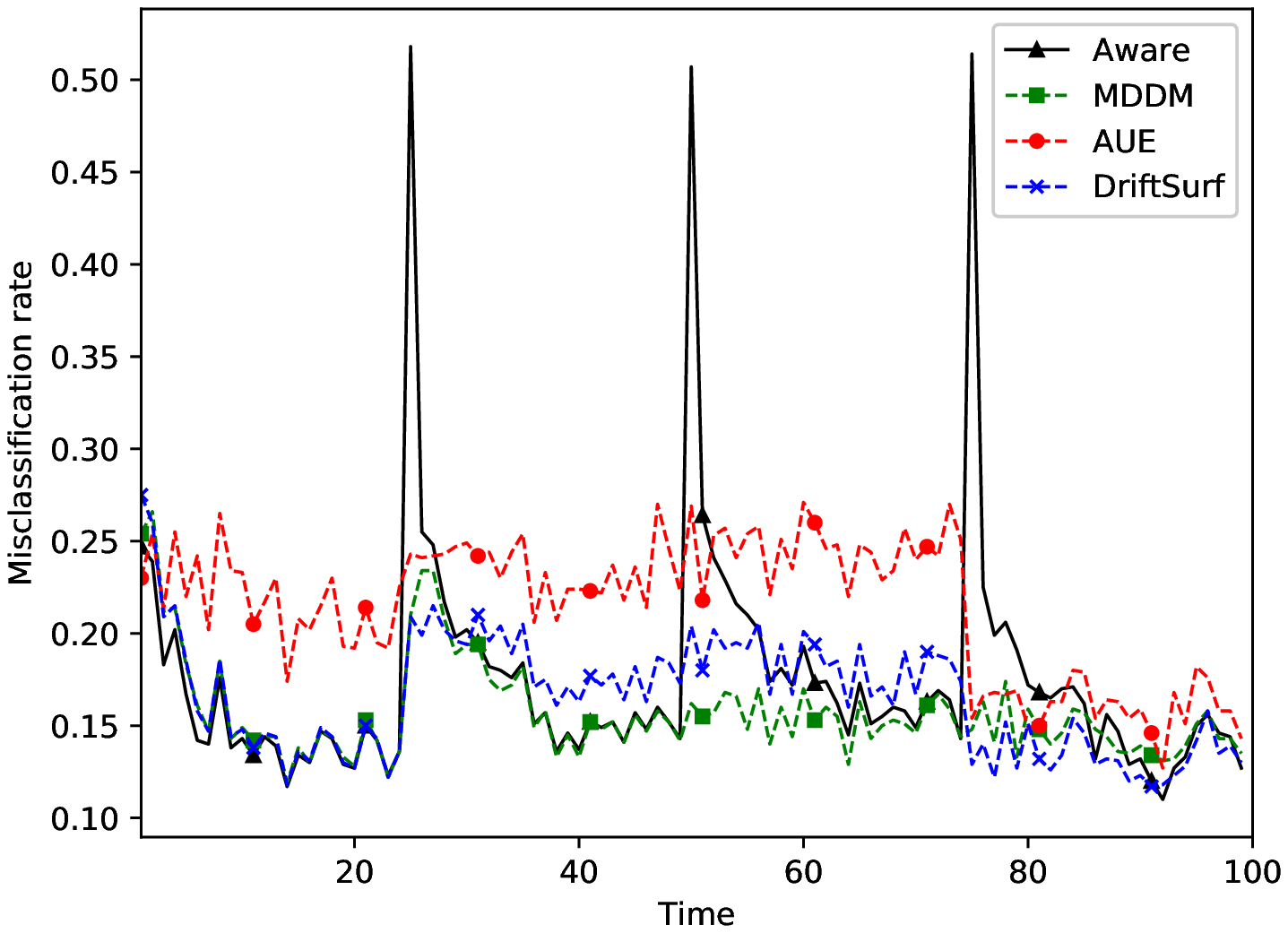}
            \caption{SEA10}
            \label{fig:sea10-limited-4}
        \end{subfigure}\hfill
        \begin{subfigure}[t]{0.32\textwidth}
            \includegraphics[width=\columnwidth]{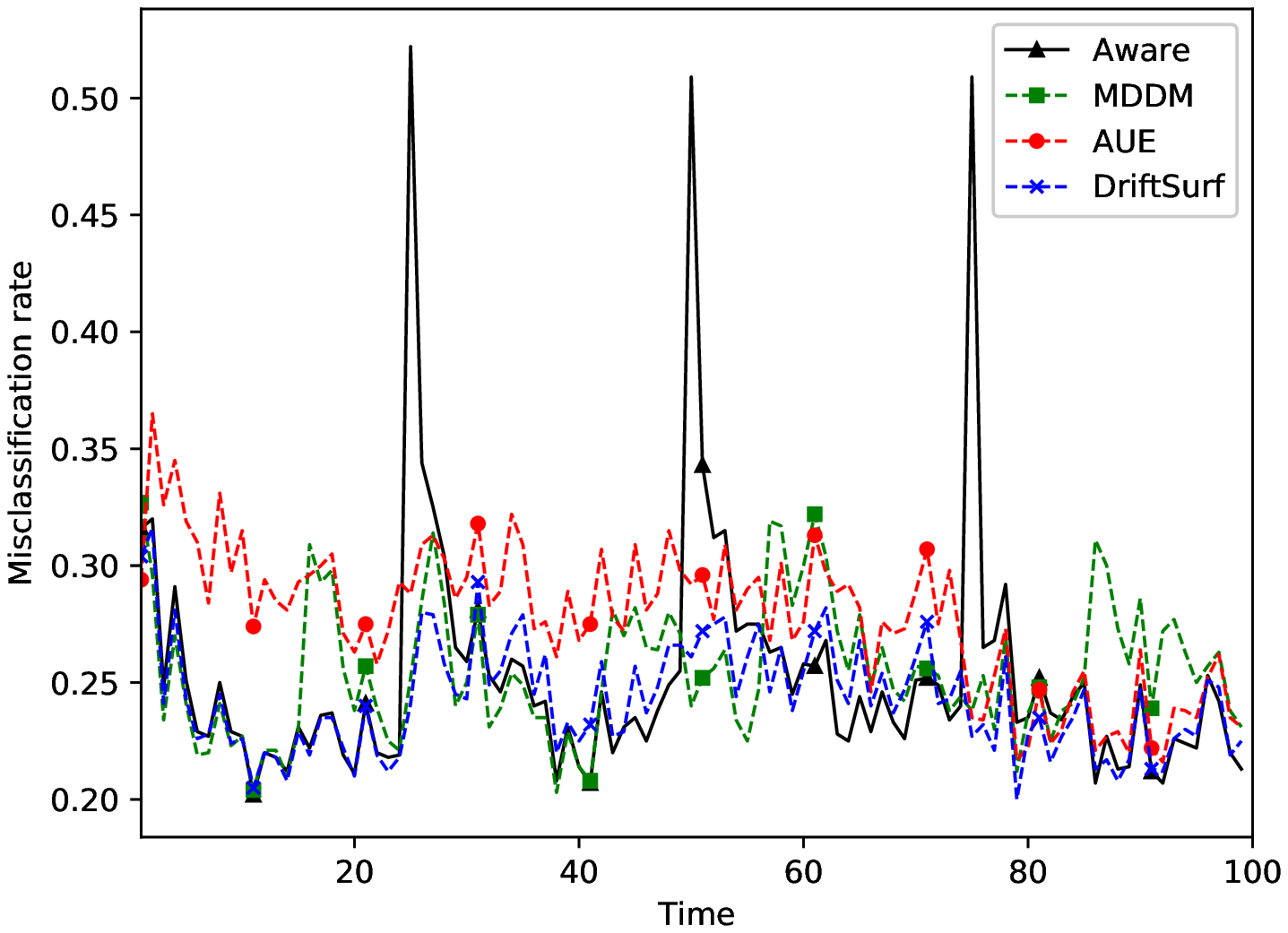}
            \caption{SEA20}
            \label{fig:sea20-limited-4}
        \end{subfigure}\hfill
        
        \begin{subfigure}[t]{0.32\textwidth}
            \includegraphics[width=\columnwidth]{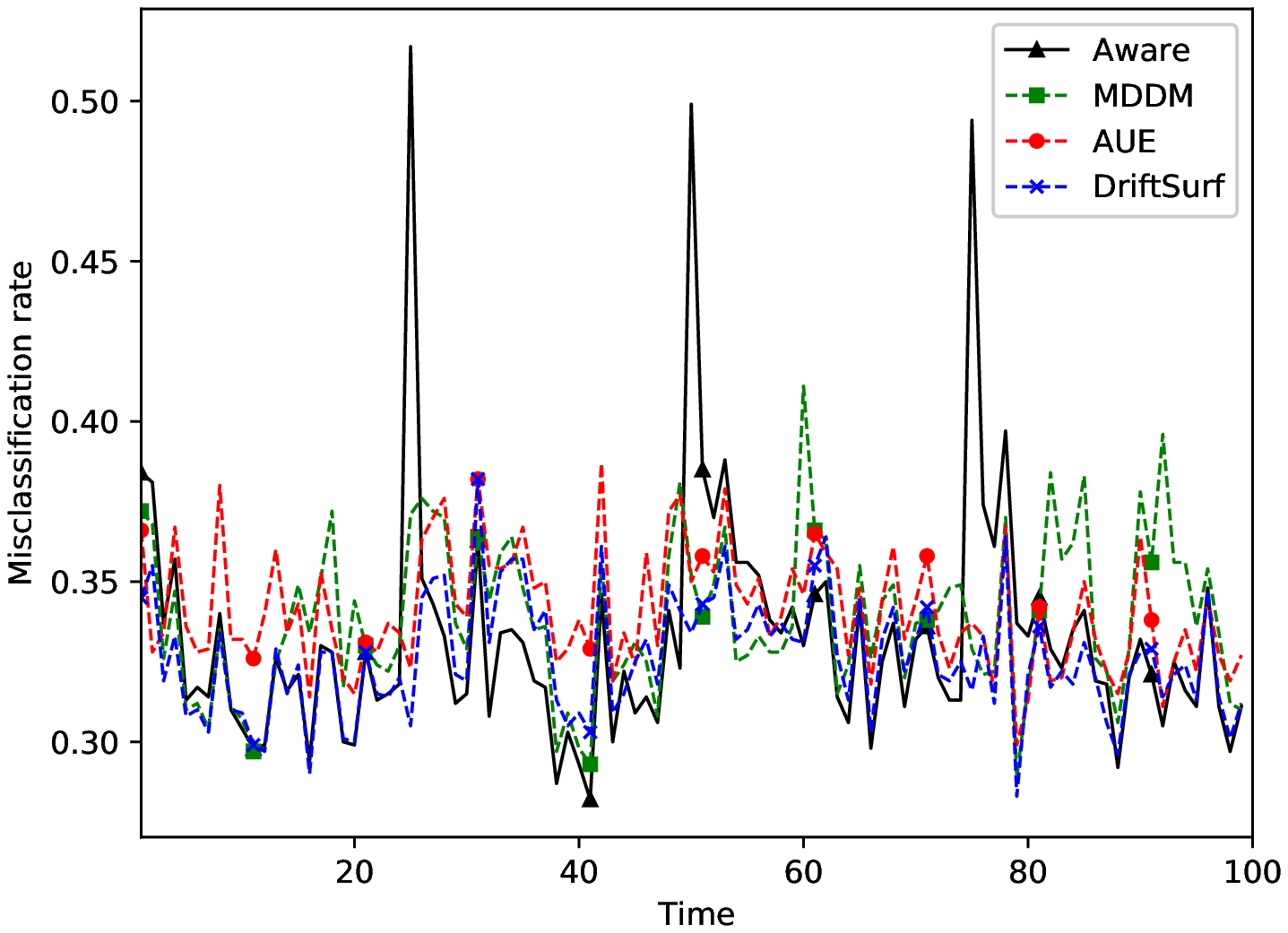}
            \caption{SEA30}
            \label{fig:sea30-limited-4}
        \end{subfigure}\hfill
        \begin{subfigure}[t]{0.32\textwidth}
            \includegraphics[width=\columnwidth]{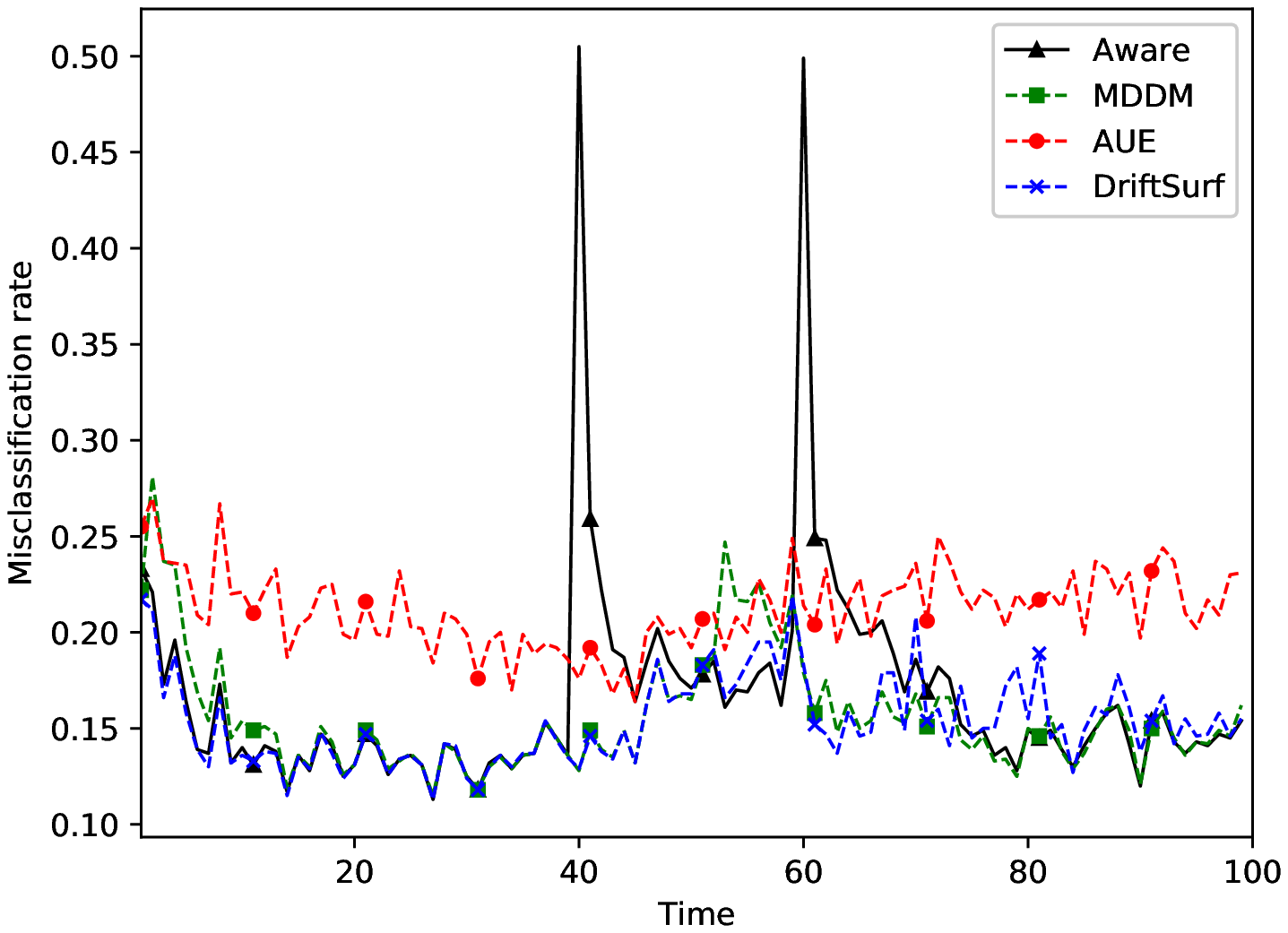}
            \caption{SEA-gradual}
            \label{fig:sea_gradual-limited-4}
        \end{subfigure}\hfill
        \begin{subfigure}[t]{0.32\textwidth}
            \includegraphics[width=\columnwidth]{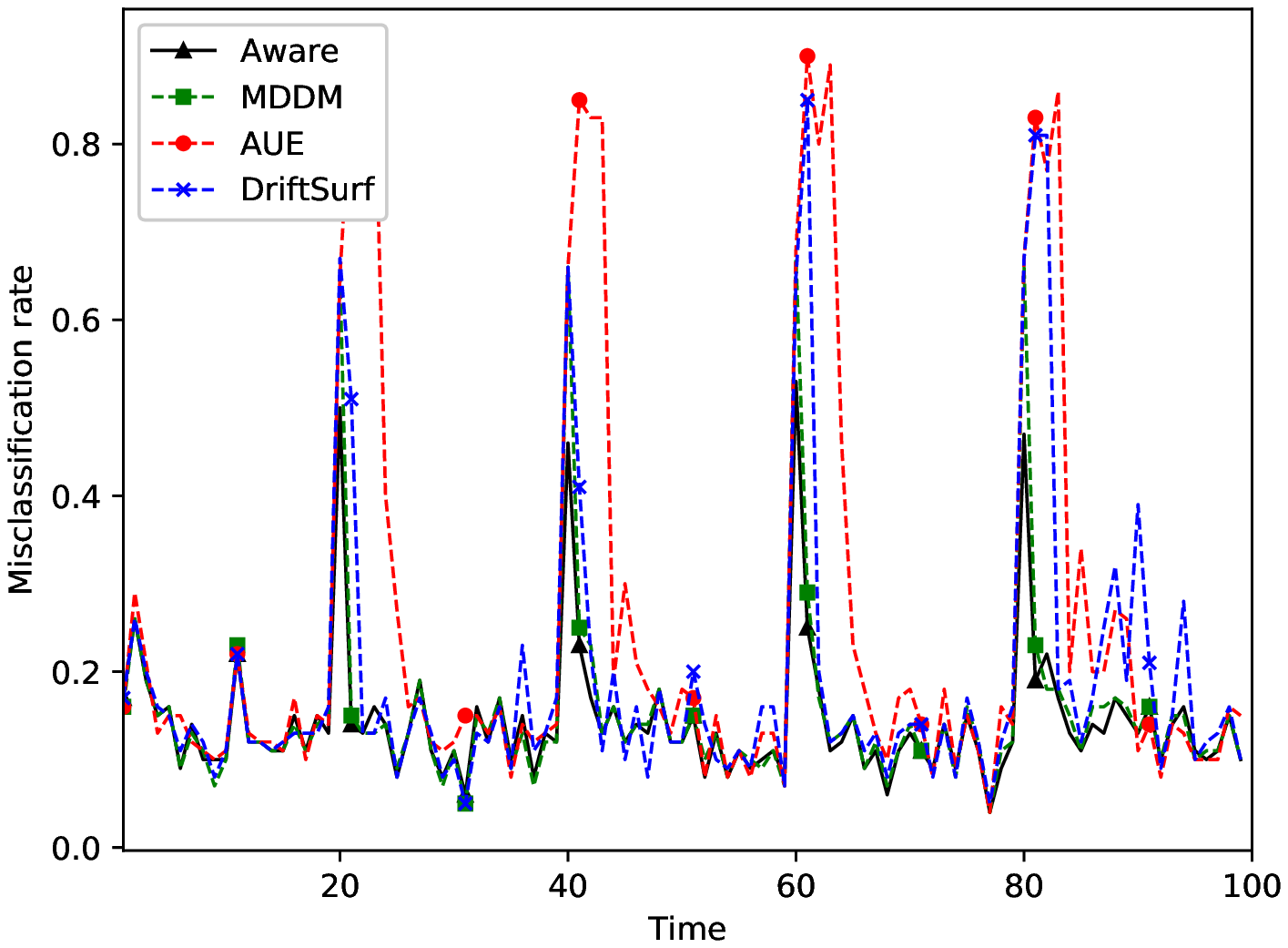}
            \caption{SINE1}
            \label{fig:sine1-limited-4}
        \end{subfigure}\hfill
        
        \begin{subfigure}[t]{0.32\textwidth}
            \includegraphics[width=\columnwidth]{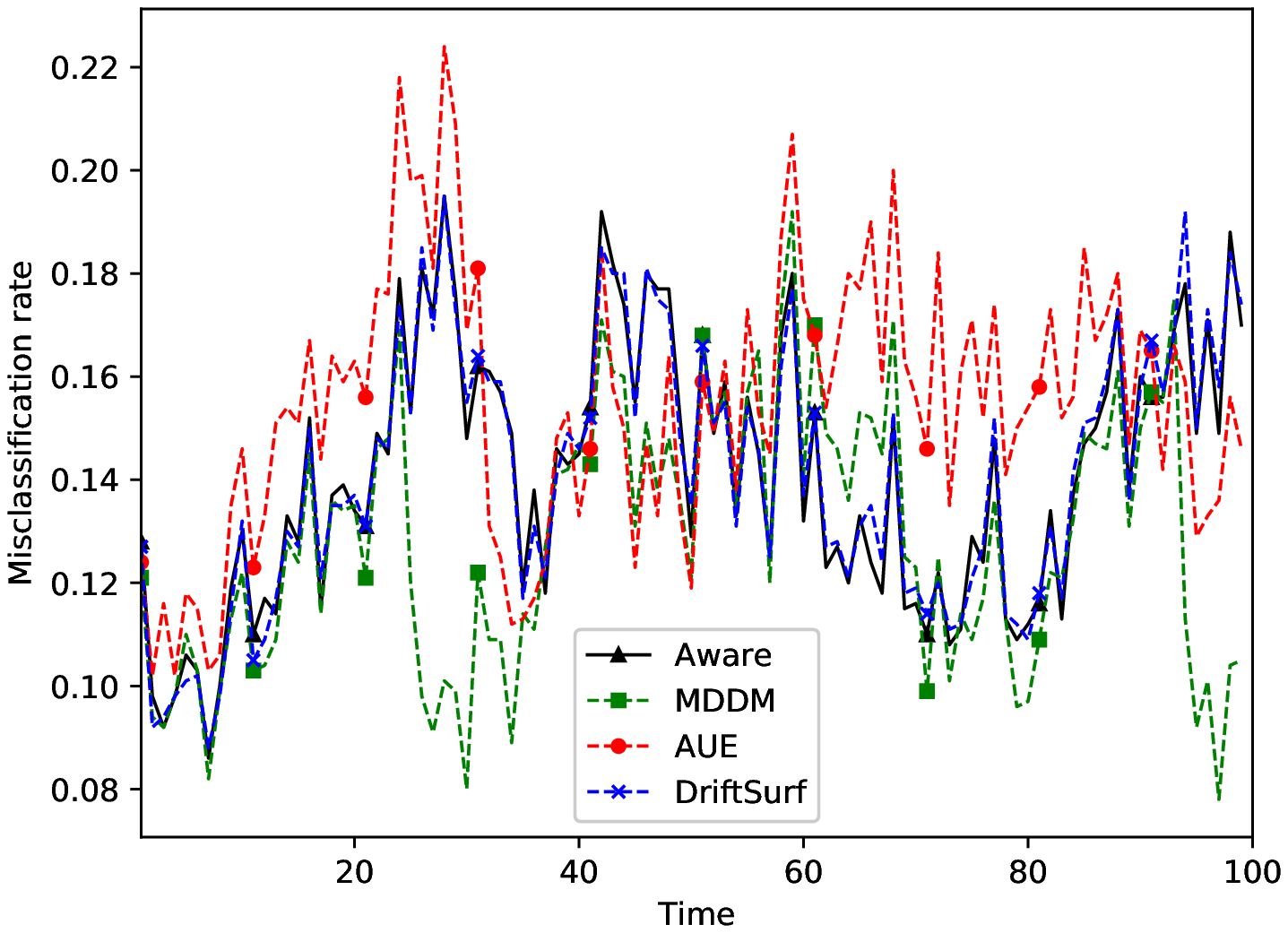}
            \caption{HyperPlane-slow}
            \label{fig:hyperplane-slow-limited-4}
        \end{subfigure}\hfill
        \begin{subfigure}[t]{0.32\textwidth}
            \includegraphics[width=\columnwidth]{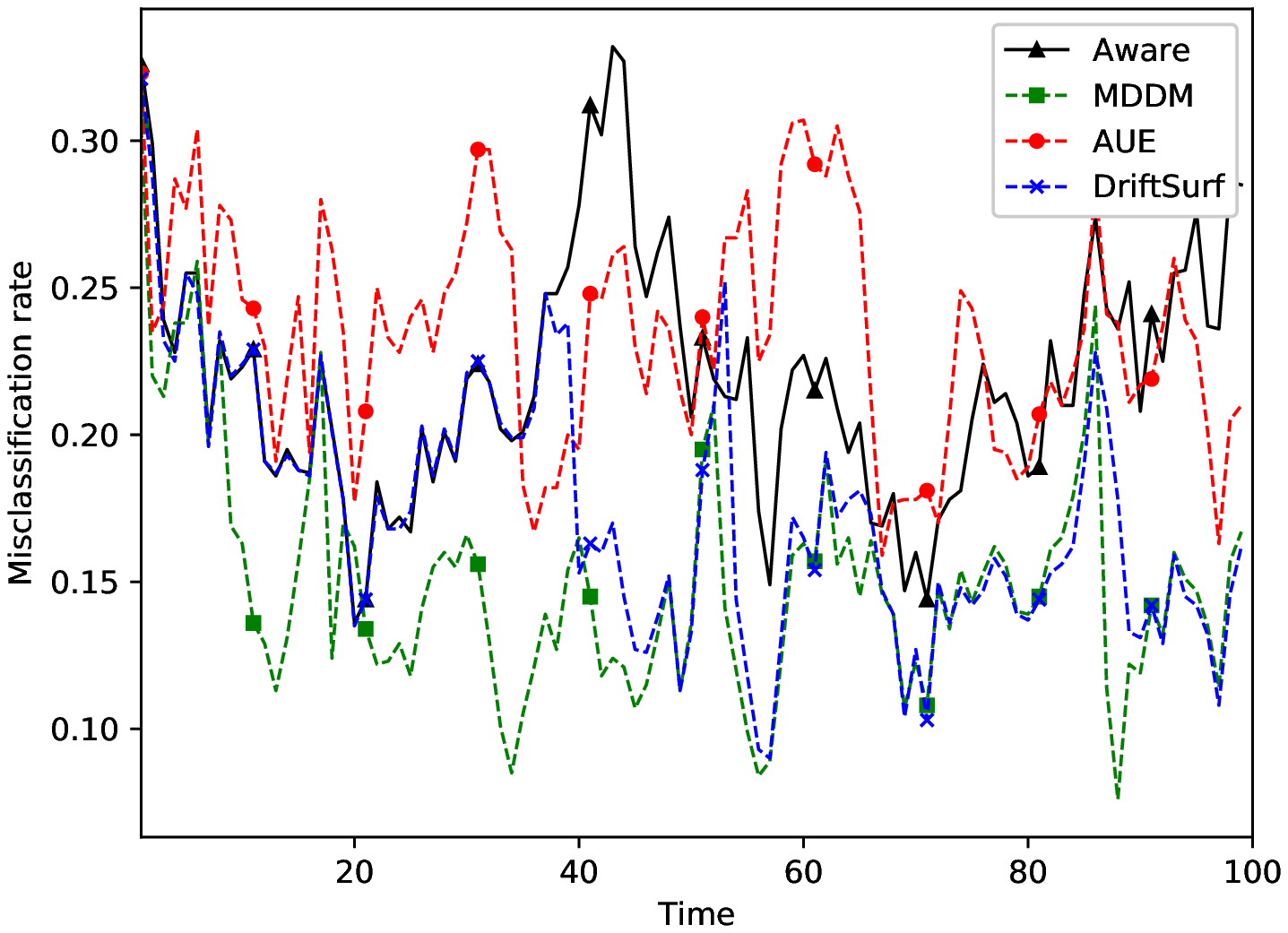}
            \caption{HyperPlane-fast}
            \label{fig:hyperplane-fast-limited-4}
        \end{subfigure}\hfill
        \begin{subfigure}[t]{0.32\textwidth}
            \includegraphics[width=\columnwidth]{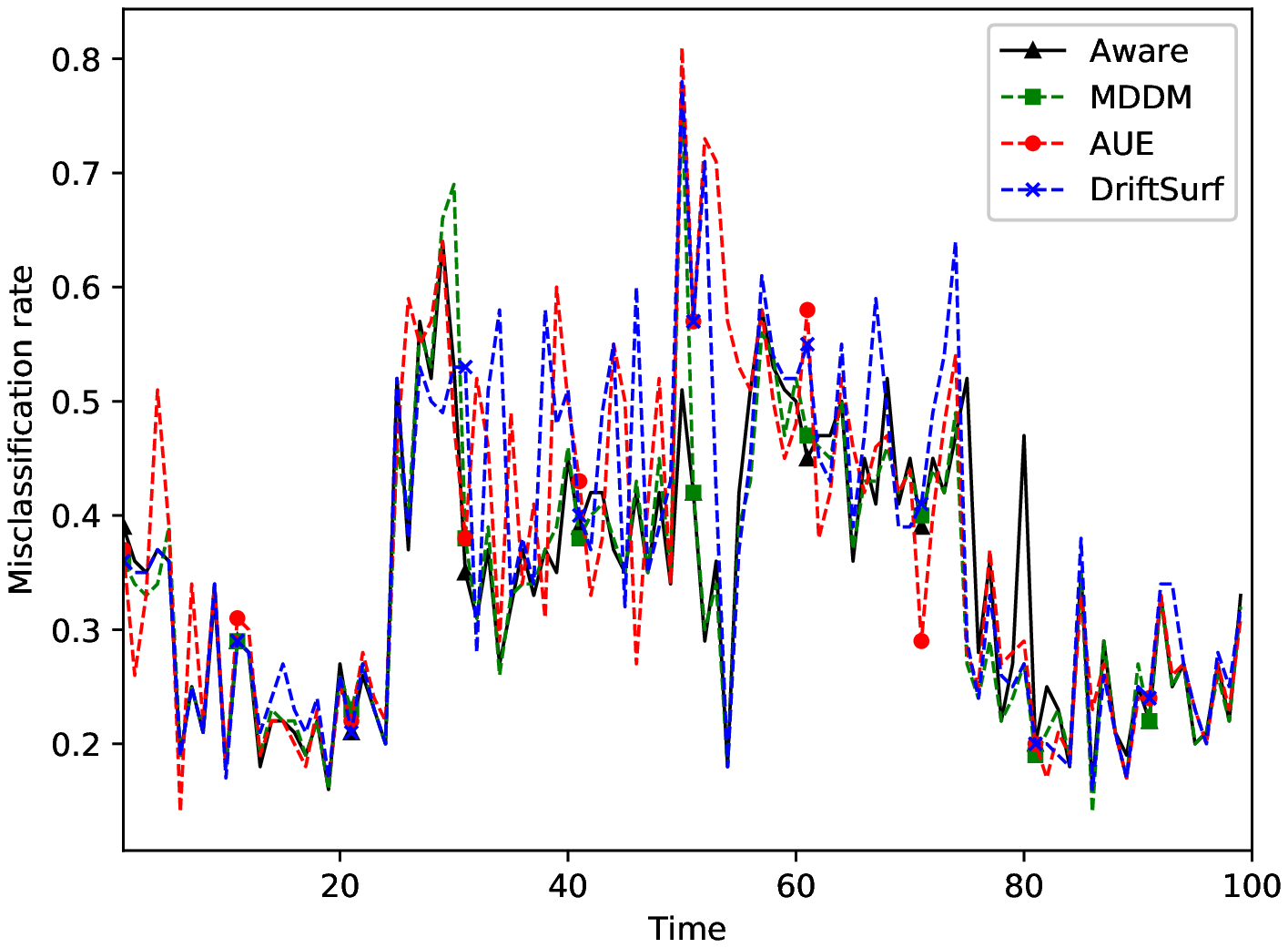}
            \caption{CIRCLES}
            \label{fig:circles-limited-4}
        \end{subfigure}\hfill
        
        \begin{subfigure}[t]{0.32\textwidth}
            \includegraphics[width=\columnwidth]{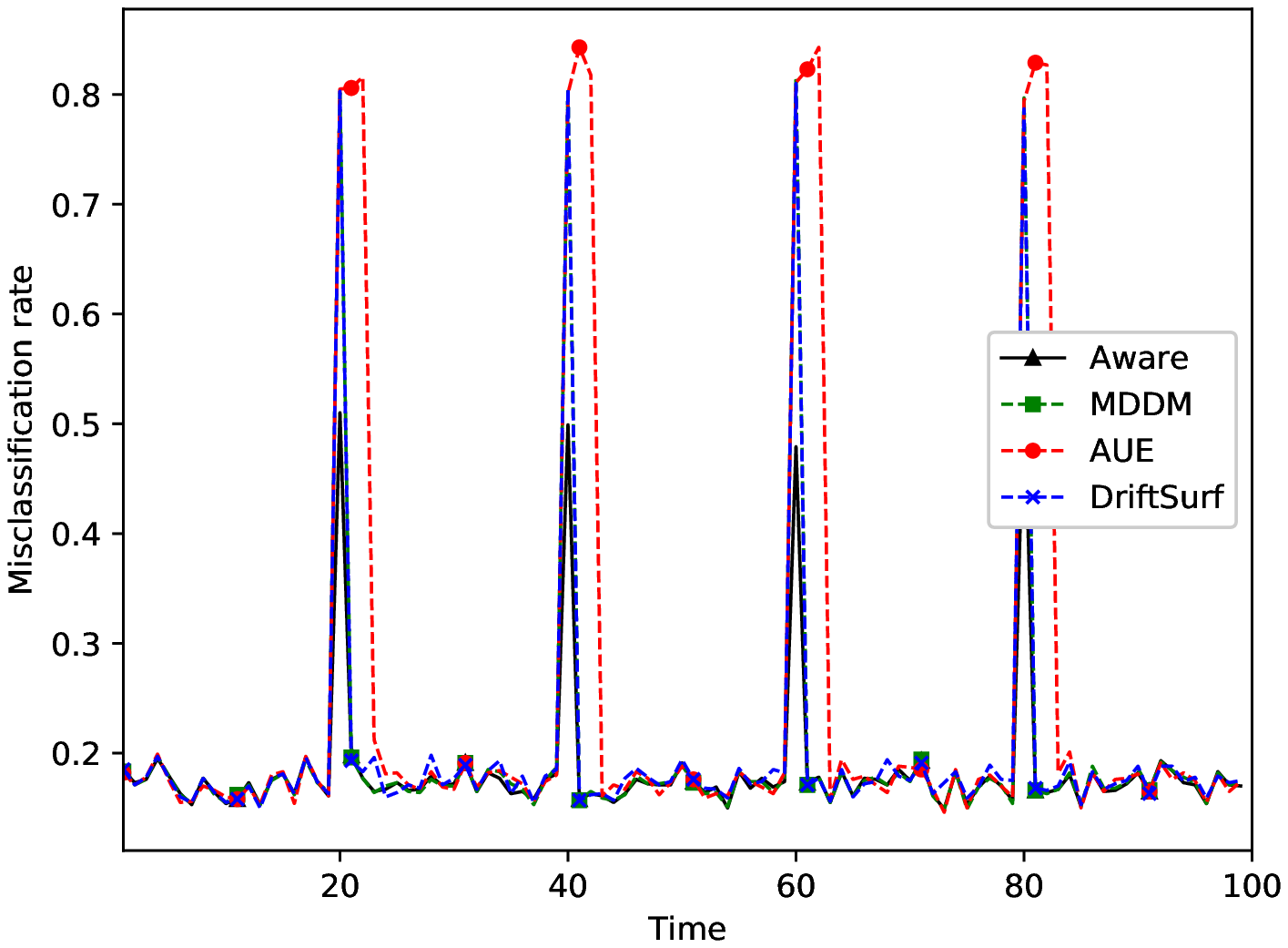}
            \caption{MIXED}
            \label{fig:mixed-limited-4}
        \end{subfigure}\hfill
        \begin{subfigure}[t]{0.32\textwidth}
            \includegraphics[width=\columnwidth]{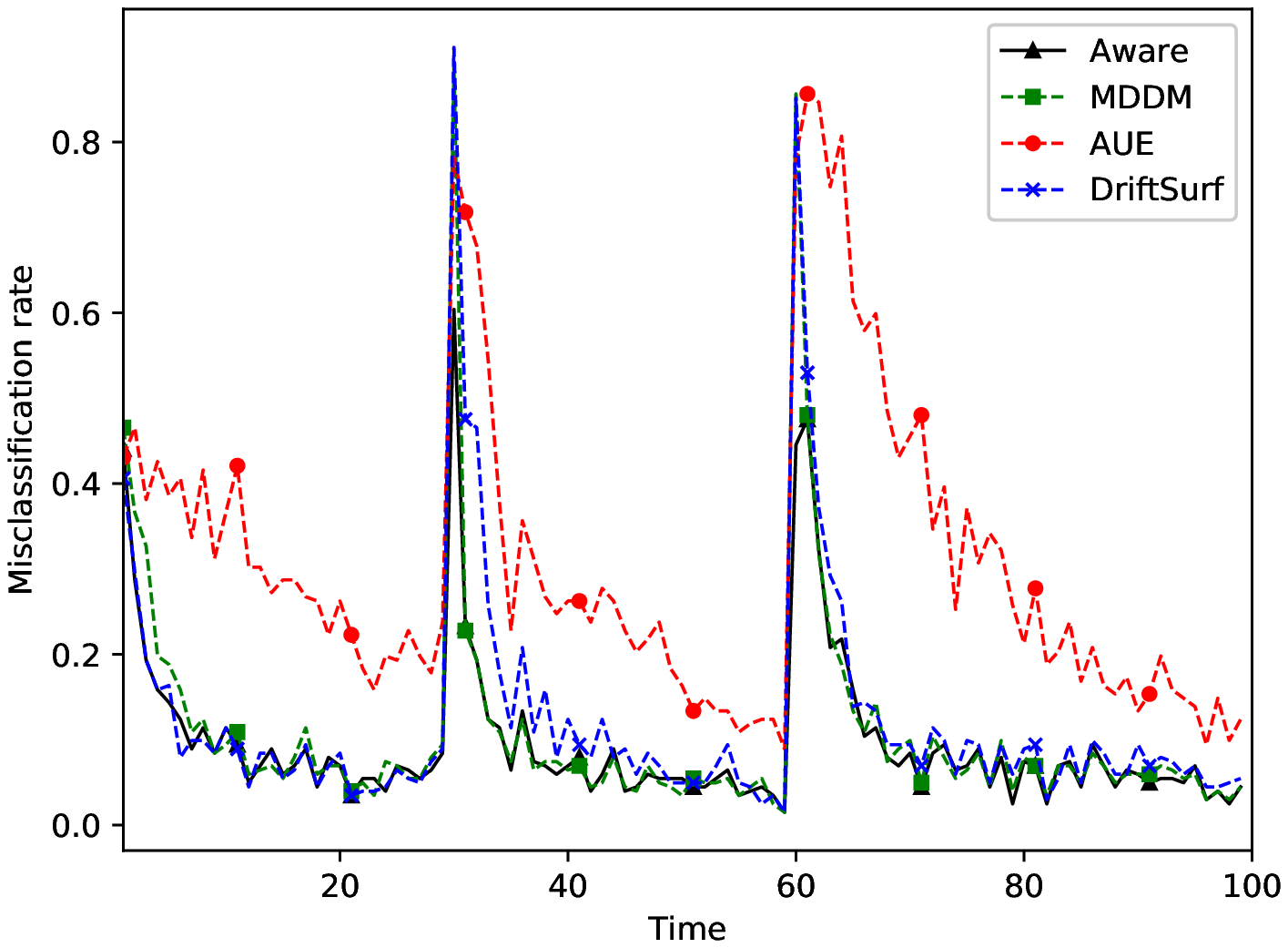}
            \caption{RCV1}
            \label{fig:rcv-limited-4}
        \end{subfigure}\hfill
        \begin{subfigure}[t]{0.32\textwidth}
            \includegraphics[width=\columnwidth]{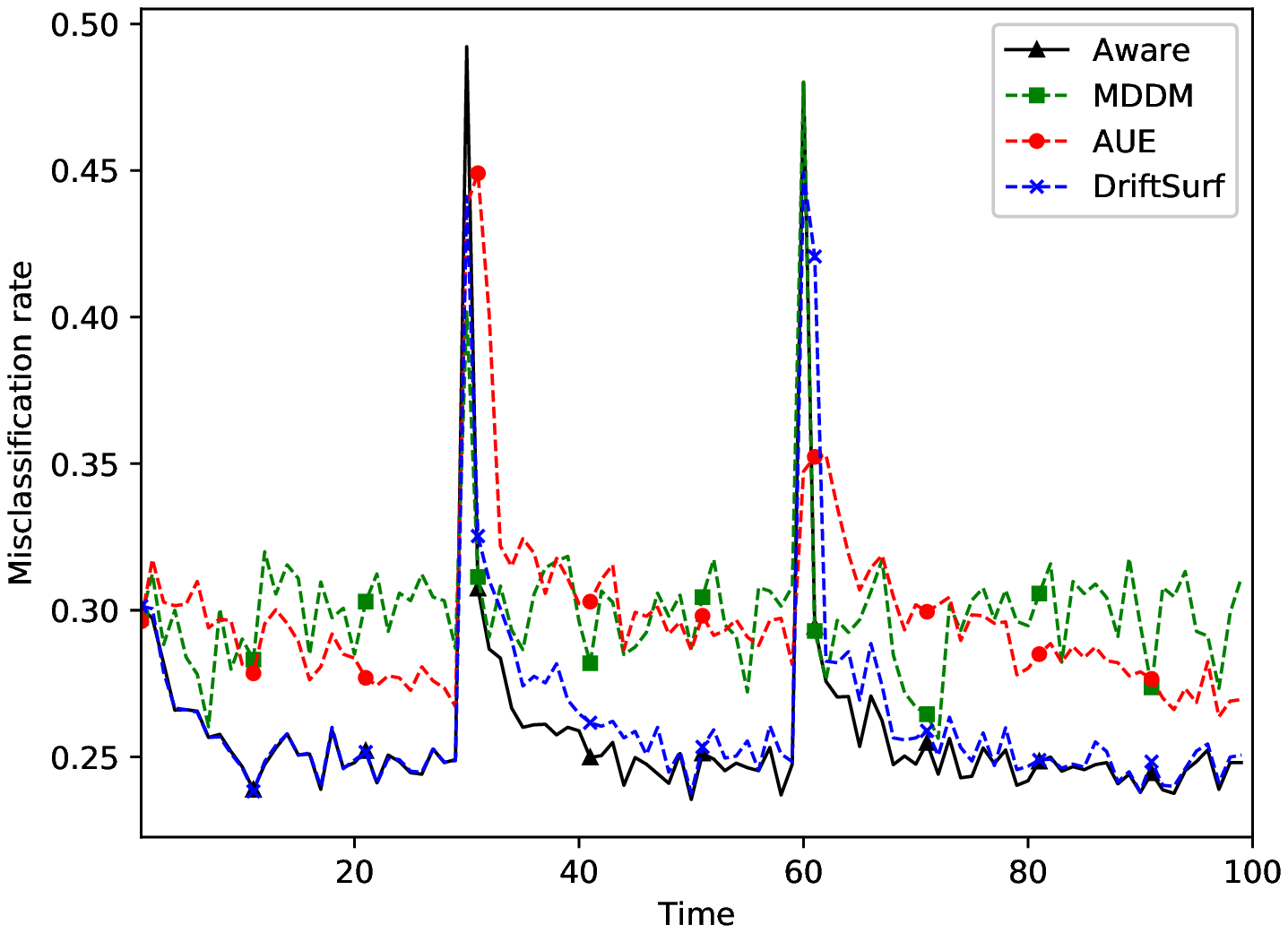}
            \caption{CoverType}
            \label{fig:covtype-limited-4}
        \end{subfigure}\hfill
        
        \begin{subfigure}[t]{0.32\textwidth}
            \includegraphics[width=\columnwidth]{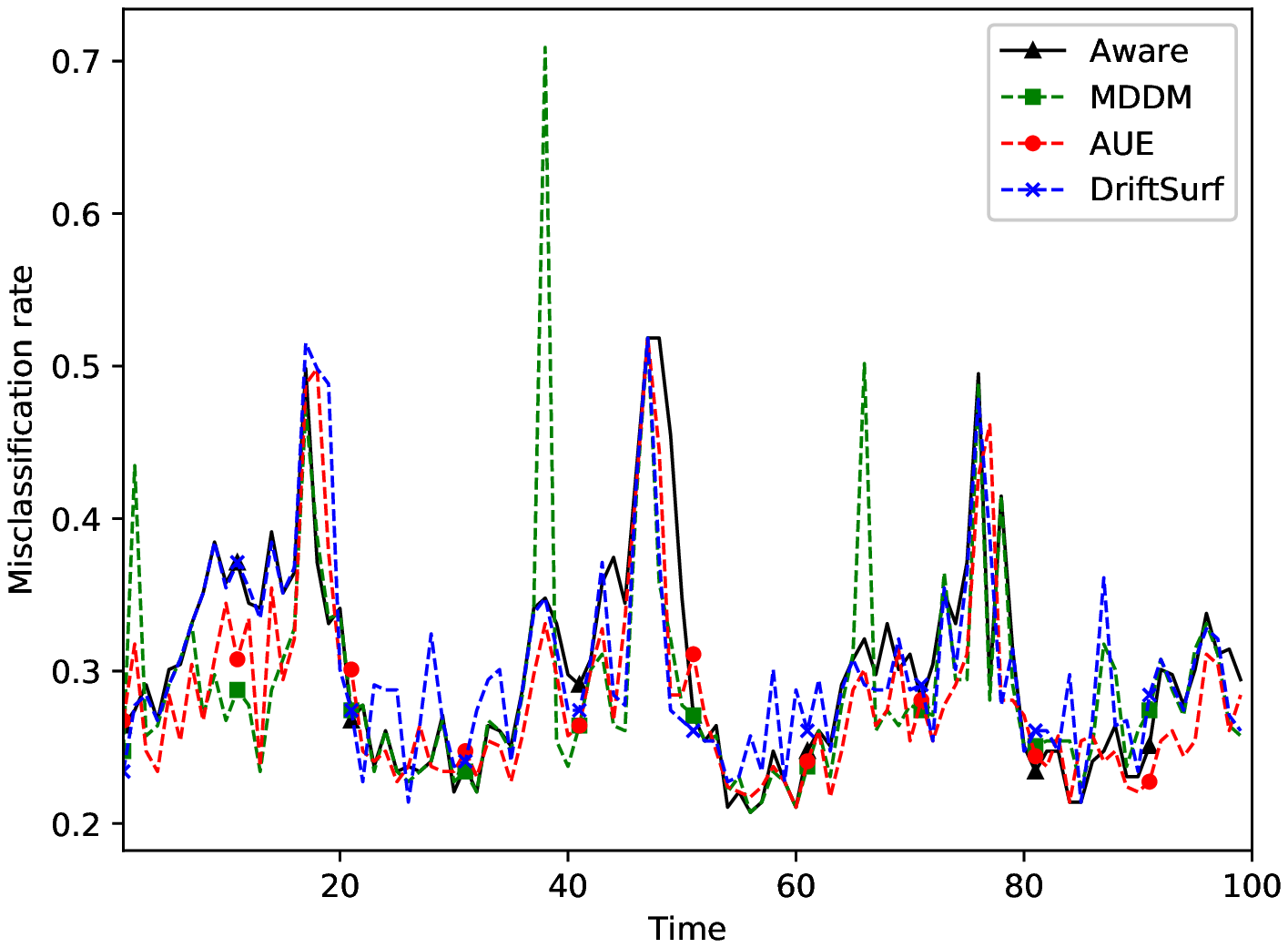}
            \caption{PowerSupply}
            \label{fig:powersupply-limited-4}
        \end{subfigure}\hfill
        \begin{subfigure}[t]{0.32\textwidth}
            \includegraphics[width=\columnwidth]{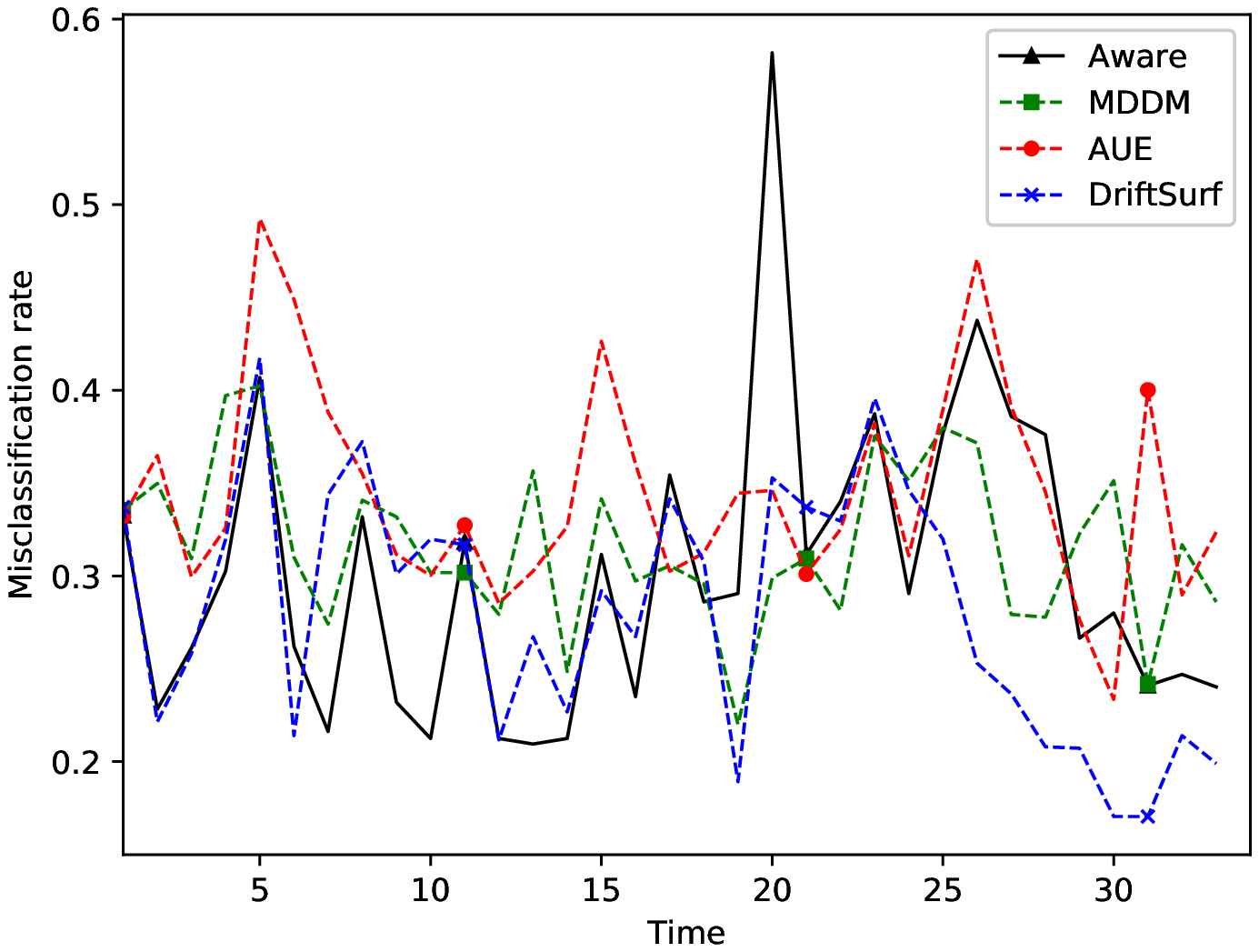}
            \caption{Electricity}
            \label{fig:elec-limited-4}
        \end{subfigure}\hfill
        \begin{subfigure}[t]{0.32\textwidth}
            \includegraphics[width=\columnwidth]{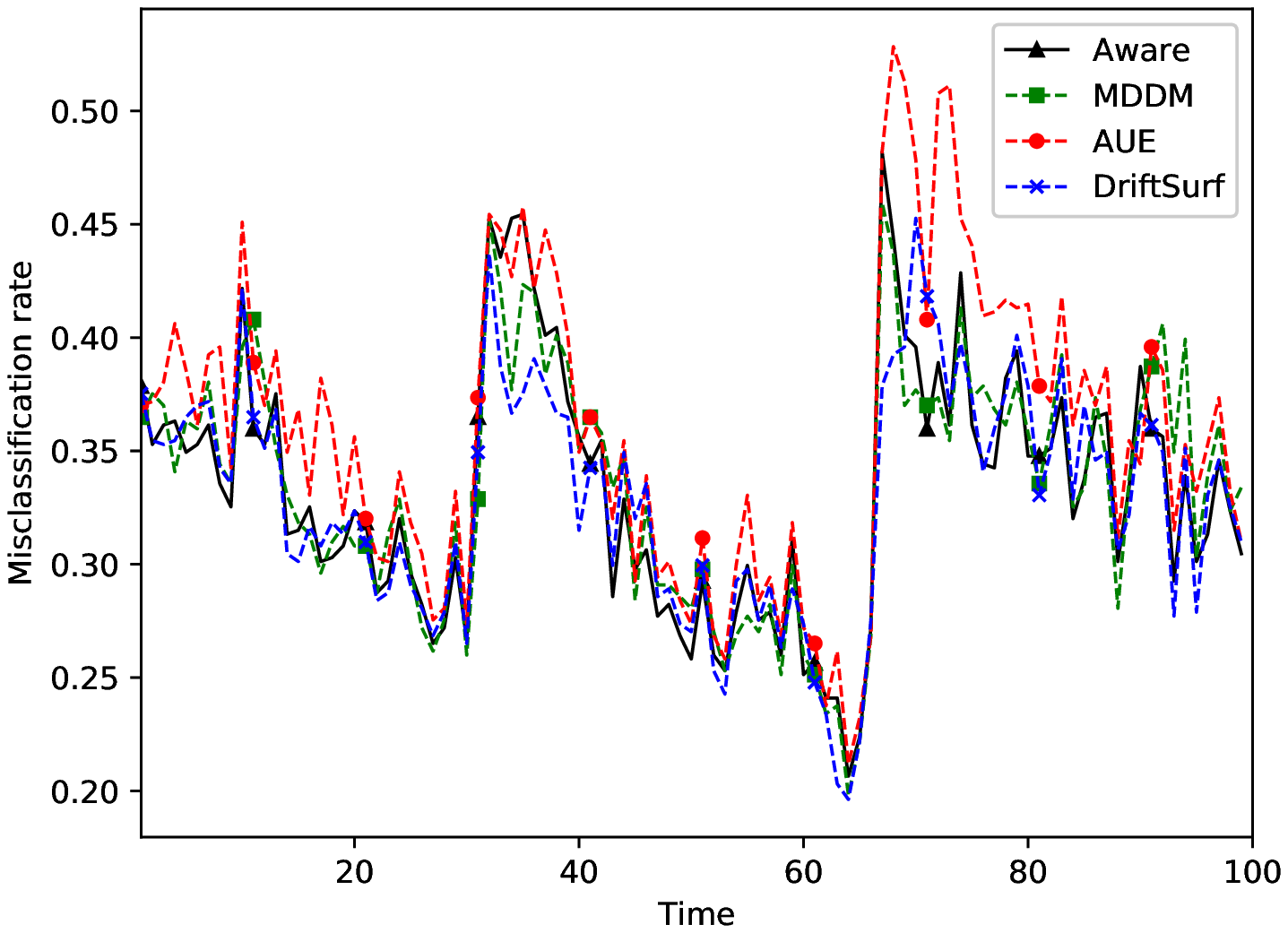}
            \caption{Airline}
            \label{fig:airline-limited-4}
        \end{subfigure}\hfill
    \end{center}
    \caption{Misclassification rate over time ($\rho=4m$ divided among all models of each algorithm) comparing $\aware, \dsurf$, AUE and MDDM}
    \label{fig:Misclassification rate over time-rho=4m}
\end{figure*}

\begin{figure*}[h!]
    \begin{center}
        \begin{subfigure}[t]{0.32\textwidth}
            \includegraphics[width=\columnwidth]{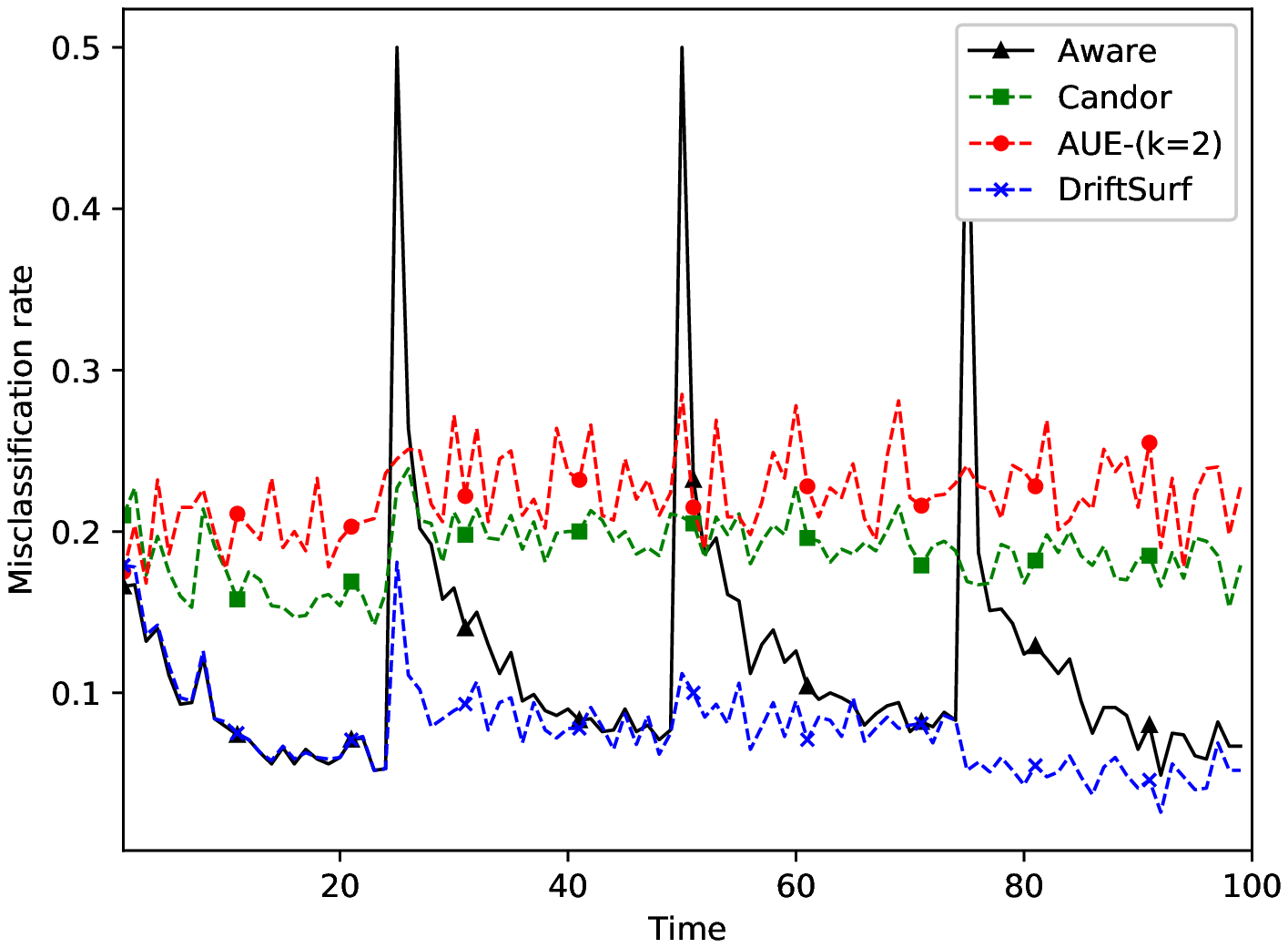}
            \caption{SEA0}
            \label{fig:sea0-limited-4-candor}
        \end{subfigure}\hfill
        \begin{subfigure}[t]{0.32\textwidth}
            \includegraphics[width=\columnwidth]{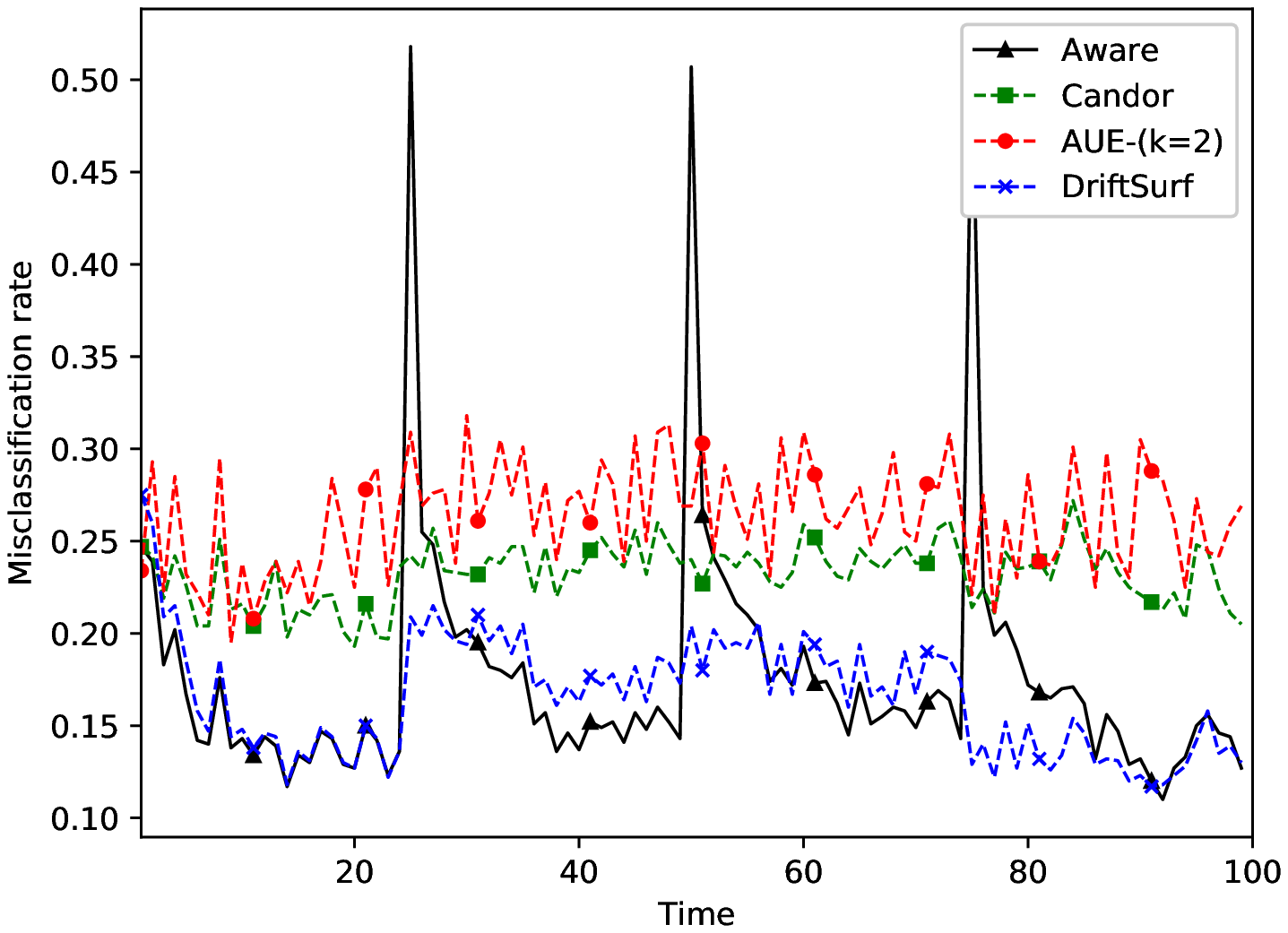}
            \caption{SEA10}
            \label{fig:sea10-limited-4-candor}
        \end{subfigure}\hfill
        \begin{subfigure}[t]{0.32\textwidth}
            \includegraphics[width=\columnwidth]{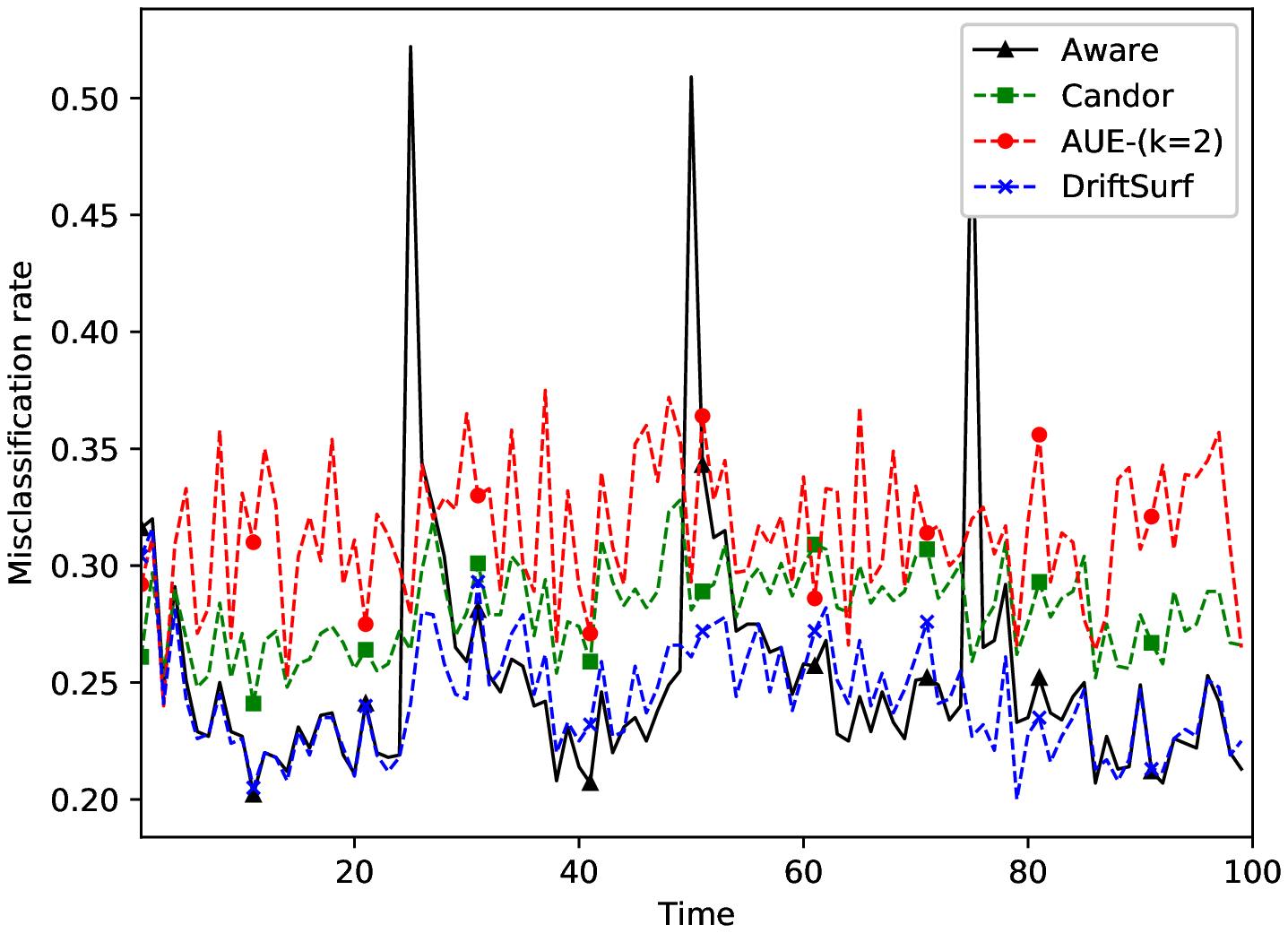}
            \caption{SEA20}
            \label{fig:sea20-limited-4-candor}
        \end{subfigure}\hfill
        
        \begin{subfigure}[t]{0.32\textwidth}
            \includegraphics[width=\columnwidth]{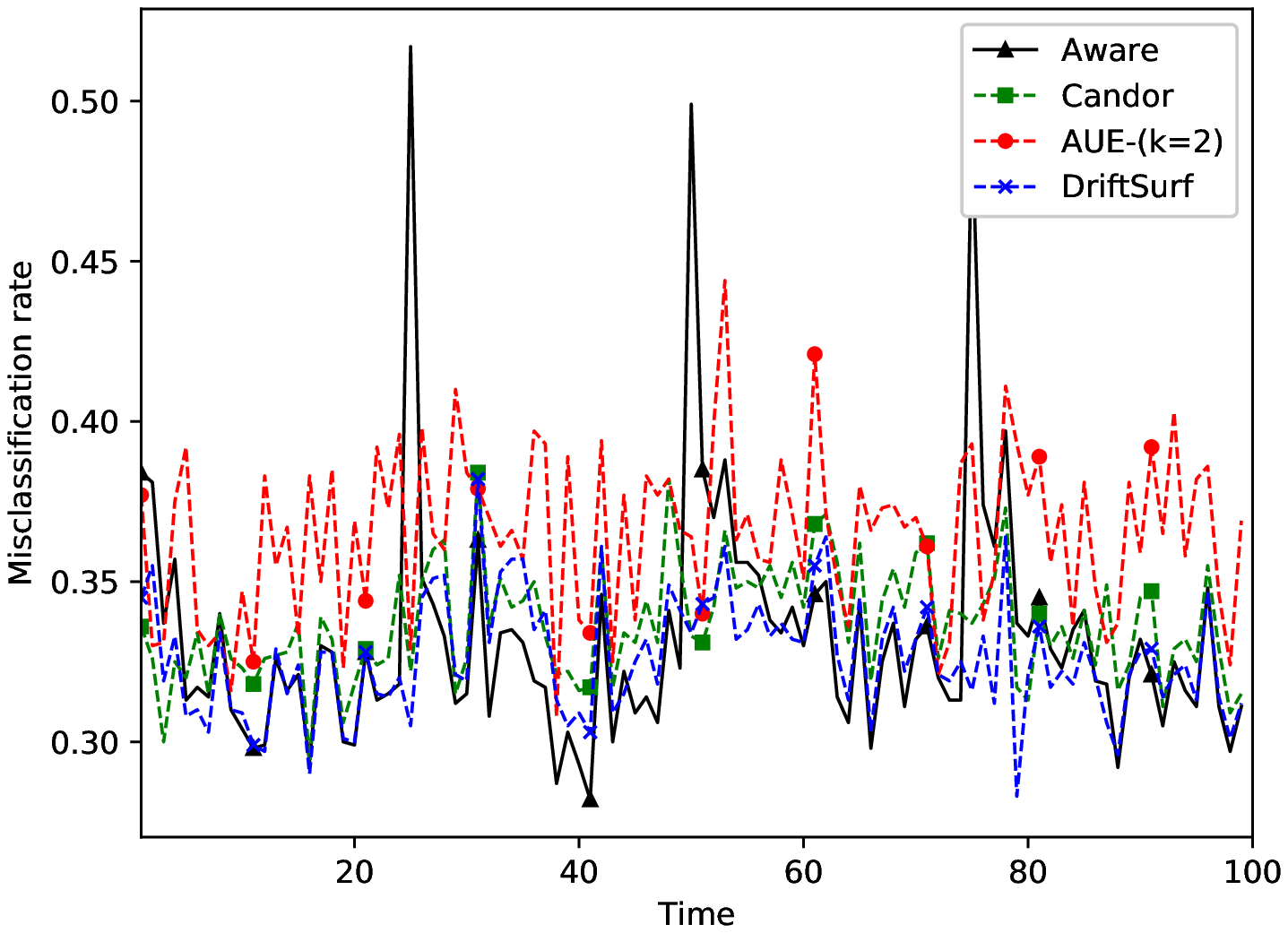}
            \caption{SEA30}
            \label{fig:sea30-limited-4-candor}
        \end{subfigure}\hfill
        \begin{subfigure}[t]{0.32\textwidth}
            \includegraphics[width=\columnwidth]{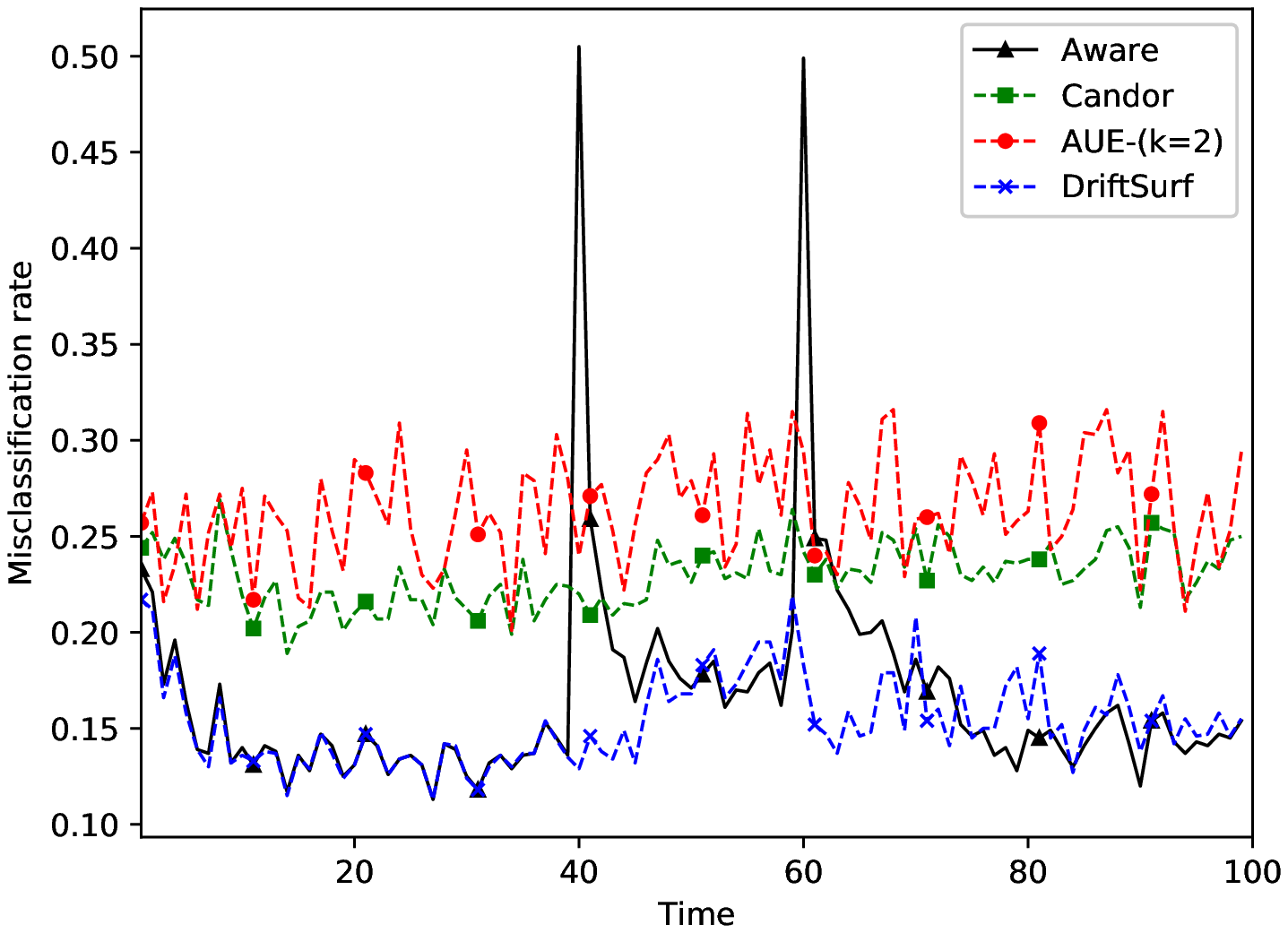}
            \caption{SEA-gradual}
            \label{fig:sea_gradual-limited-4-candor}
        \end{subfigure}\hfill
        \begin{subfigure}[t]{0.32\textwidth}
            \includegraphics[width=\columnwidth]{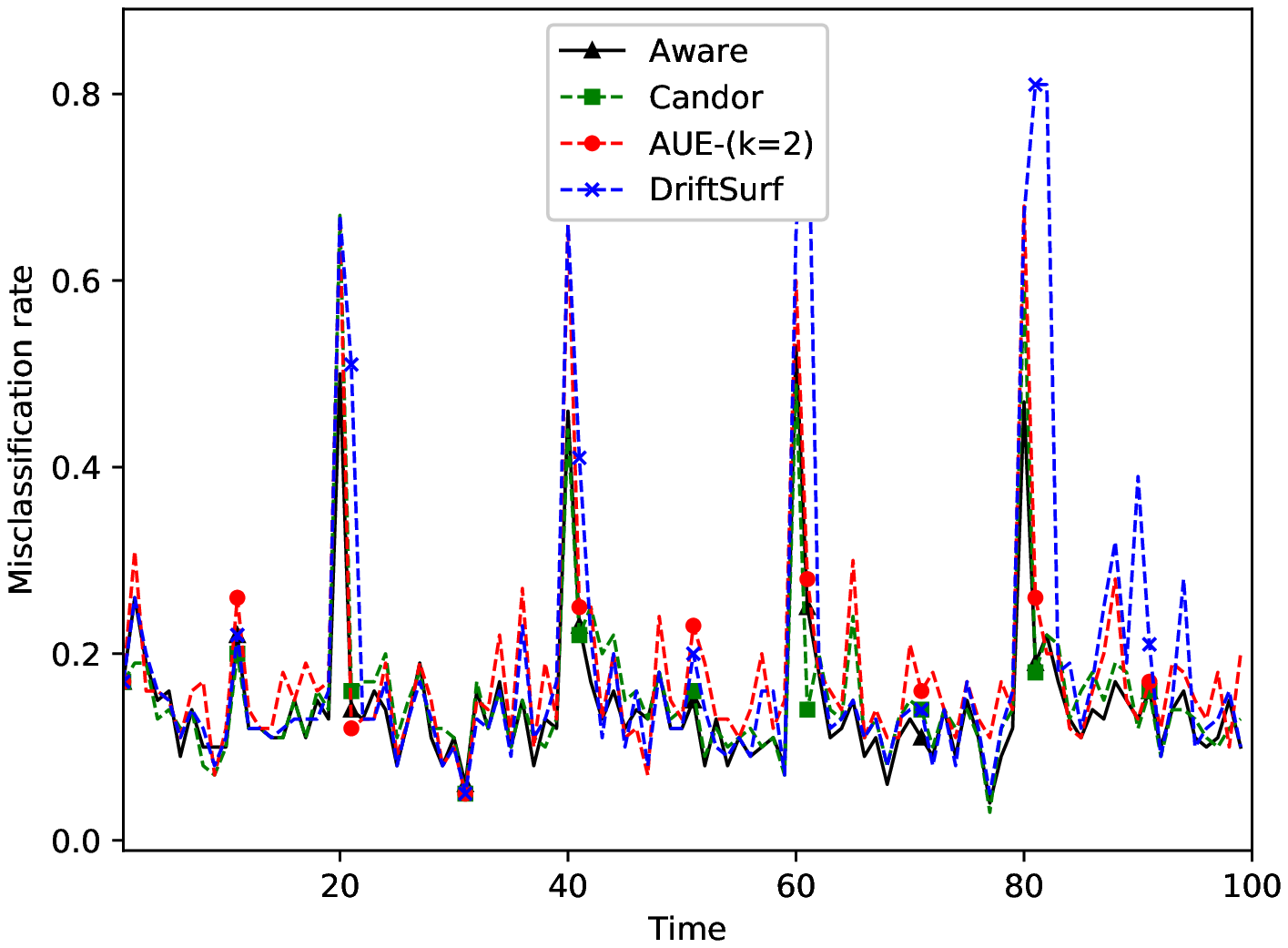}
            \caption{SINE1}
            \label{fig:sine1-limited-4-candor}
        \end{subfigure}\hfill
        
        \begin{subfigure}[t]{0.32\textwidth}
            \includegraphics[width=\columnwidth]{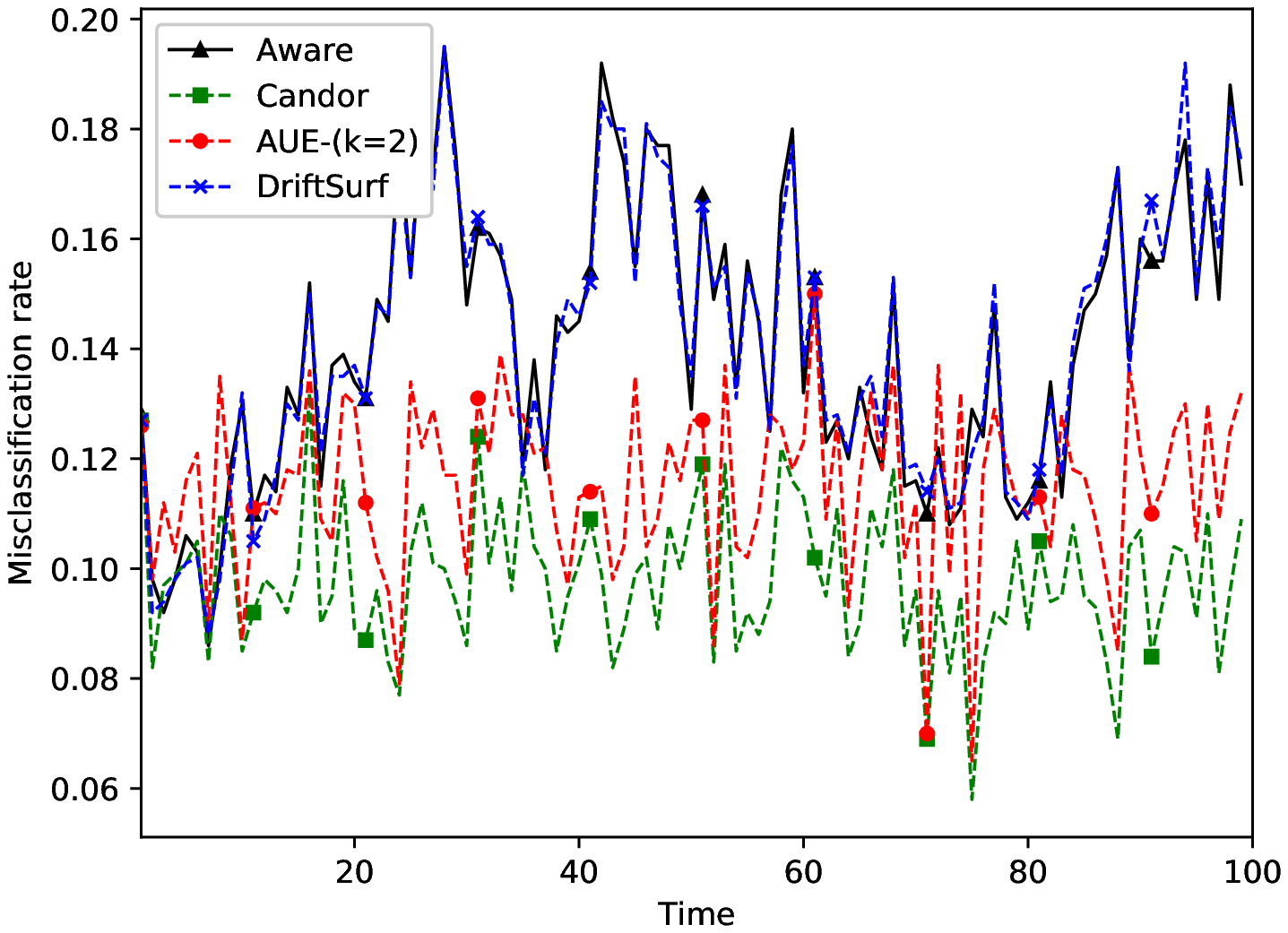}
            \caption{HyperPlane-slow}
            \label{fig:hyperplane-slow-limited-4-candor}
        \end{subfigure}\hfill
        \begin{subfigure}[t]{0.32\textwidth}
            \includegraphics[width=\columnwidth]{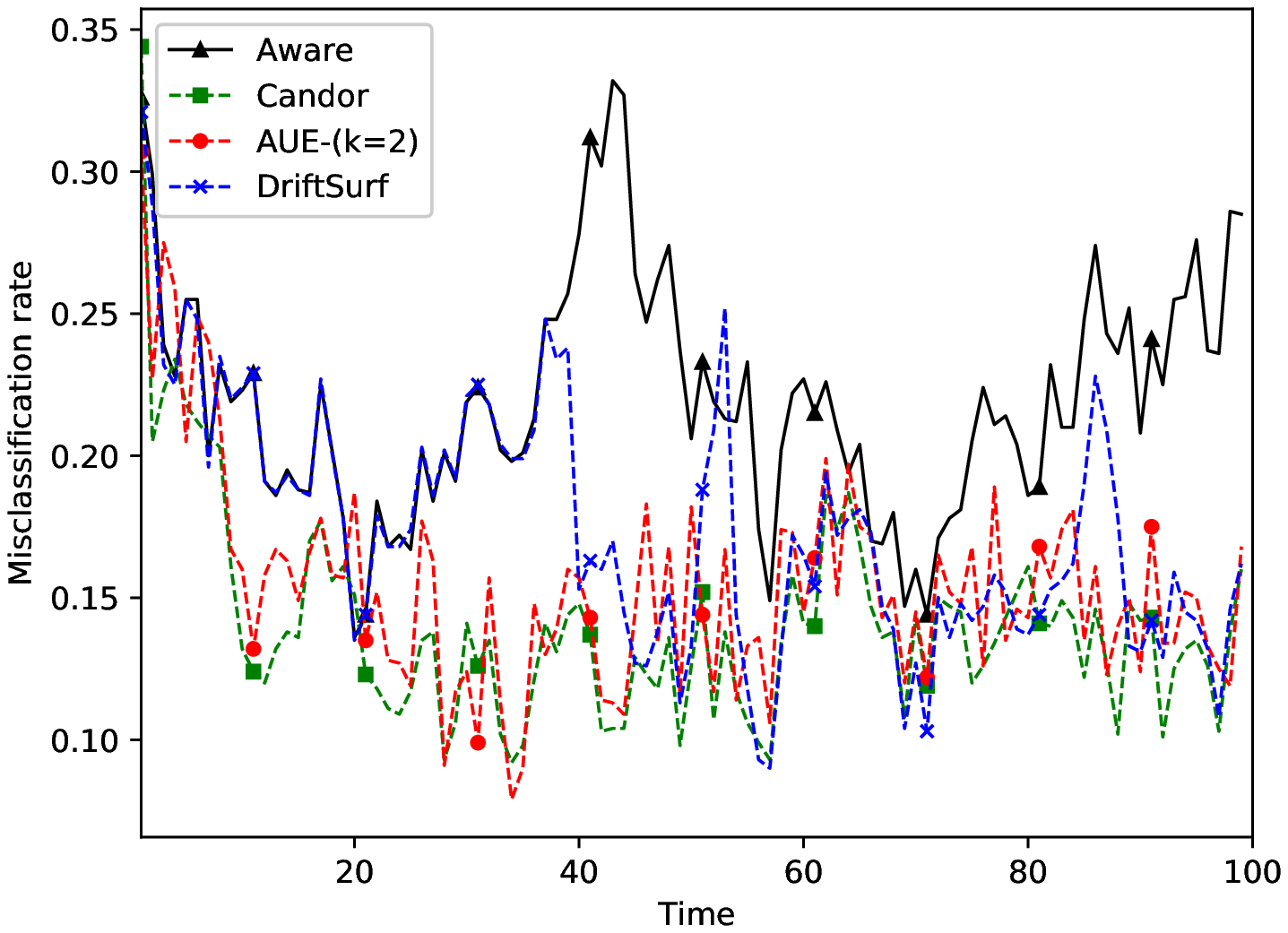}
            \caption{HyperPlane-fast}
            \label{fig:hyperplane-fast-limited-4-candor}
        \end{subfigure}\hfill
        \begin{subfigure}[t]{0.32\textwidth}
            \includegraphics[width=\columnwidth]{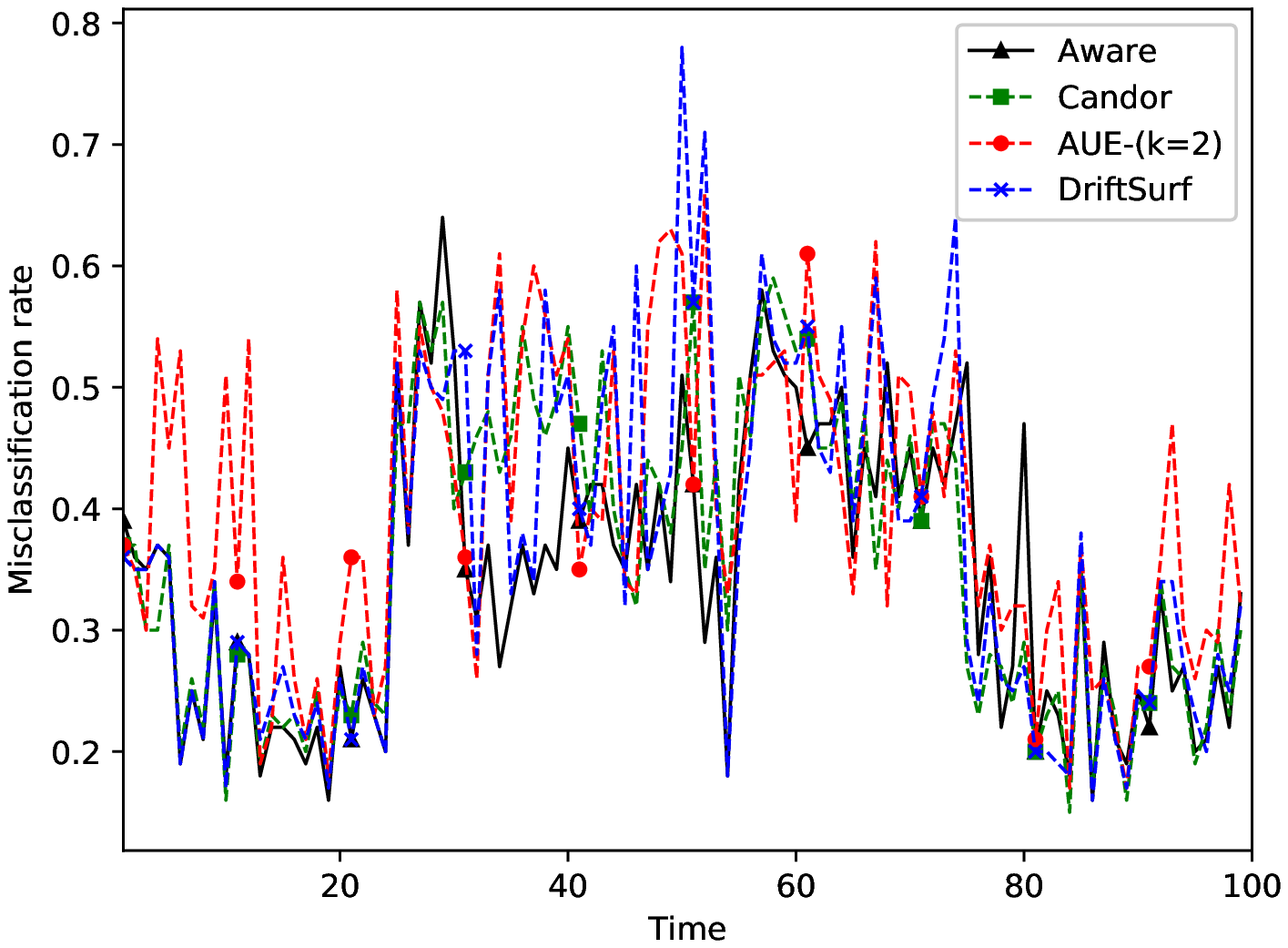}
            \caption{CIRCLES}
            \label{fig:circles-limited-4-candor}
        \end{subfigure}\hfill
        
        \begin{subfigure}[t]{0.32\textwidth}
            \includegraphics[width=\columnwidth]{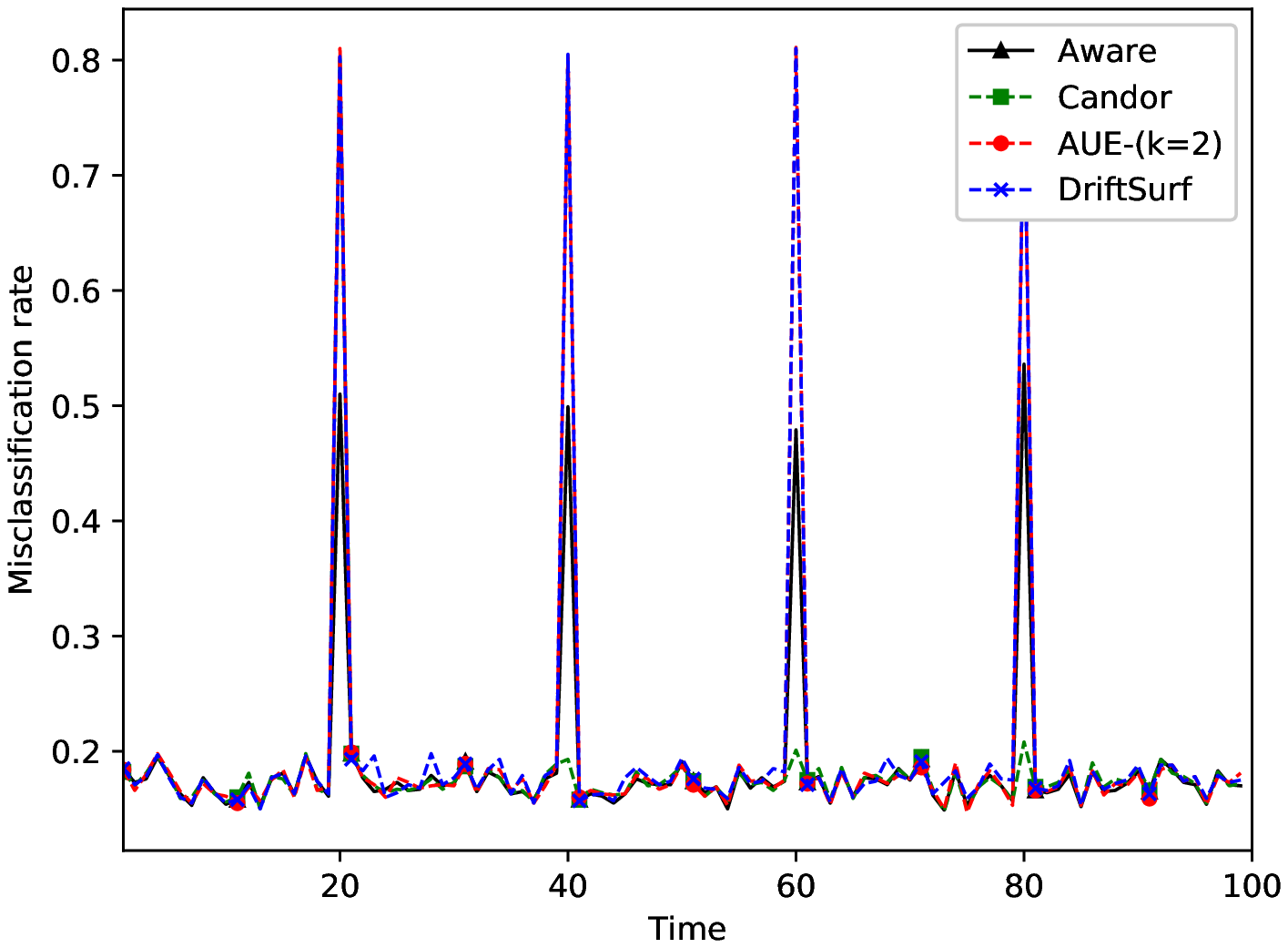}
            \caption{MIXED}
            \label{fig:mixed-limited-4-candor}
        \end{subfigure}\hfill
        \begin{subfigure}[t]{0.32\textwidth}
            \includegraphics[width=\columnwidth]{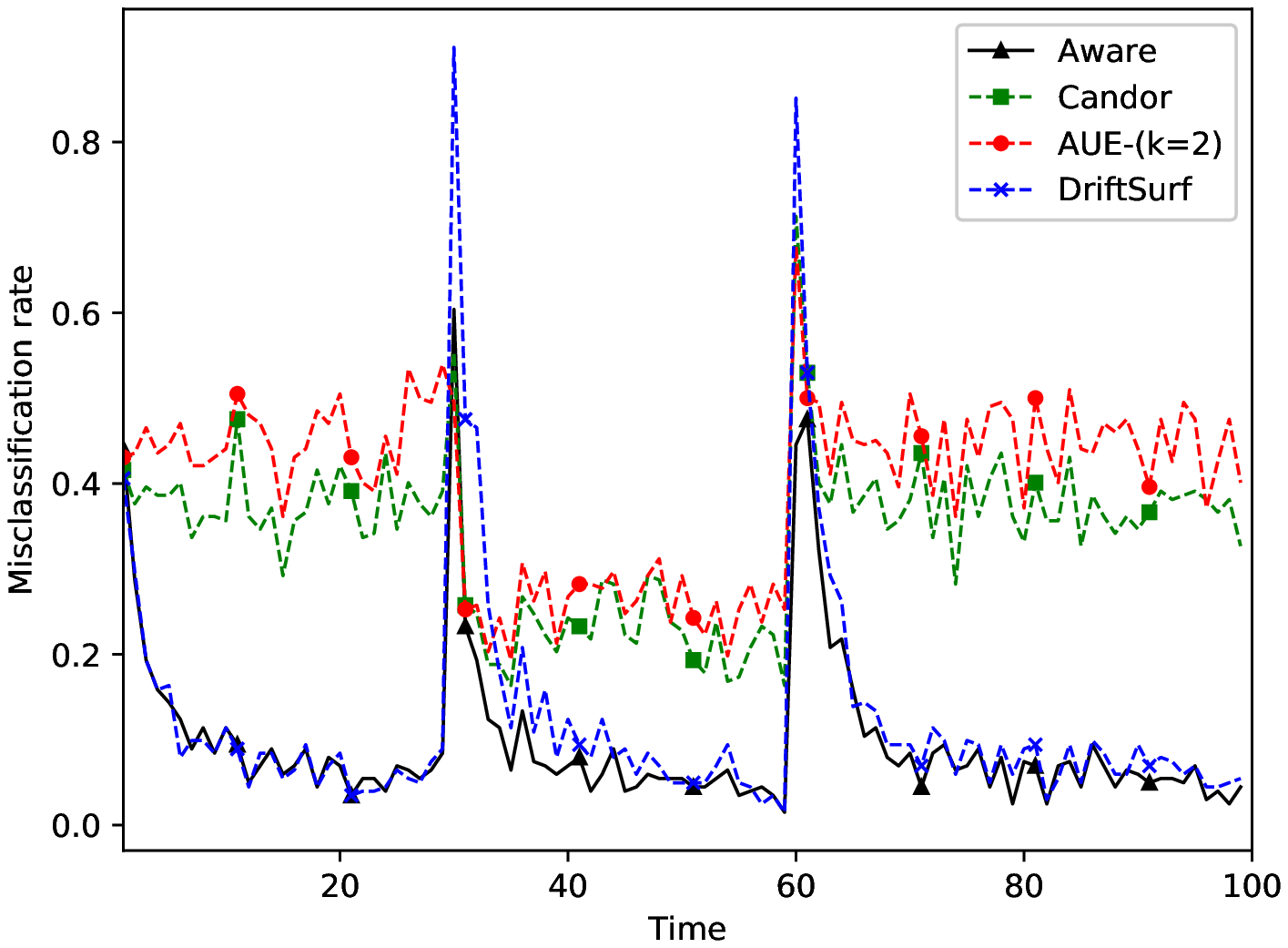}
            \caption{RCV1}
            \label{fig:rcv-limited-4-candor}
        \end{subfigure}\hfill
        \begin{subfigure}[t]{0.32\textwidth}
            \includegraphics[width=\columnwidth]{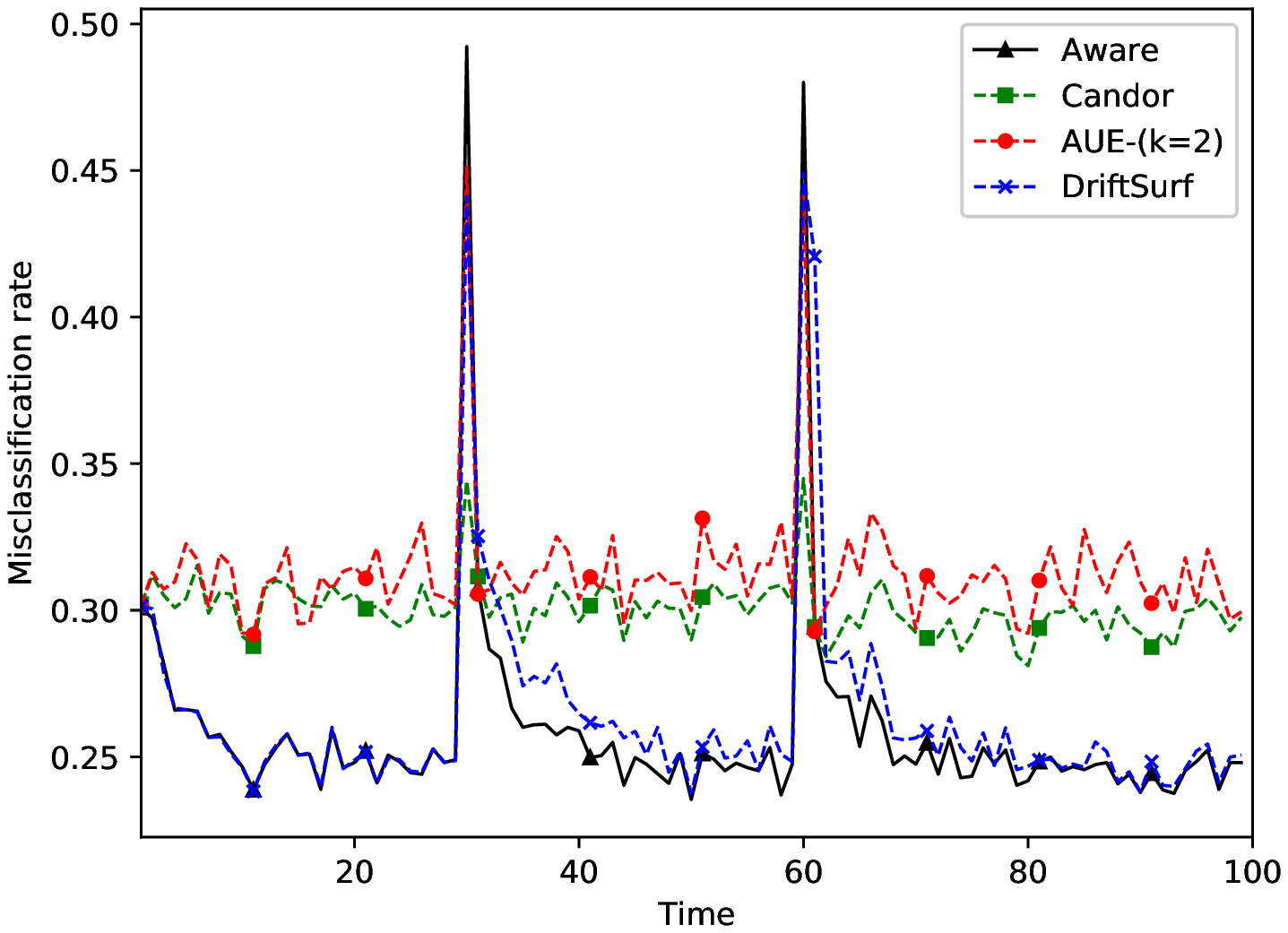}
            \caption{CoverType}
            \label{fig:covtype-limited-4-candor}
        \end{subfigure}\hfill
        
        \begin{subfigure}[t]{0.32\textwidth}
            \includegraphics[width=\columnwidth]{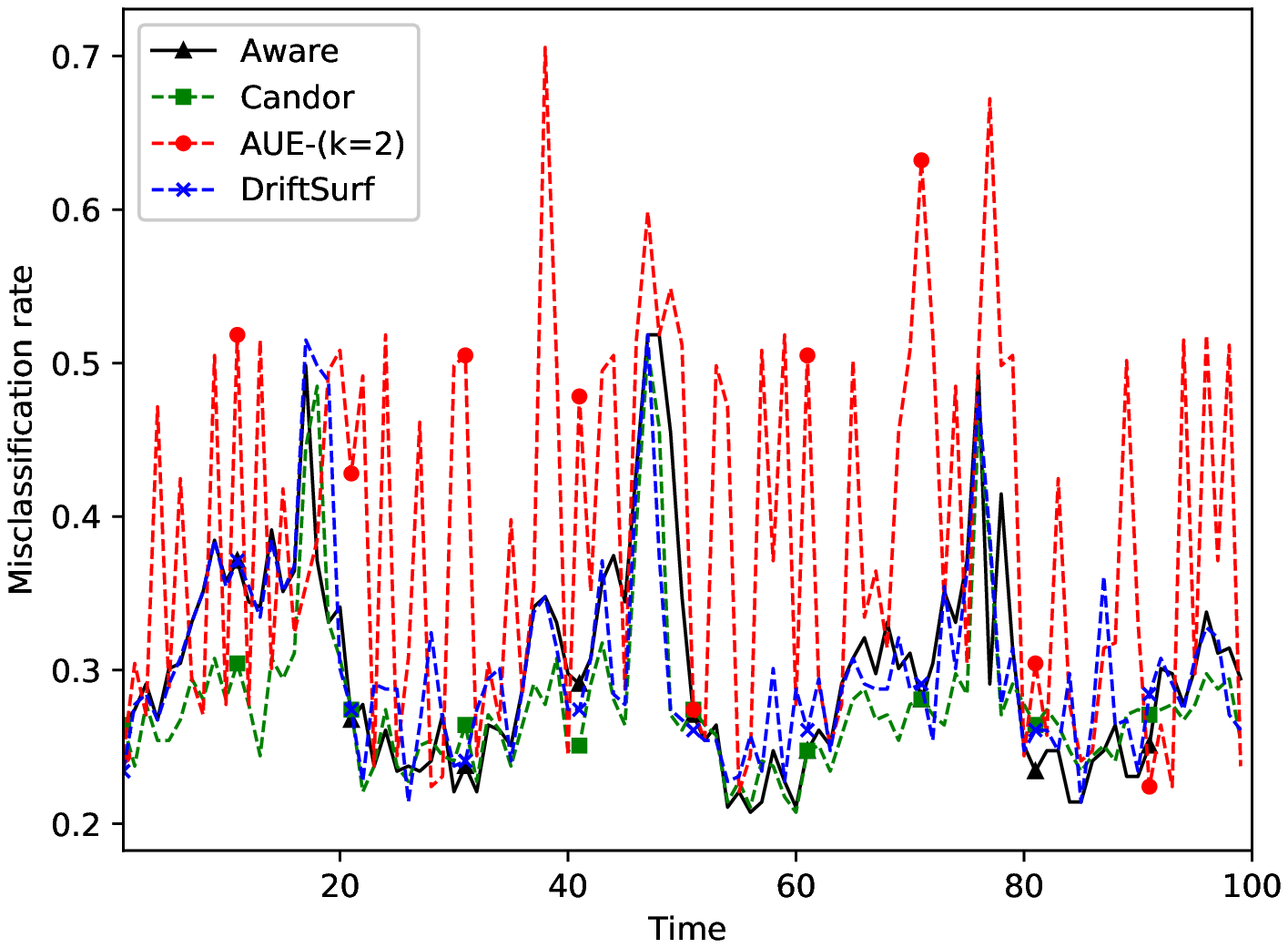}
            \caption{PowerSupply}
            \label{fig:powersupply-limited-4-candor}
        \end{subfigure}\hfill
        \begin{subfigure}[t]{0.32\textwidth}
            \includegraphics[width=\columnwidth]{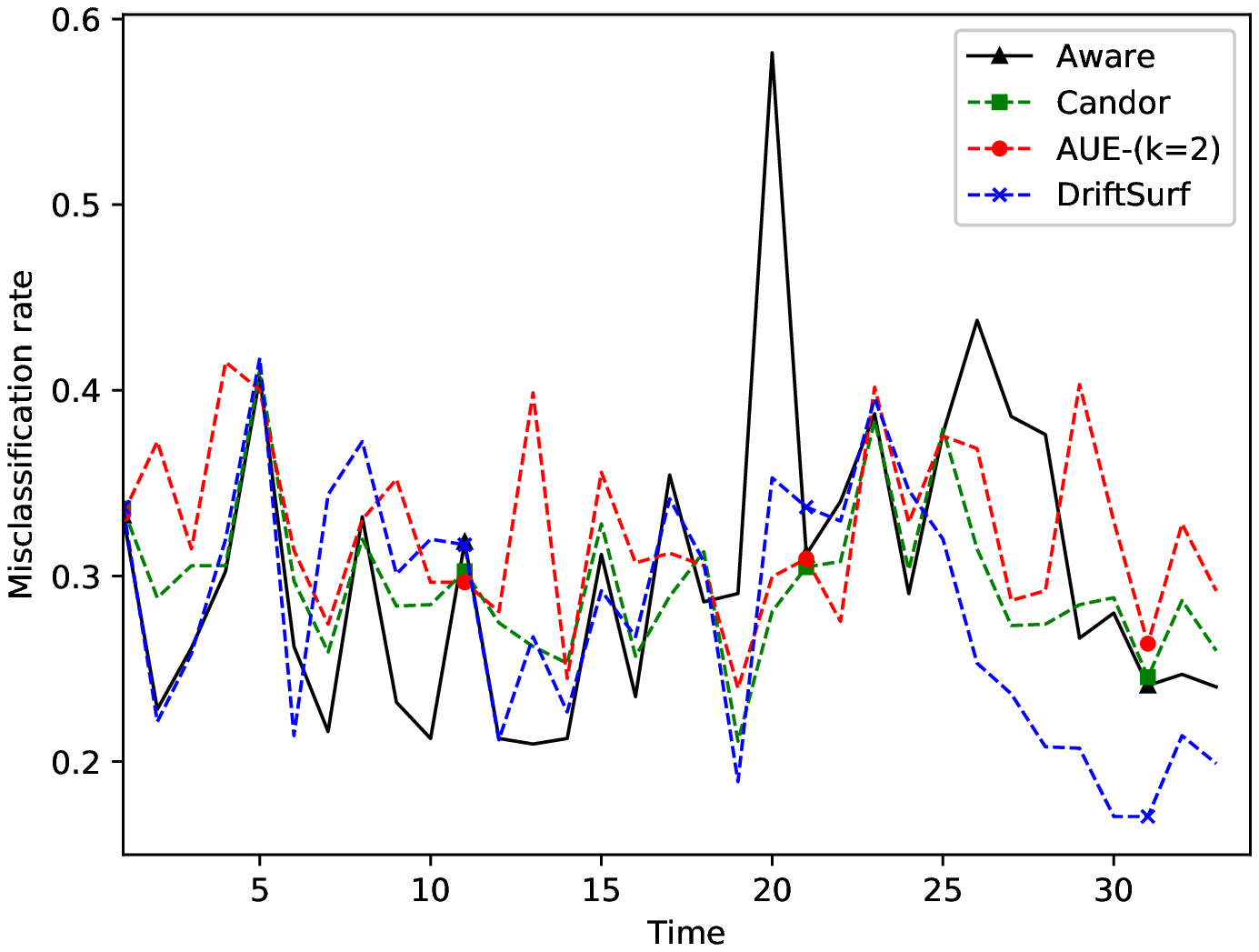}
            \caption{Electricity}
            \label{fig:elec-limited-4-candor}
        \end{subfigure}\hfill
        \begin{subfigure}[t]{0.32\textwidth}
            \includegraphics[width=\columnwidth]{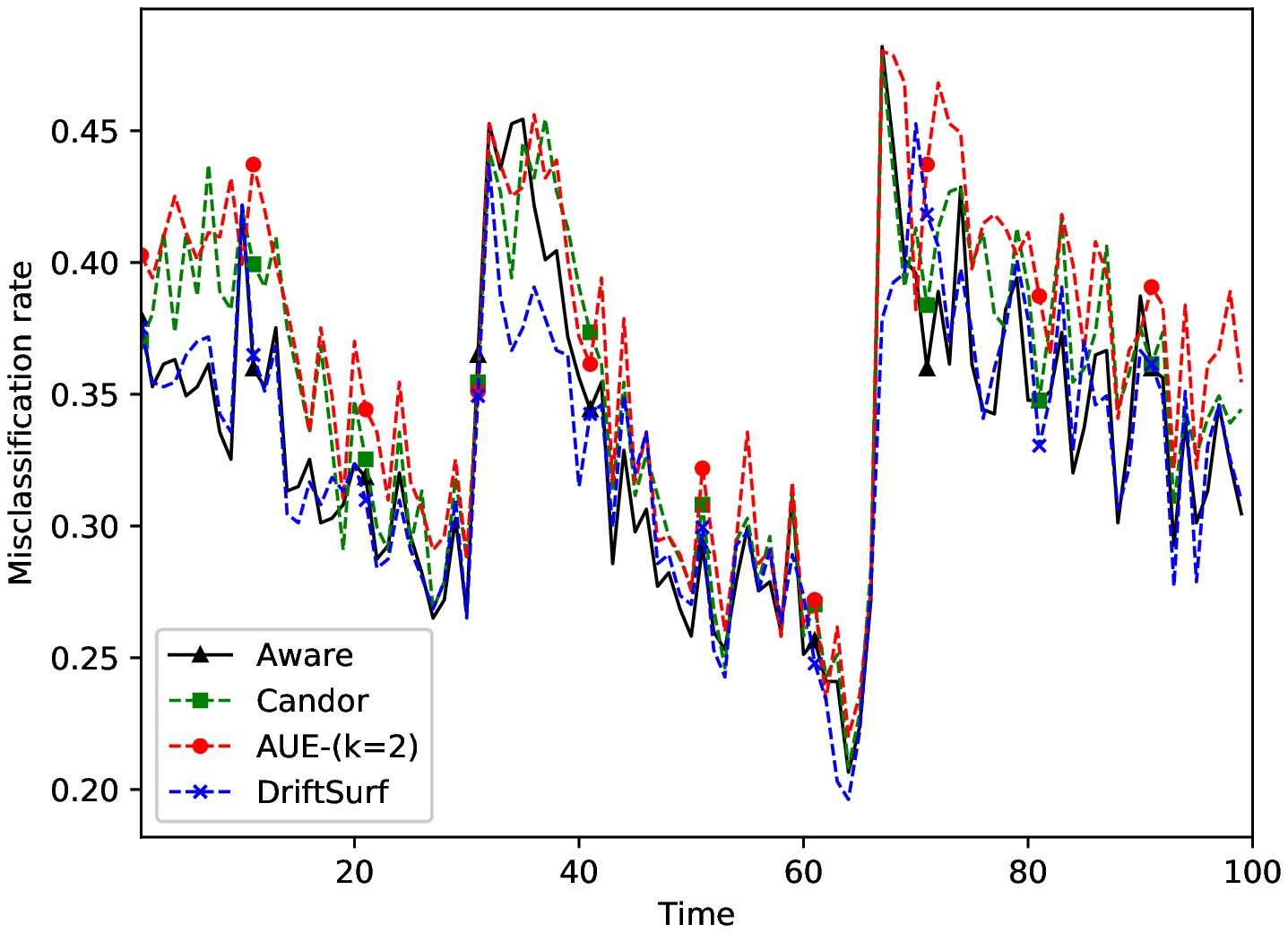}
            \caption{Airline}
            \label{fig:airline-limited-4-candor}
        \end{subfigure}\hfill
    \end{center}
    \caption{Misclassification rate over time ($\rho=4m$ divided among all models of each algorithm) comparing $\aware, \dsurf$, AUE with $k=2$ and Candor}
    \label{fig:Misclassification rate over time-rho=4m-candor}
\end{figure*}

\begin{table}[ht!]
\caption{Total average of misclassification rate ($\rho=2m$ divided among all models of each algorithm)}
\label{table:ave-misclassification-rho=2m}
\begin{center}
\begin{small}
\begin{sc}
\begin{tabular}{l*{6}c}
\toprule
 Dataset    	& \aware & \dsurf		& MDDM-G  	& AUE  	& AUE ($k$=2) 		& Candor\\
\midrule
SEA0        	& 0.133 & 0.098 		& \textbf{0.089} & 0.201  	& 0.230			& 0.200\\
SEA10       	& 0.197 & \textbf{0.161} 	& 0.183 		& 0.237  	& 0.275			& 0.238\\
SEA20       	& 0.266 & \textbf{0.246} 	& 0.283 		& 0.291 	& 0.327			& 0.292\\
SEA30       	& 0.352 & \textbf{0.337} 	& 0.360 		& 0.354 	& 0.381			& 0.345\\
SEA-gradual   	& 0.174 & \textbf{0.157} 	& 0.172 		& 0.24 	& 0.273			& 0.239\\
Hyper-slow  	& 0.117 & \textbf{0.116} 	& \textbf{0.116} & 0.191  	& 0.166			& 0.122\\
Hyper-fast  	& 0.191 & 0.199 		& \textbf{0.164} & 0.278 	& 0.211			& 0.166\\
SINE1       	& 0.168 & 0.220 		& \textbf{0.178} & 0.309  	& 0.246			& 0.179\\
Mixed	       	& 0.191 & {0.204} 	& {0.204} & 0.259    & {0.204}		& \textbf{0.182}\\
Circles       	& 0.368 & \textbf{0.372} 	& \textbf{0.372} & 0.401  	& 0.415			& 0.384\\
RCV1        	& 0.120 & 0.174 		& \textbf{0.131} & 0.403	& 0.467			& 0.401 \\
CoverType   	& 0.267 &\textbf{0.276} 	& 0.313 		& 0.317 	& 0.330			& 0.312 \\
Airline     		& 0.338 & \textbf{0.351} 	& \textbf{0.351} & 0.369  	& 0.380			& 0.365\\
Electricity 		& 0.311 & 0.349 		& 0.339		 & 0.364  	& 0.363			& \textbf{0.313}\\
PowerSupply 	& 0.311 & \textbf{0.305} 	& 0.309 		& 0.313	& 0.463			& 0.338\\
\bottomrule
\end{tabular}
\end{sc}
\end{small}
\end{center}
\end{table}

\begin{figure*}[h!]
    \begin{center}
        \begin{subfigure}[t]{0.32\textwidth}
            \includegraphics[width=\columnwidth]{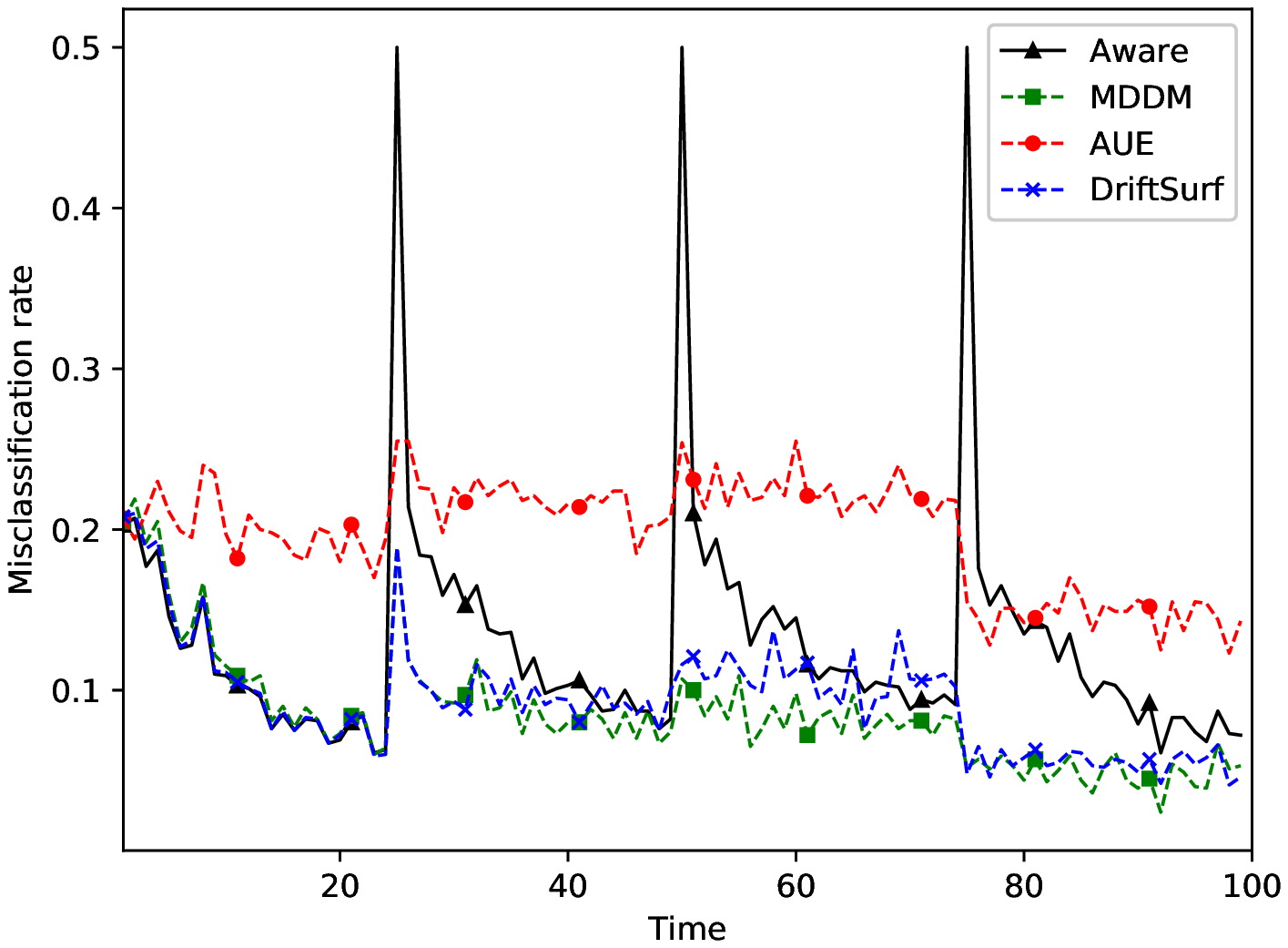}
            \caption{SEA0}
            \label{fig:sea0-limited-2}
        \end{subfigure}\hfill
        \begin{subfigure}[t]{0.32\textwidth}
            \includegraphics[width=\columnwidth]{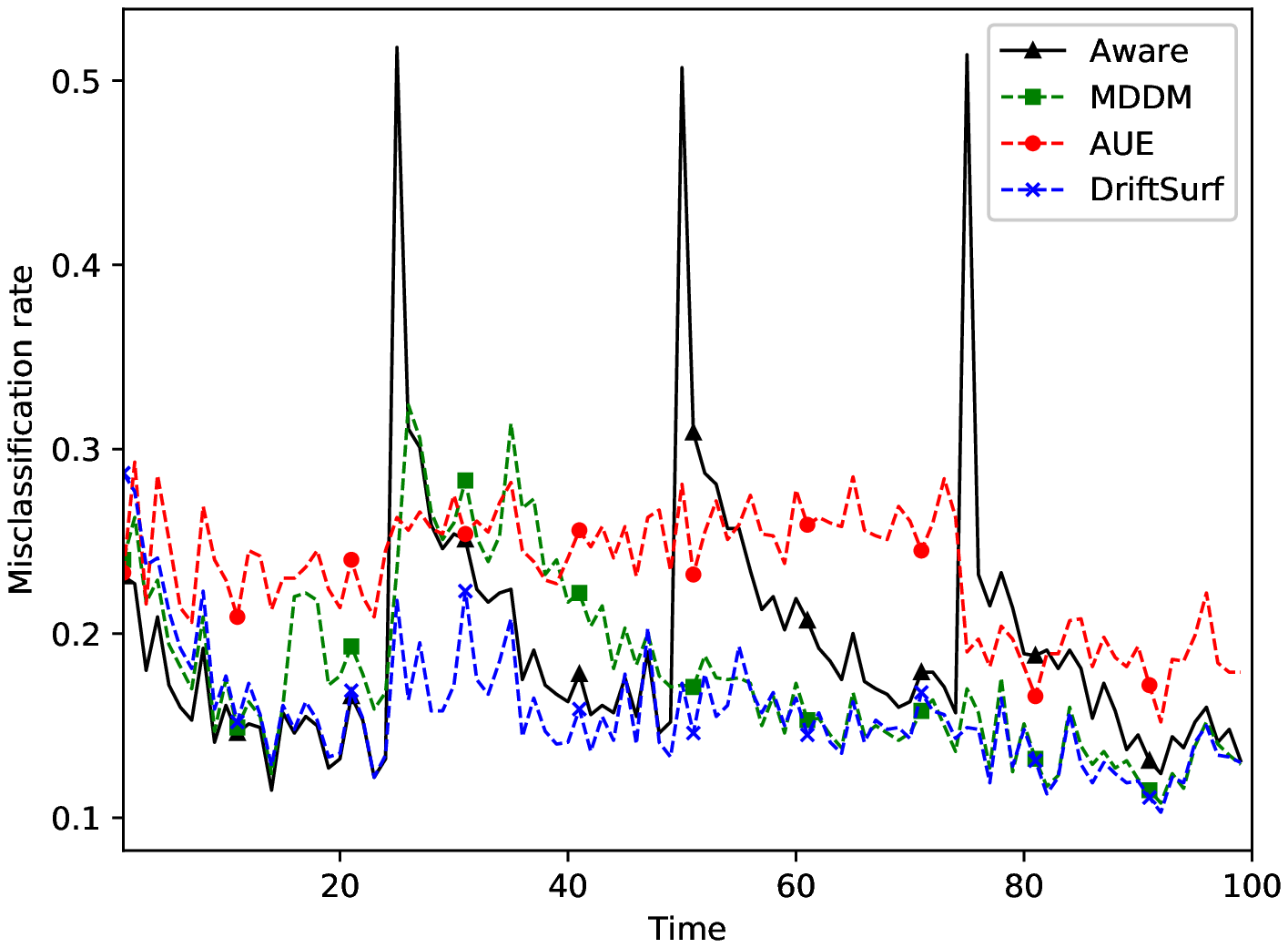}
            \caption{SEA10}
            \label{fig:sea10-limited-2}
        \end{subfigure}\hfill
        \begin{subfigure}[t]{0.32\textwidth}
            \includegraphics[width=\columnwidth]{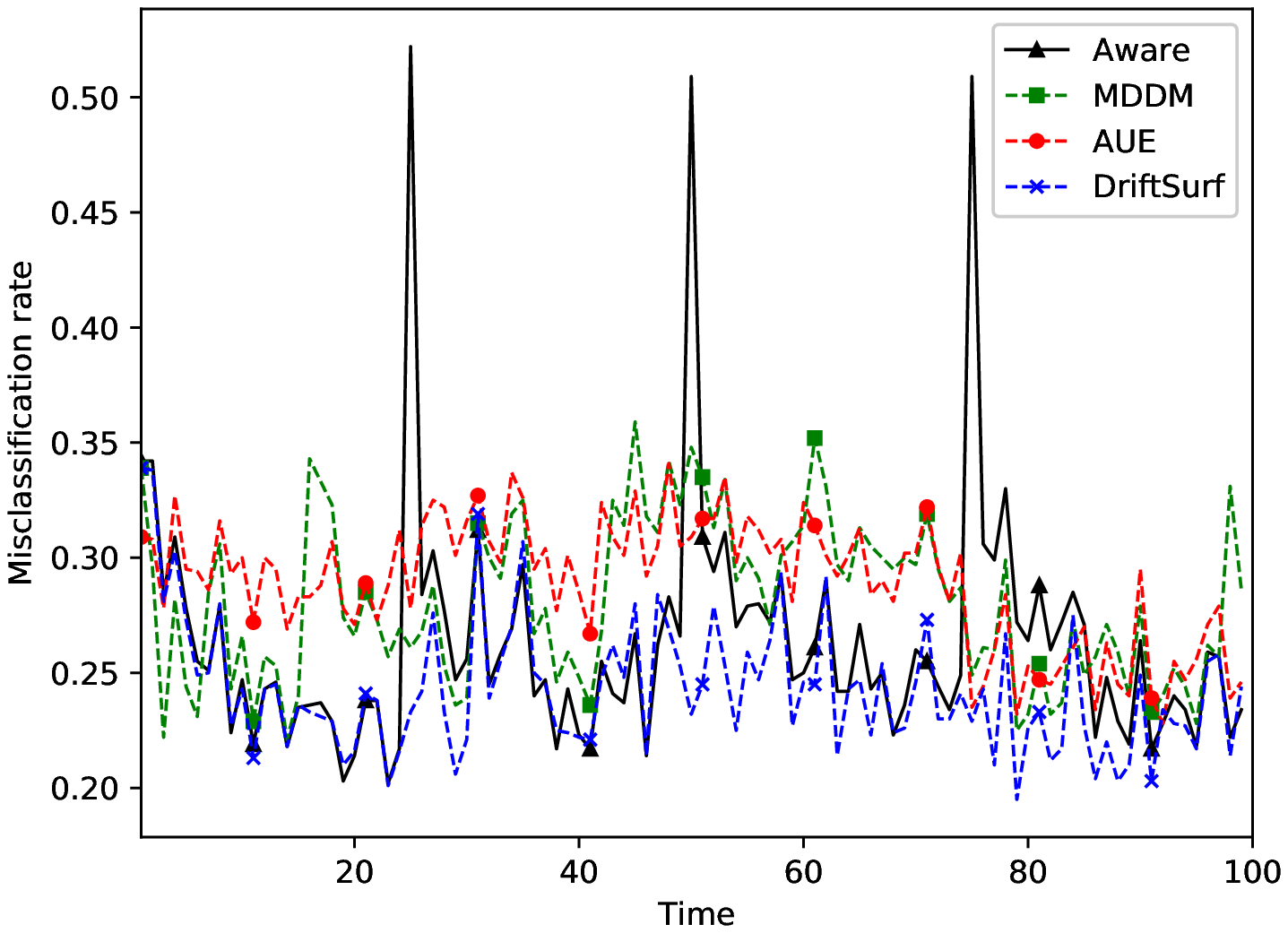}
            \caption{SEA20}
            \label{fig:sea20-limited-2}
        \end{subfigure}\hfill
        
        \begin{subfigure}[t]{0.32\textwidth}
            \includegraphics[width=\columnwidth]{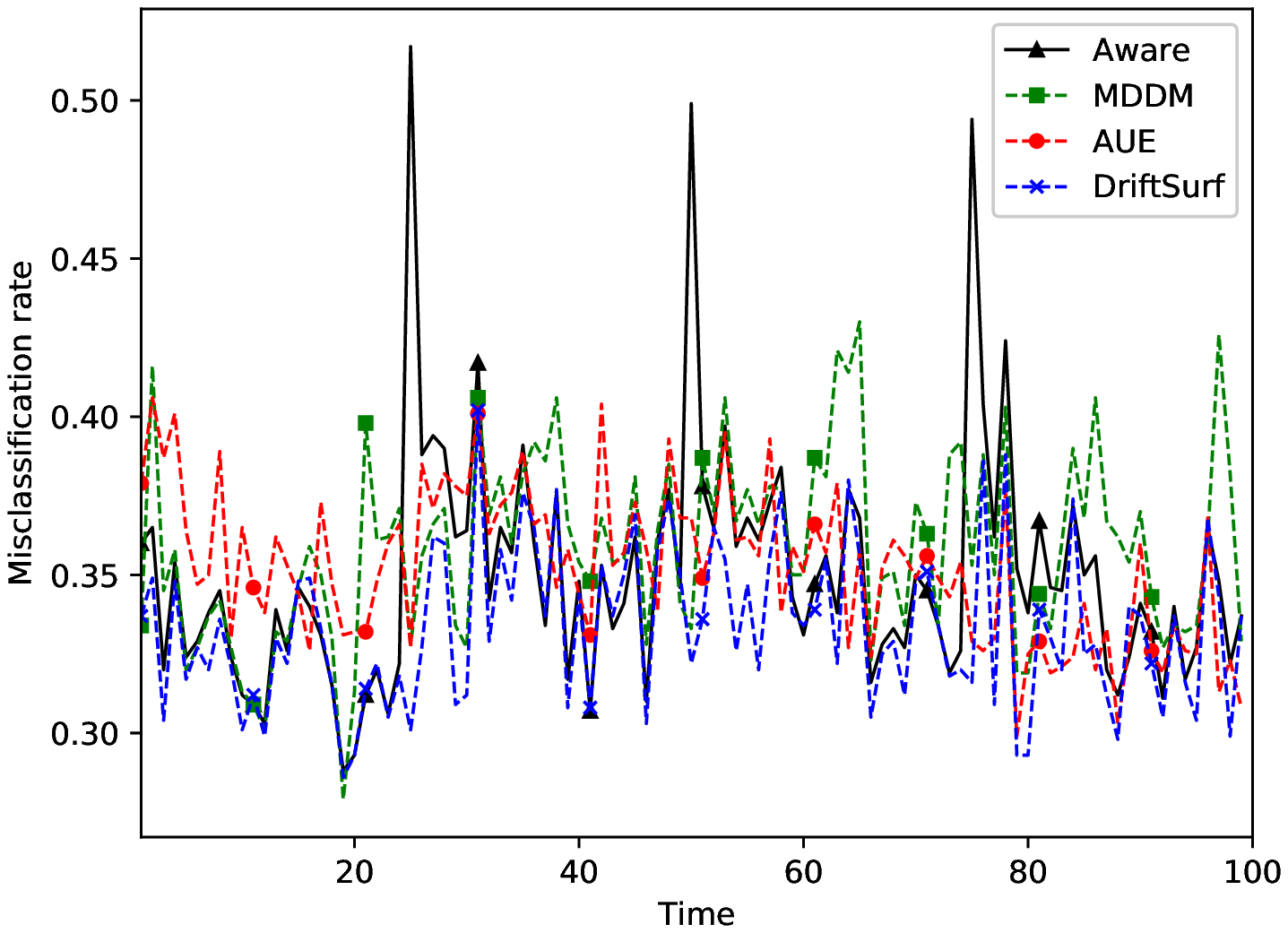}
            \caption{SEA30}
            \label{fig:sea30-limited-2}
        \end{subfigure}\hfill
        \begin{subfigure}[t]{0.32\textwidth}
            \includegraphics[width=\columnwidth]{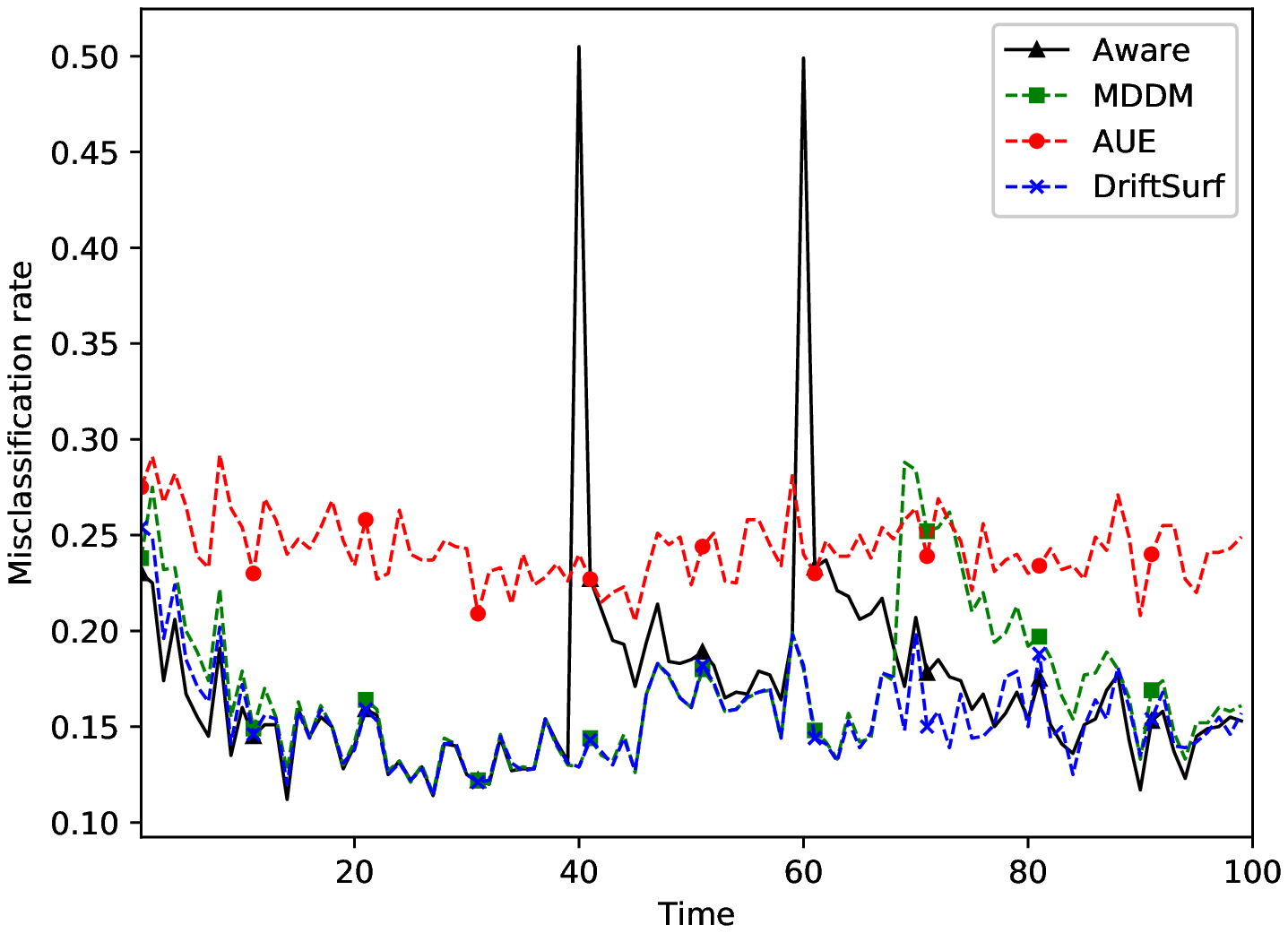}
            \caption{SEA-gradual}
            \label{fig:sea_gradual-limited-2}
        \end{subfigure}\hfill
        \begin{subfigure}[t]{0.32\textwidth}
            \includegraphics[width=\columnwidth]{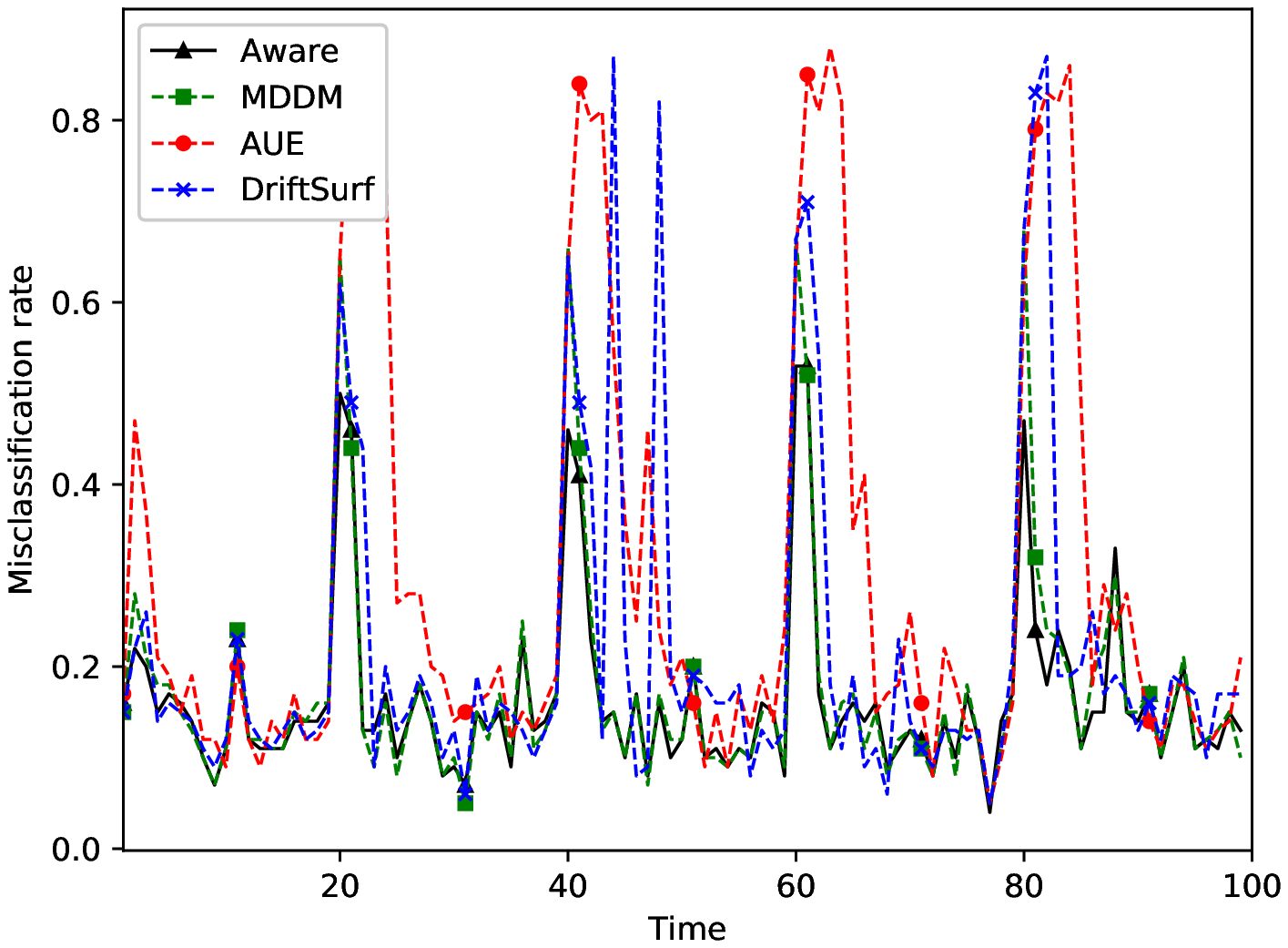}
            \caption{SINE1}
            \label{fig:sine1-limited-2}
        \end{subfigure}\hfill

        \begin{subfigure}[t]{0.32\textwidth}
            \includegraphics[width=\columnwidth]{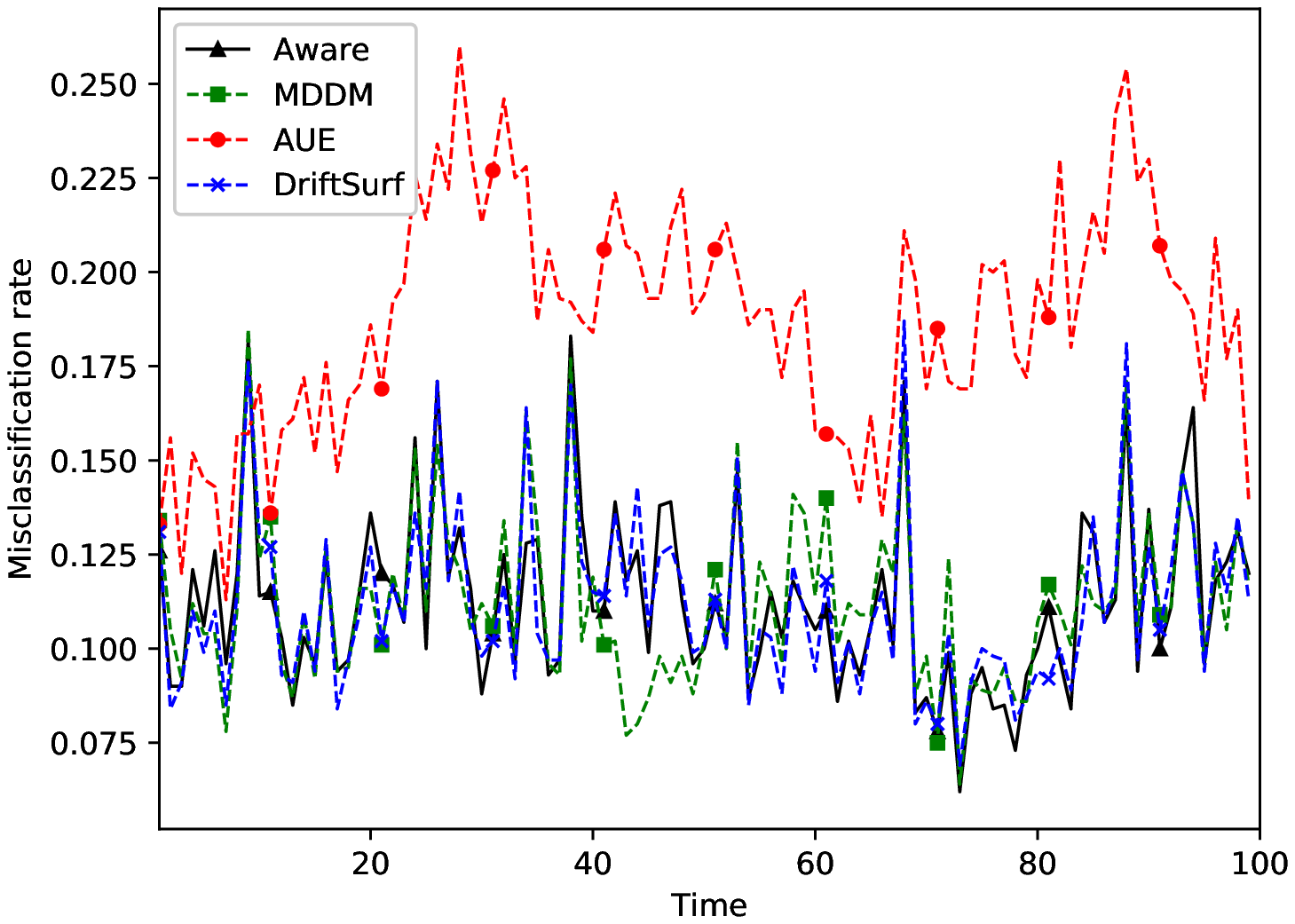}
            \caption{HyperPlane-slow}
            \label{fig:hyperplane-slow-limited-2}
        \end{subfigure}\hfill
        \begin{subfigure}[t]{0.32\textwidth}
            \includegraphics[width=\columnwidth]{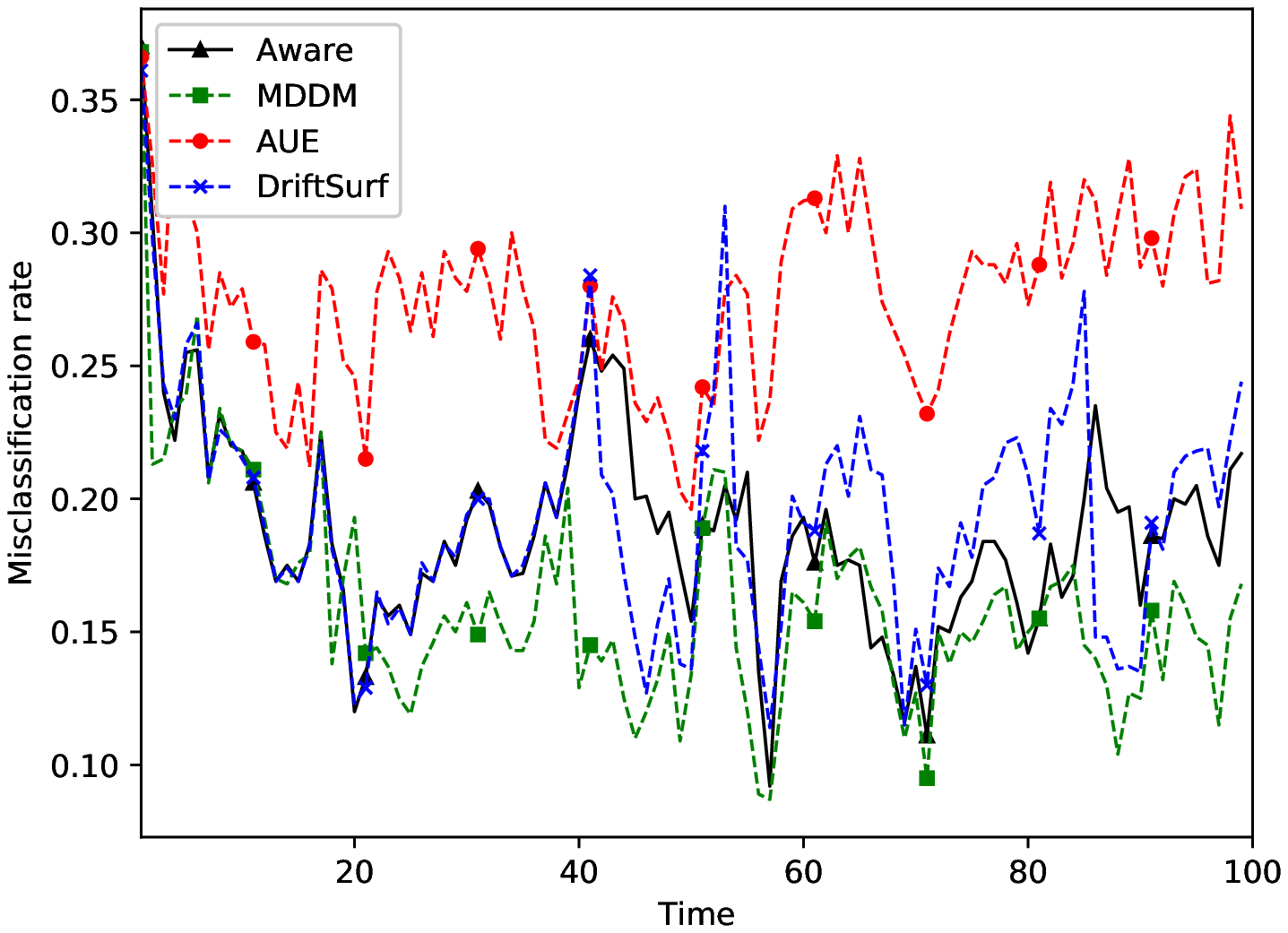}
            \caption{HyperPlane-fast}
            \label{fig:hyperplane-fast-limited-2}
        \end{subfigure}\hfill
        \begin{subfigure}[t]{0.32\textwidth}
            \includegraphics[width=\columnwidth]{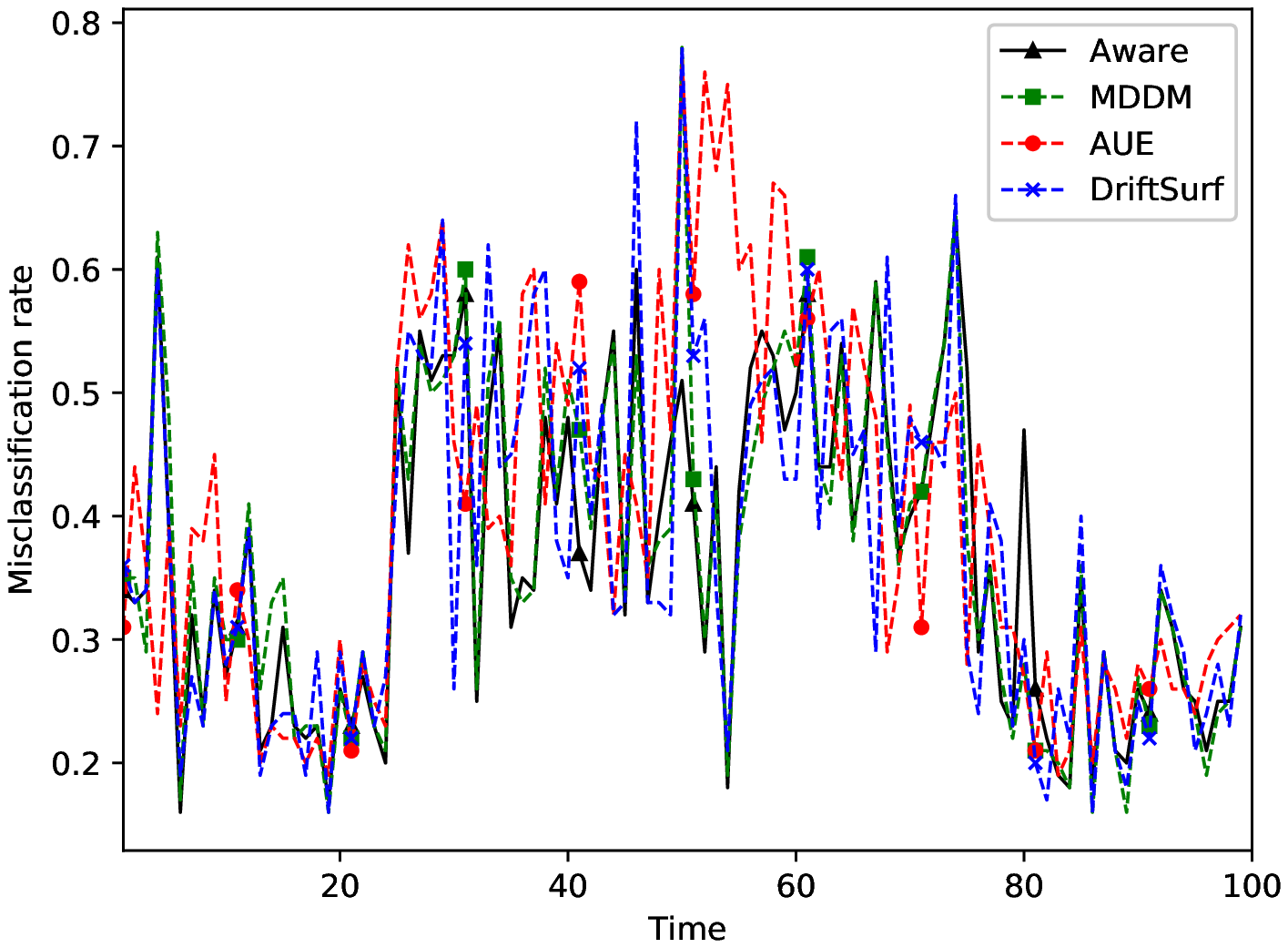}
            \caption{CIRCLES}
            \label{fig:circles-limited-2}
        \end{subfigure}\hfill
        
        \begin{subfigure}[t]{0.32\textwidth}
            \includegraphics[width=\columnwidth]{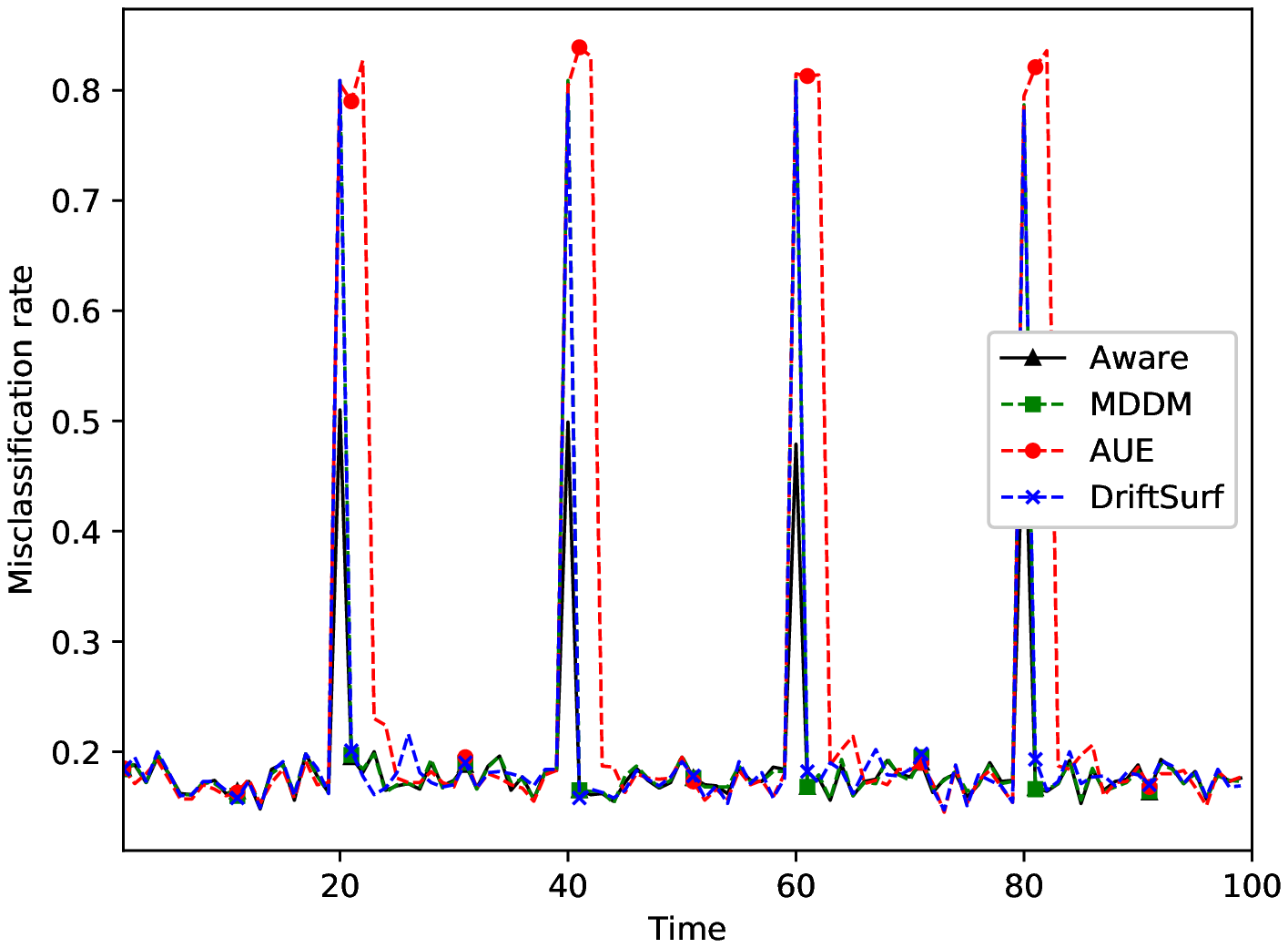}
            \caption{MIXED}
            \label{fig:mixed-limited-2}
        \end{subfigure}\hfill
        \begin{subfigure}[t]{0.32\textwidth}
            \includegraphics[width=\columnwidth]{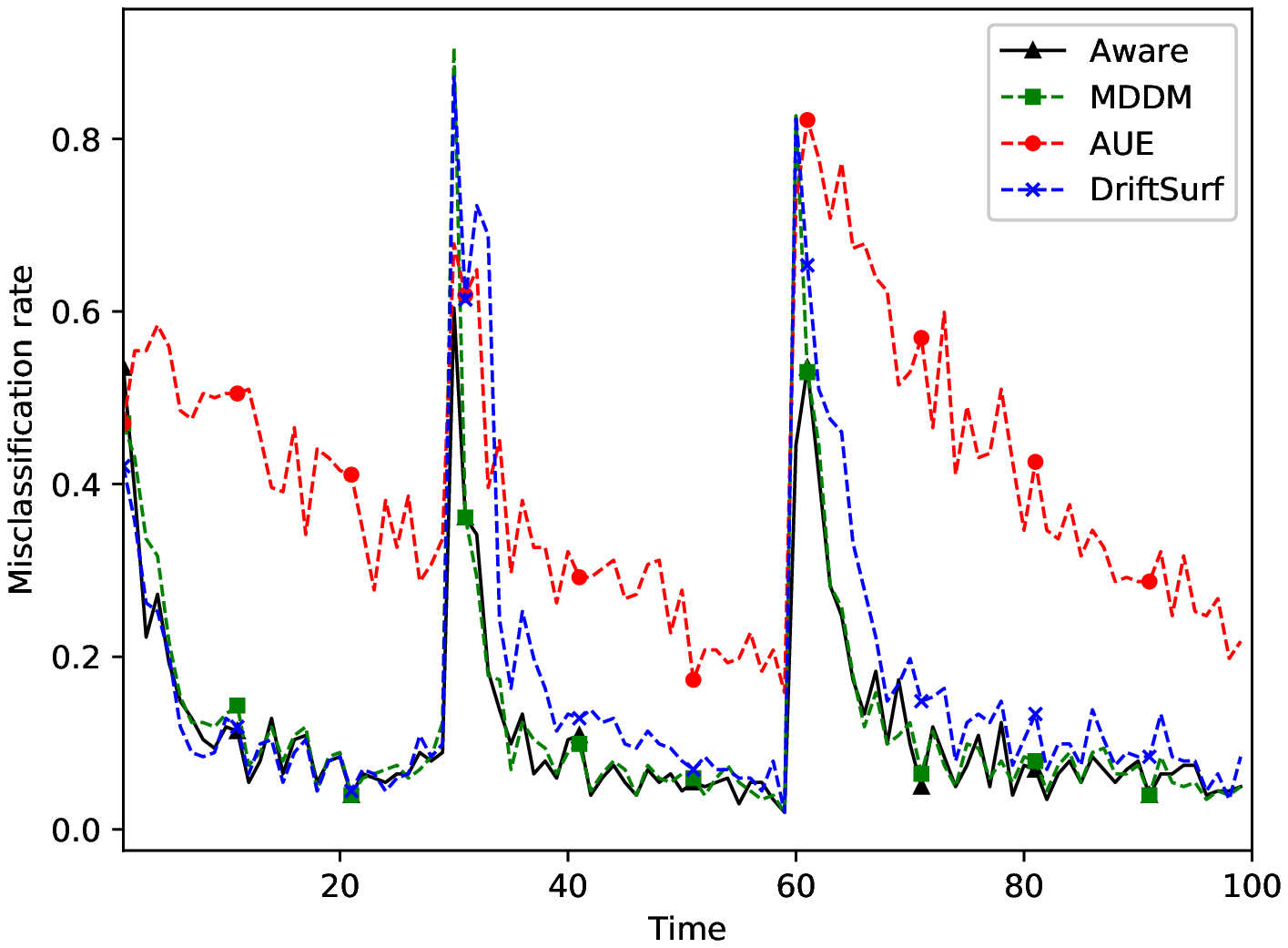}
            \caption{RCV1}
            \label{fig:rcv-limited-2}
        \end{subfigure}\hfill
        \begin{subfigure}[t]{0.32\textwidth}
            \includegraphics[width=\columnwidth]{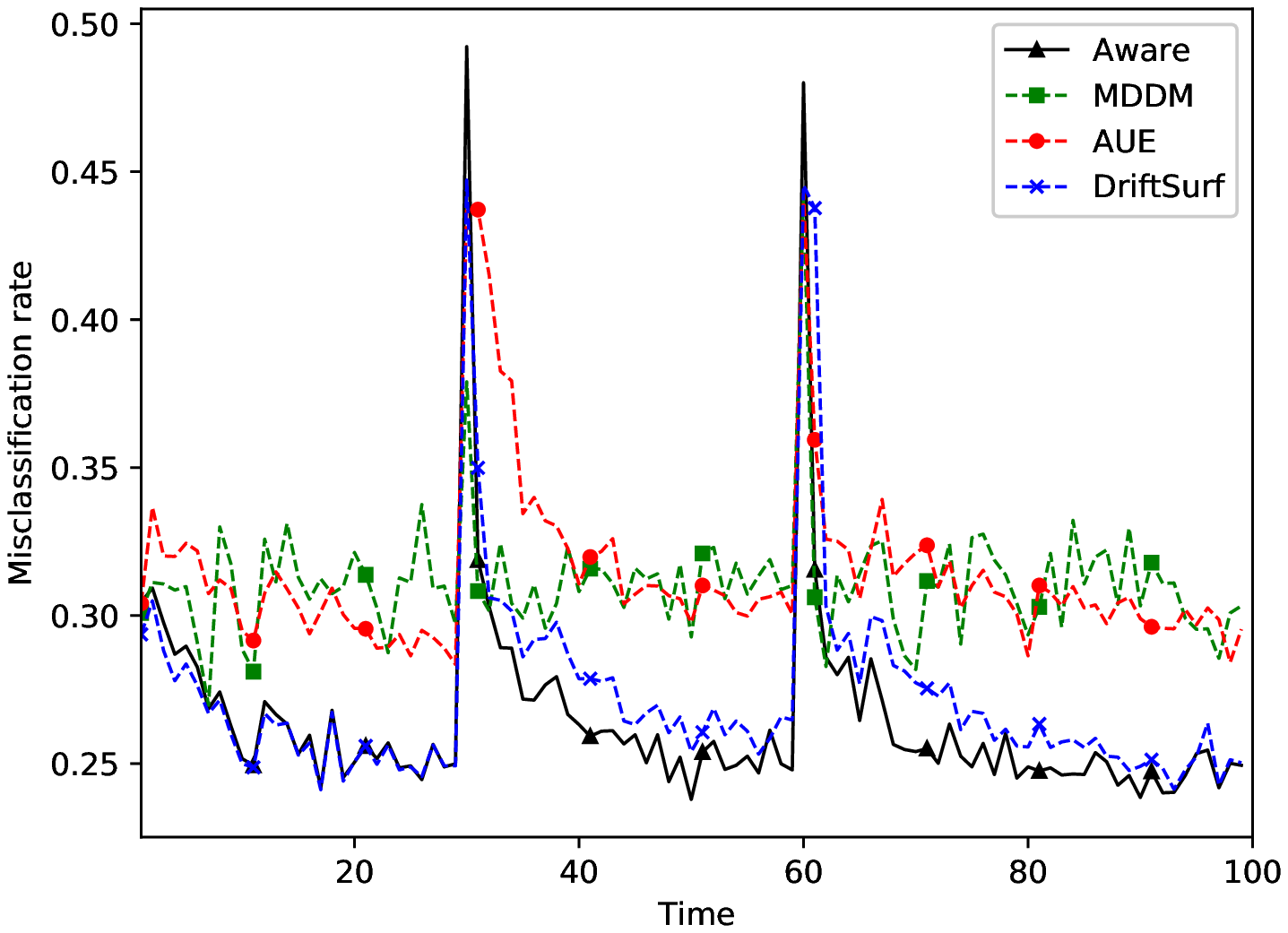}
            \caption{CoverType}
            \label{fig:covtype-limited-2}
        \end{subfigure}\hfill
        
        \begin{subfigure}[t]{0.32\textwidth}
            \includegraphics[width=\columnwidth]{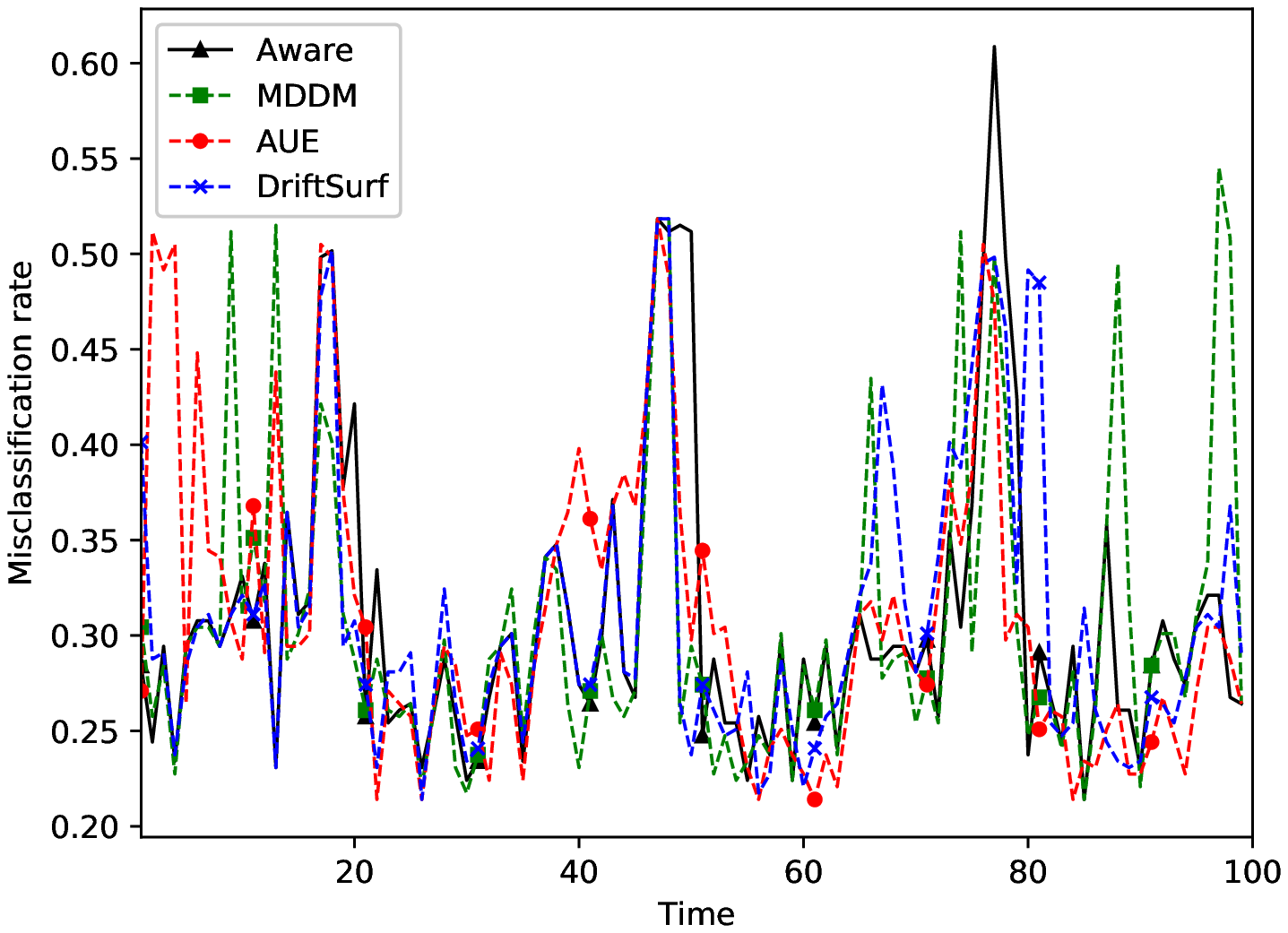}
            \caption{PowerSupply}
            \label{fig:powersupply-limited-2}
        \end{subfigure}\hfill
        \begin{subfigure}[t]{0.32\textwidth}
            \includegraphics[width=\columnwidth]{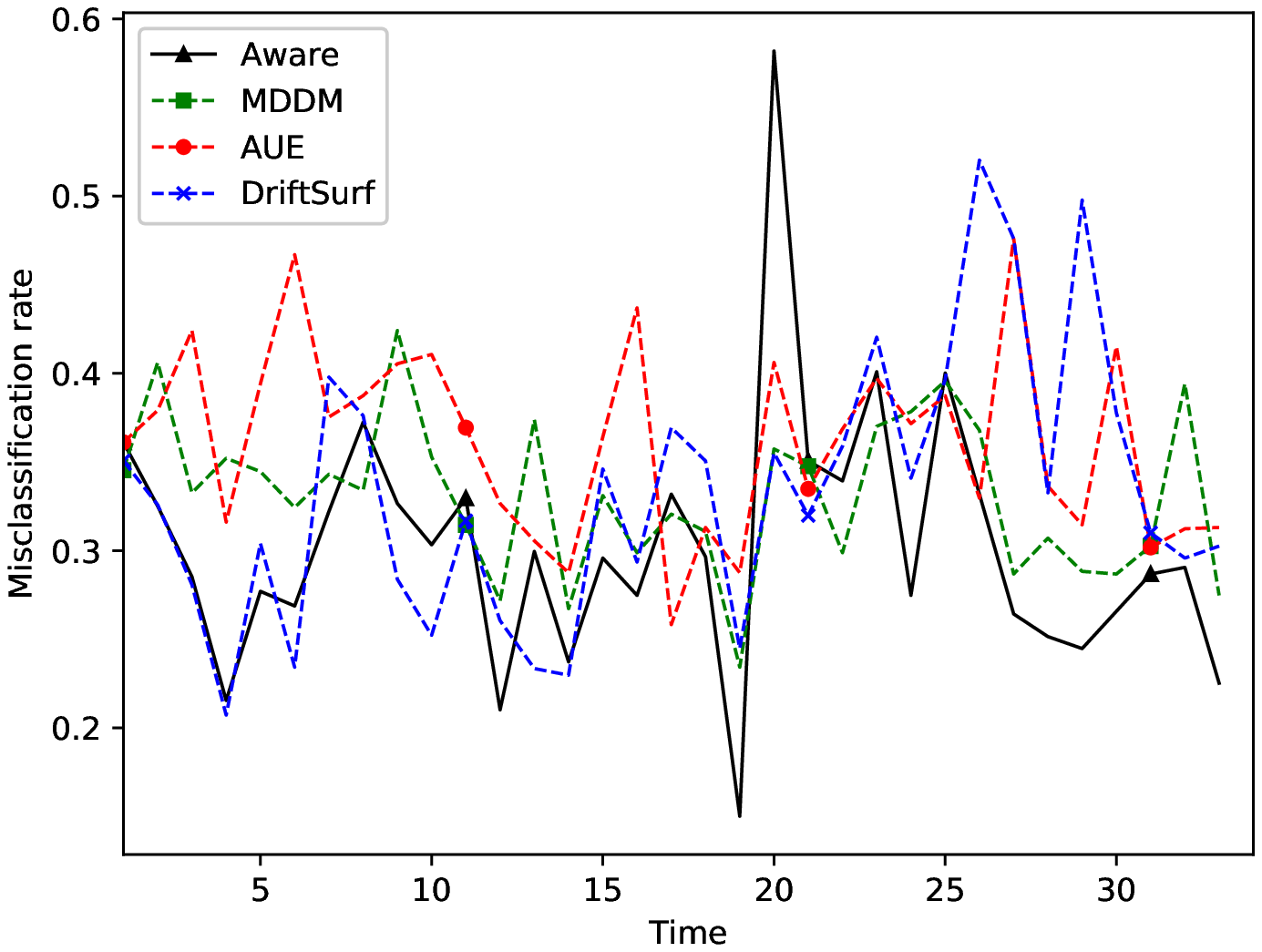}
            \caption{Electricity}
            \label{fig:elec-limited-2}
        \end{subfigure}\hfill
        \begin{subfigure}[t]{0.32\textwidth}
            \includegraphics[width=\columnwidth]{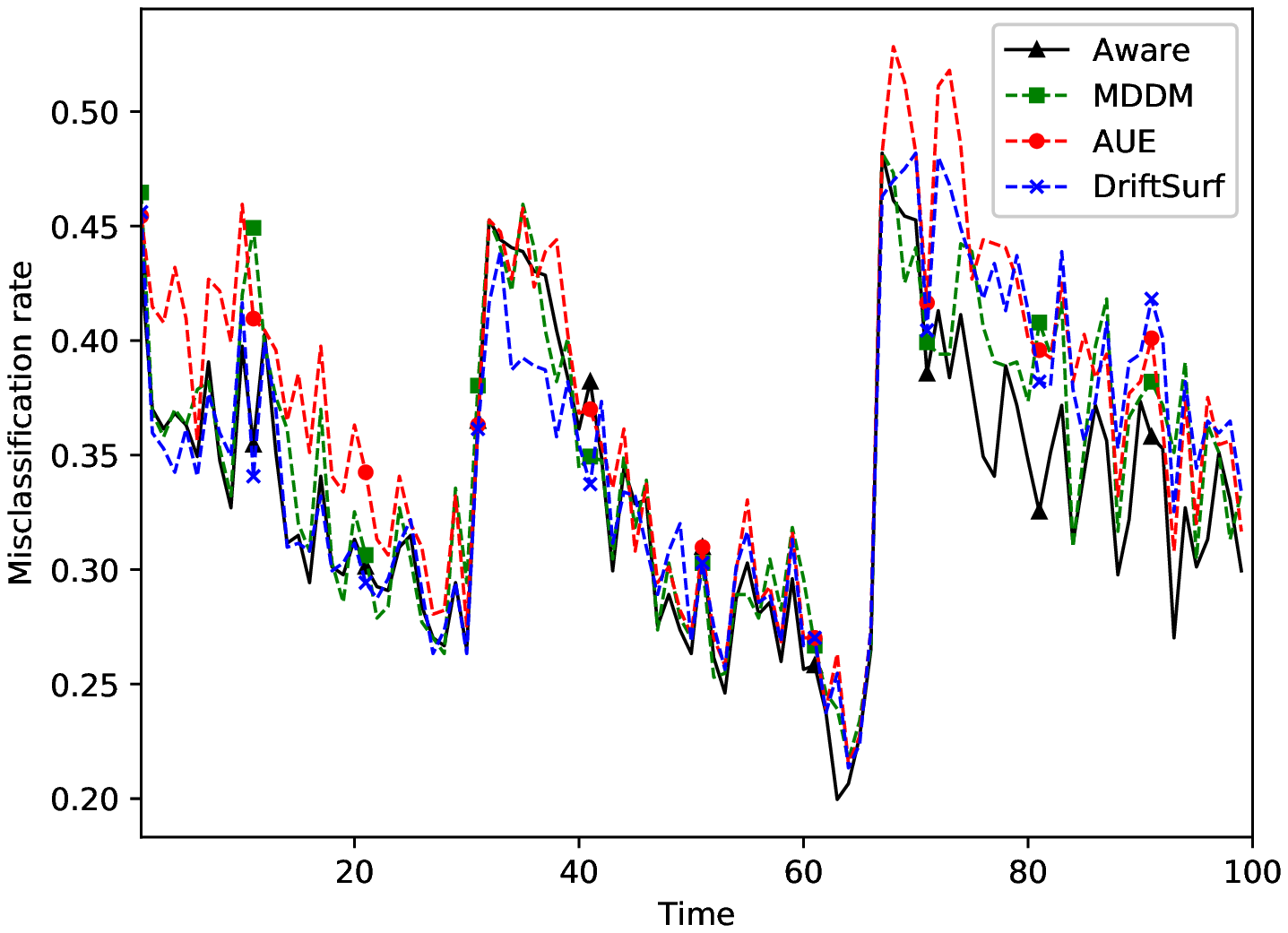}
            \caption{Airline}
            \label{fig:airline-limited-2}
        \end{subfigure}\hfill
    \end{center}
    \caption{Misclassification rate over time ($\rho=2m$ divided among all models of each algorithm)}
    \label{fig:Misclassification rate over time-rho=2m}
\end{figure*}

\subsection{Comparison to 1PASS-SGD and Oblivious}
\label{sec:expt-results-oblivious}

Figure \ref{fig:Misclassification rate over time-sgd_obl} shows the comparison to 1PASS-SGD and the oblivious algorithm (OBL) for the RCV1 and Electricity datasets at each time. The time average misclassification rate for each dataset are in Table \ref{table:ave-misclassification-app}. In the case of the large, abrupt drift in RCV1, we observe that 1PASS-SGD and especially oblivious have poor performance after drift. The oblivious algorithm continues to re-sample the data from the older distributions, and leads to a model with random, or worse than random, accuracy on the current distribution. Even for 1PASS-SGD, which only trains over data from the most recent time step, we observe its convergence rate is slow after a drift, where its previous training on the old data still hinders it. On the Electricity data with a more subtle drift, we observe that oblivious is actually the best performing algorithm, as discussed earlier, because data from all over time can be trained and fit by a single model. However, 1PASS-SGD still has lower accuracy because, as a single pass method, it uses only $m$ update steps at each time even when $\rho = 2m$ are available to the other algorithms, and also because SGD has a slower convergence rate than the variance-reduced method \incsaga.

\subsection{Evaluation of Greedy Reactive State}
\label{sec:expt-results-no-greedy}

One design choice in the \dsurf algorithm is that during the reactive state, the predictive model follows a greedy approach---the choice of the predictive model at the current time is the model that had the better performance in the previous time step---and then at the end of the reactive state, the decision is made whether or not to use the reactive model going forward. The natural alternative choice is that switching to the new reactive model can happen only at the end of the reactive state, and the stable model is the predictive model throughout the reactive state; we call this \dsurf (no-greedy). In Figure \ref{fig:rcv-greedy} and Table \ref{table:ave-misclassification-greedy} we show the comparison of \dsurf to \dsurf (no-greedy). We observe that \dsurf performs similar or better across each dataset, with the biggest improvements on the SINE1, RCV1, and Mixed datasets that we earlier observed MDDM and \aware perform well on because it is desirable to immediately switch to the new model after the large, abrupt drift. Figure \ref{fig:rcv-greedy} shows the delayed switch of \dsurf (no-greedy) to the new model in the presence of drift until only the end of the reactive state.

\begin{figure*}[h!]
    \begin{center}
        \minipage{0.64\textwidth}
            \begin{subfigure}[t]{0.5\textwidth}
                \includegraphics[width=\columnwidth]{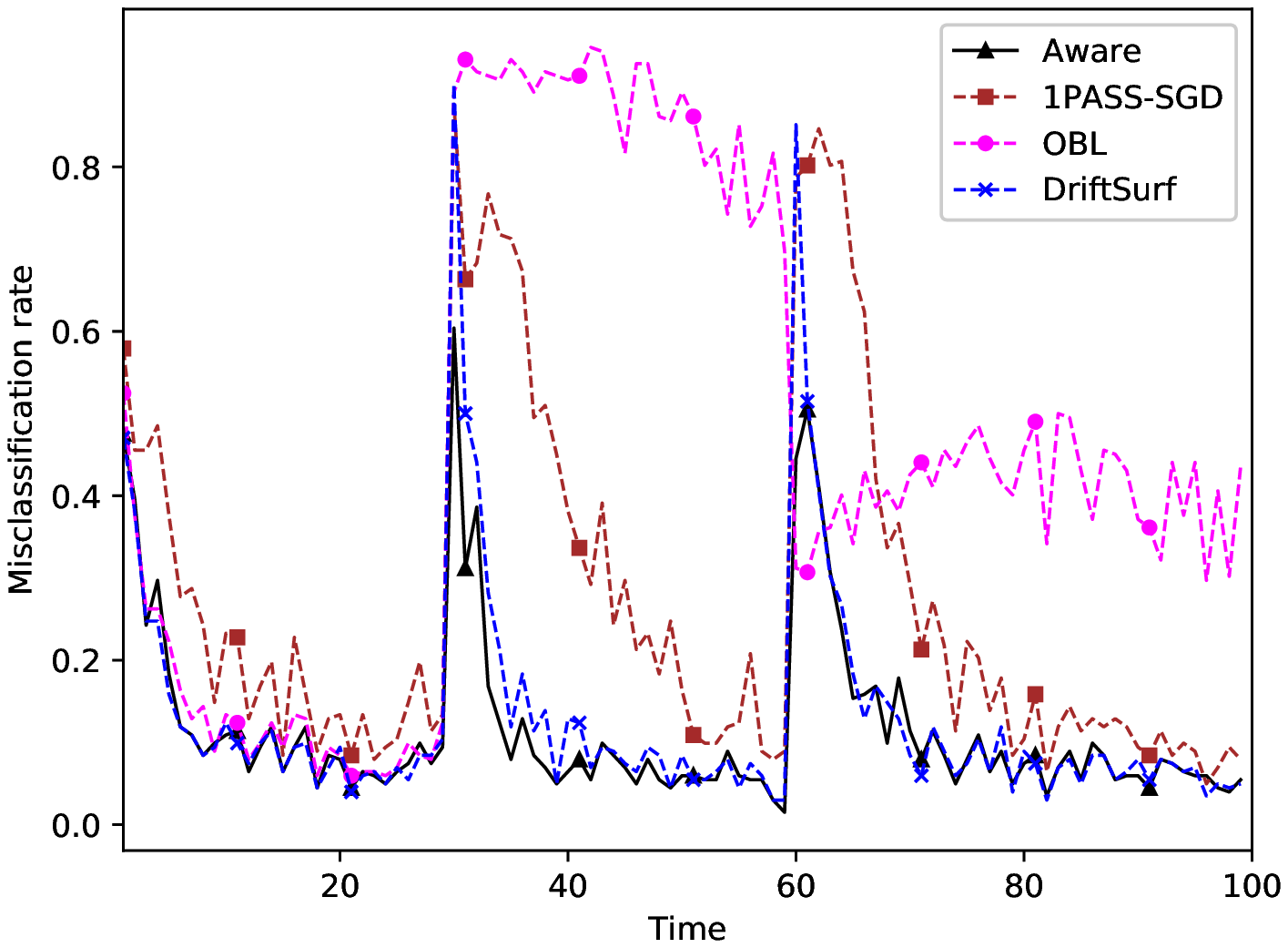}
                \caption{RCV1}
                \label{fig:rcv-sgd_obl}
            \end{subfigure}\hfill
            \begin{subfigure}[t]{0.5\textwidth}
                \includegraphics[width=\columnwidth]{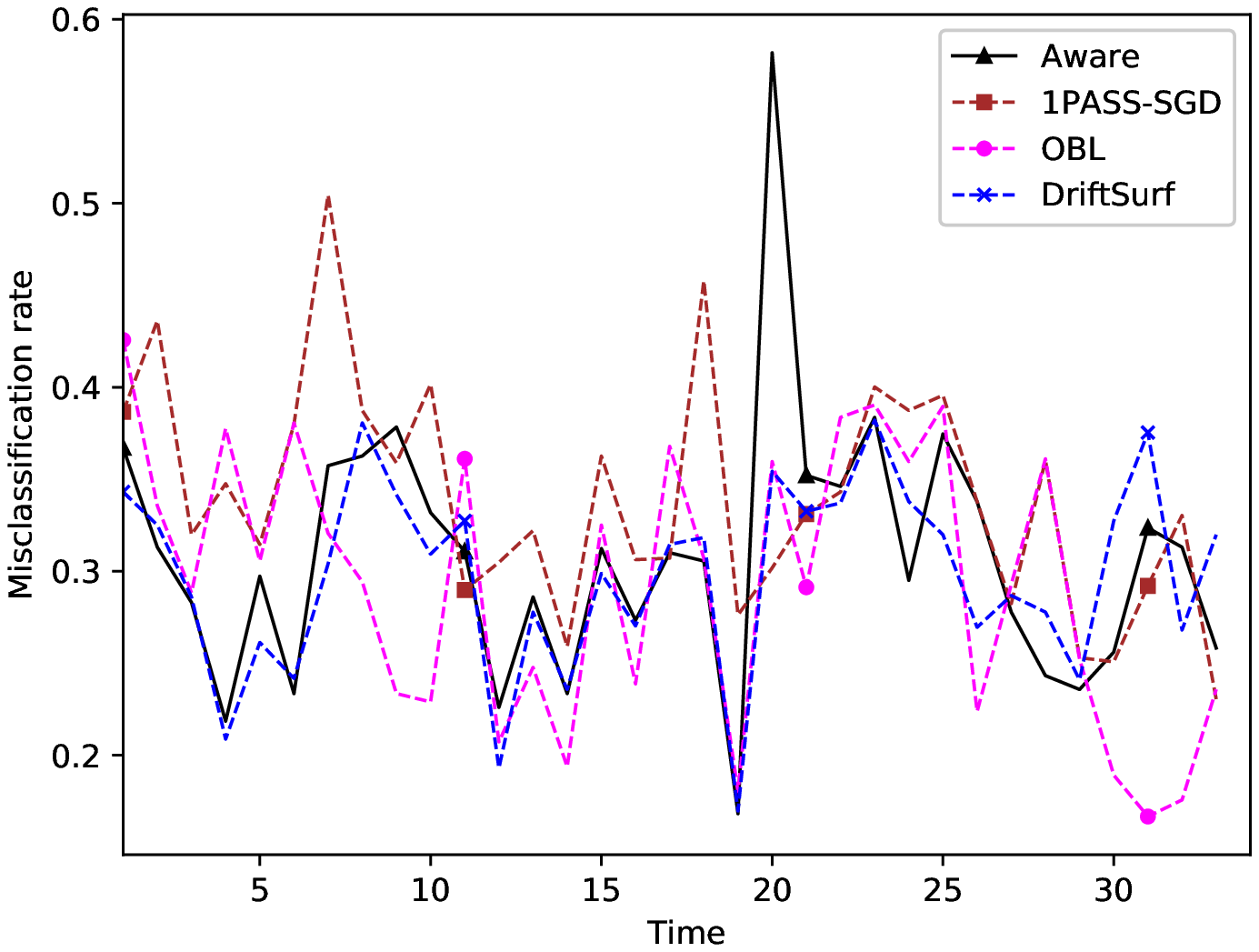}
                \caption{Electricity}
                \label{fig:elec-sgd_obl}
            \end{subfigure}\hfill
            \caption{Misclassification rate over time ($\rho=2m$ for each model)}
            \label{fig:Misclassification rate over time-sgd_obl}
        \endminipage\hfill
        \minipage{0.32\textwidth}
            \centerline{\includegraphics[width=\columnwidth]{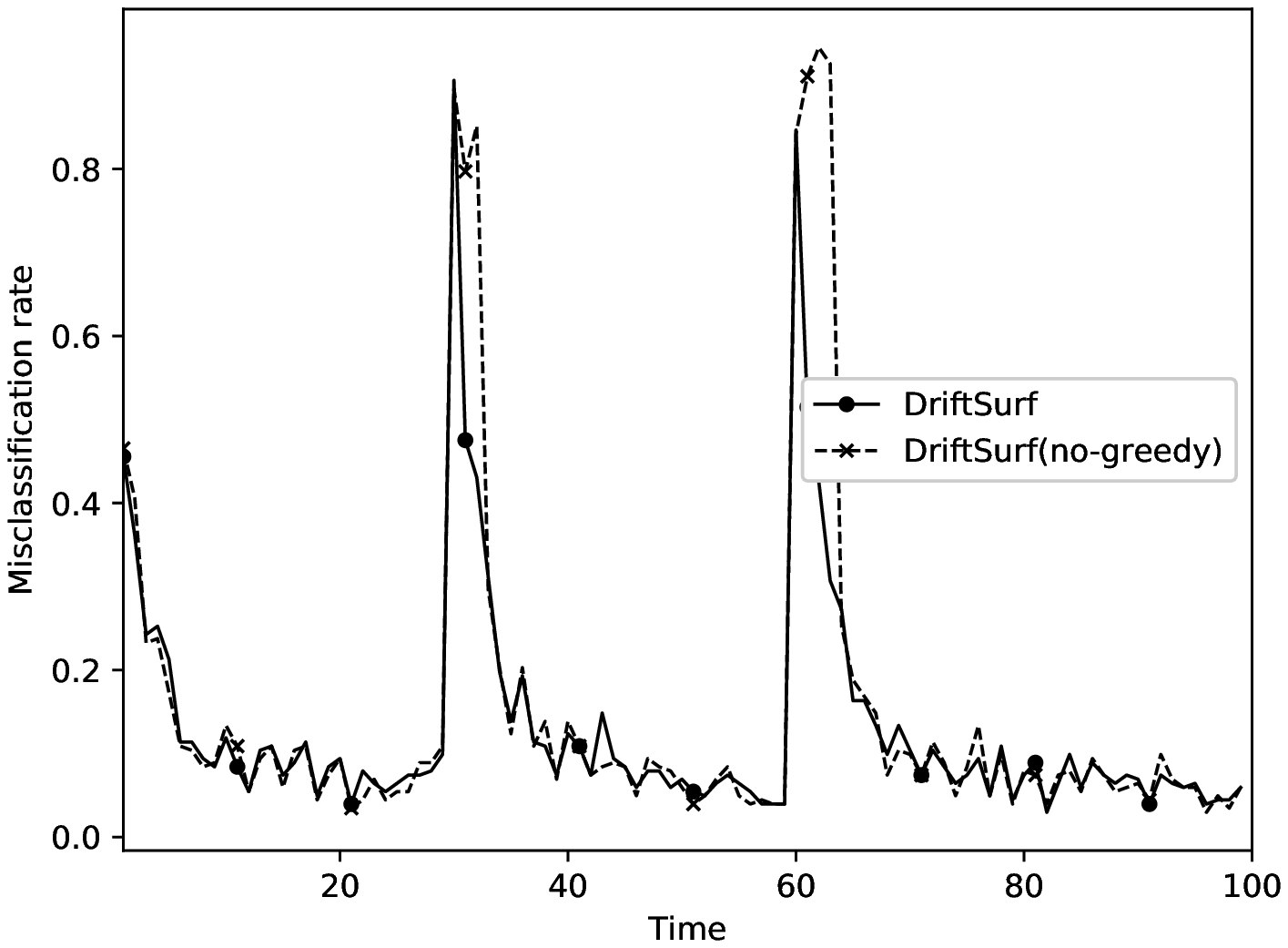}}
            \caption{Misclassification rate over time for RCV1 - \dsurf vs \dsurf (no-greedy) ($\rho=2m$ for each model)}
            \label{fig:rcv-greedy}
        \endminipage\hfill
    
    \end{center}
\end{figure*}

\begin{table}[ht!]
\caption{Total average of misclassification rate - \dsurf vs \dsurf (no-greedy) ($\rho=2m$ for each model)}
\label{table:ave-misclassification-greedy}
\begin{center}
\begin{small}
\begin{sc}
\begin{tabular}{l*{2}c}
\toprule
 Dataset    	& \dsurf & \dsurf(no-greedy)\\
\midrule
SEA0        	& 0.085 & 0.085 \\
SEA10       	& 0.160 & 0.158 \\
SEA20       	& 0.245 & 0.245 \\
SEA30       	& 0.336 & 0.335 \\
SEA-gradual     & 0.161 & 0.159 \\
Hyper-slow  	& 0.116 & 0.117  \\
Hyper-fast  	& 0.173 & 0.175 \\
Mixed       		& 0.204 & 0.232 \\
SINE1       	& 0.199 & 0.212 \\
Circles       	& 0.372 & 0.375 \\
RCV1        	& 0.136 & 0.263 \\
CoverType   	& 0.266 & 0.273 \\
Airline     		& 0.333 & 0.333 \\
Electricity 		& 0.290 & 0.291 \\
PowerSupply 	& 0.301 & 0.303 \\
\bottomrule
\end{tabular}
\end{sc}
\end{small}
\end{center}
\end{table}

\subsection{Using SGD as the Update Process}
\label{sec:expt-results-sgd}

As mentioned earlier we choose \incsaga as the update process because of two main reasons: (i) \incsaga is designed in a way that can handle different arrival distributions, and (ii)  it achieves a faster convergence rate because of using variance-reduced update step. We study the impact of the choice of the update process on the performance. We re-run the previous experiments using SGD instead of \incsaga. Table \ref{table:ave-misclassification-sgd} shows the average misclassification rate for the case where $\rho=2m$ update steps are used for each model.

As the results presented in Table~\ref{table:ave-misclassification-sgd} suggest, AUE, unlike the previous experiment, outperforms MDDM and \dsurf for the majority of the studied datasets. The reason is that AUE mitigates the high variance of SGD. MDDM and \dsurf both use performance-degradation for drift detection. Such drift detection is sensitive to the high variance during the training which may be mistaken for drift in the underlying distribution. However, comparing the results of \dsurf and MDDM shows the advantage of going though a reactive state before restarting the model in reducing the false positive rate of drift detection. AUE, on the other hand, overcomes the high variance of SGD by using a bag of experts and making ensemble based decisions. 

Similar to the previous experiments, to examine the accuracy achieved when enforcing equal processing time, we repeated the experiment for the case where $\rho=2m$ steps are used by each algorithm and divided among its models. Reported results in Table~\ref{table:ave-misclassification-sgd-rho/m=2} suggest that the variance-reduction effect of AUE is not able to overcome the limited training.

\incsaga because of its variance-reduced update step achieves a faster convergence rate in comparison to SGD. Difference between the reported results in Table~\ref{table:ave-misclassification-app} and Table~\ref{table:ave-misclassification-sgd} confirms the advantage of using \incsaga over SGD as the update process.

\begin{table}[ht!]
\caption{Total average of misclassification rate - update process: SGD ($\rho=2m$ for each model)}
\label{table:ave-misclassification-sgd}
\begin{center}
\begin{small}
\begin{sc}
\begin{tabular}{l*{9}c}
\toprule
 Dataset    	& \aware 		& \dsurf		& MDDM-G 	& AUE \\
\midrule
SEA0       		& 0.170 		& \textbf{0.118} & 0.127 		& 0.125  \\
SEA10       	& 0.217 		& 0.194 		& 0.197 		& \textbf{0.184}\\
SEA20       	& 0.279 		& \textbf{0.260} & 0.296 		& 0.263 \\
SEA30       	& 0.360 		& 0.346		 & 0.382 		& \textbf{0.340} \\
SEA-gradual    & 0.205 		& \textbf{0.184} & 0.216 		& 0.188 \\
Hyper-slow  	& 0.169 		& 0.151 		& 0.140 		& \textbf{0.124} \\
Hyper-fast  	& 0.272 		& 0.199 		& \textbf{0.179} & 0.204 \\
Mixed       		& 0.194 		& \textbf{0.206}	& 0.209 		& 0.242 \\
SINE1       	& 0.194 		& 0.255 		& \textbf{0.200} & 0.239 \\
Circles       	& 0.362 		& 0.375 		& 0.386		 & \textbf{0.362} \\
RCV1        	& 0.151 		& 0.170 		& \textbf{0.162} & 0.208\\
CoverType   	& 0.274 		& \textbf{0.275} & 0.326 		& 0.286 \\
Airline     		& 0.356 		& 0.365 		& 0.359 		& \textbf{0.343} \\
Electricity 		& 0.335 		& 0.314 		& 0.348 		& \textbf{0.299} \\
PowerSupply 	& 0.350 		& 0.318		& 0.365 		& \textbf{0.300}\\
\bottomrule
\end{tabular}
\end{sc}
\end{small}
\end{center}
\end{table}

\begin{table}[ht!]
\caption{Total average of misclassification rate - update process: SGD ($\rho=2m$ divided among all models of each algorithm)}
\label{table:ave-misclassification-sgd-rho/m=2}
\begin{center}
\begin{small}
\begin{sc}
\begin{tabular}{l*{9}c}
\toprule
 Dataset    	& \aware 	& \dsurf		& MDDM-G 	& AUE \\
\midrule
SEA0        	& 0.163 	& 0.125 		& \textbf{0.123} & 0.193 \\
SEA10       	& 0.229 	& \textbf{0.182} & 0.197 		& 0.238 \\
SEA20       	& 0.283 	& \textbf{0.268} & 0.311 		& 0.287 \\
SEA30      	& 0.366 	& \textbf{0.350} & 0.376 		& 0.358 \\
SEA-gradual    & 0.204 	& \textbf{0.185} & 0.196 		& 0.236 \\
Hyper-slow  	& 0.169 	& 0.158 		& \textbf{0.143} & 0.158 \\
Hyper-fast  	& 0.269 	& 0.211		& \textbf{0.185} & 0.274 \\
SINE1       	& 0.200 	& 0.341 		& \textbf{0.205} & 0.302  \\
Mixed       		& 0.195 	& \textbf{0.208}& 0.209		 & 0.262  \\
Circles       	& 0.306 	& 0.380 		& \textbf{0.367} & 0.429  \\
RCV1        	& 0.146 	& 0.204 		& \textbf{0.161} & 0.437 \\
CoverType   	& 0.275 	& \textbf{0.286} & 0.323 		& 0.316 \\
Airline     		& 0.354 	& \textbf{0.359} & 0.366 		& 0.370 \\
Electricity 		& 0.343 	& 0.356 		& \textbf{0.350}	& 0.354 \\
PowerSupply 	& 0.336 	& 0.318		 & 0.356 		& \textbf{0.316} \\
\bottomrule
\end{tabular}
\end{sc}
\end{small}
\end{center}
\end{table}

%% file: broader.tex
\section{Broader Impact}

There are many ethical and societal reasons to adapt to concept
drifts, because ML decisions should be based on the most relevant
data.  Consider, for example, using ML to recommend the terms for a
loan.  Here the data are individual profiles of employment status,
debt history, savings balance, etc., and the labels are whether the
individual repaid the loan (or more broadly, the history of
repayment).  The Covid-19 lockdown forced millions of people into
temporary unemployment, increased debt, and decreased
savings---features that would make the pre-lockdown ML model assess
them as big credit risks, when in fact, a new model is needed that
properly accounts for this temporary new reality.  Our algorithm would
seek to detect this concept drift, so that such people would not be
unjustly penalized for the lockdown.  In this setting, the
consequences of failure are: Failing to detect a real drift would
unduly penalize loan applicants for a lockdown beyond their control,
whereas falsely detecting a drift that did not exist would unfairly
evaluate applicants whose profiles best match individuals from before
the false detection.

Our implementation of \dsurf{} logs all its transitions between the
stable and reactive states, noting whenever the predictive model
changes.  This provides a measure of explanability/transparency:
Humans can review these logs to see when the model changed, and assess
whether the change was warranted (was there really a drift?) and
whether there are any biases in that decision.

%% file: neurips_2020.bbl
\begin{thebibliography}{10}

\bibitem{bach2008paired}
S.~H. Bach and M.~A. Maloof.
\newblock Paired learners for concept drift.
\newblock In {\em ICDM}, pages 23--32, 2008.

\bibitem{baena2006early}
M.~Baena-Garc{\'i}a, J.~del Campo-{\'A}vila, R.~Fidalgo, A.~Bifet,
  R.~Gavald{\`a}, and R.~Morales-Bueno.
\newblock Early drift detection method.
\newblock In {\em StreamKDD}, pages 77--86, 2006.

\bibitem{besbes2015non}
O.~Besbes, Y.~Gur, and A.~Zeevi.
\newblock Non-stationary stochastic optimization.
\newblock {\em Operations Research}, 63(5):1227--1244, 2015.

\bibitem{bifet2007learning}
A.~Bifet and R.~Gavald{\`a}.
\newblock Learning from time-changing data with adaptive windowing.
\newblock In {\em ICDM}, pages 443--448, 2007.

\bibitem{bifet2010moa}
A.~Bifet, G.~Holmes, R.~Kirkby, and B.~Pfahringer.
\newblock M{OA}: Massive online analysis.
\newblock {\em JMLR}, 11:1601--1604, 2010.

\bibitem{bousquet2008tradeoffs}
O.~Bousquet and L.~Bottou.
\newblock The tradeoffs of large scale learning.
\newblock In {\em NIPS}, pages 161--168, 2007.

\bibitem{brzezinski2013reacting}
D.~Brzezinski and J.~Stefanowski.
\newblock Reacting to different types of concept drift: The accuracy updated
  ensemble algorithm.
\newblock {\em IEEE Trans. Neural Netw. Learn. Syst}, 25(1):81--94, 2013.

\bibitem{chiu2018diversity}
C.~W. Chiu and L.~L. Minku.
\newblock Diversity-based pool of models for dealing with recurring concepts.
\newblock In {\em IJCNN}, pages 1--8, 2018.

\bibitem{daneshmand2016starting}
H.~Daneshmand, A.~Lucchi, and T.~Hofmann.
\newblock Starting small-learning with adaptive sample sizes.
\newblock In {\em {ICML}}, pages 1463--1471, 2016.

\bibitem{dau2019ucr}
H.~A. Dau, A.~Bagnall, K.~Kamgar, C.-C.~M. Yeh, Y.~Zhu, S.~Gharghabi, C.~A.
  Ratanamahatana, and E.~Keogh.
\newblock The {UCR} time series archive.
\newblock {\em IEEE/CAA Journal of Automatica Sinica}, 6(6):1293--1305, 2019.

\bibitem{UCI-ML}
D.~Dua and C.~Graff.
\newblock {UCI} machine learning repository, 2017.

\bibitem{elwell2011incremental}
R.~Elwell and R.~Polikar.
\newblock Incremental learning of concept drift in nonstationary environments.
\newblock {\em IEEE Trans. Neural Netw.}, 22(10):1517--1531, 2011.

\bibitem{gama2004learning}
J.~Gama, P.~Medas, G.~Castillo, and P.~Rodrigues.
\newblock Learning with drift detection.
\newblock In {\em Advances in Artificial Intelligence-SBIA}, pages 286--295,
  2004.

\bibitem{gama2014survey}
J.~Gama, I.~{\v{Z}}liobait{\.e}, A.~Bifet, M.~Pechenizkiy, and A.~Bouchachia.
\newblock A survey on concept drift adaptation.
\newblock {\em ACM Comput. Surv.}, 46(4):44, 2014.

\bibitem{harel2014concept}
M.~Harel, K.~Crammer, R.~El-Yaniv, and S.~Mannor.
\newblock Concept drift detection through resampling.
\newblock In {\em ICML}, pages 1009--1017, 2014.

\bibitem{Harries99splice-2comparative}
M.~Harries.
\newblock Splice-2 comparative evaluation: Electricity pricing.
\newblock Technical report, University of New South Wales, 1999.

\bibitem{hentschel2019online}
B.~Hentschel, P.~J. Haas, and Y.~Tian.
\newblock Online model management via temporally biased sampling.
\newblock {\em ACM SIGMOD Record}, 48(1):69--76, 2019.

\bibitem{ElenaIko20:online}
E.~Ikonomovska.
\newblock Airline dataset.
\newblock (Accessed on 02/06/2020).

\bibitem{janson2018tail}
S.~Janson.
\newblock Tail bounds for sums of geometric and exponential variables.
\newblock {\em Statistics \& Probability Letters}, 135:1--6, 2018.

\bibitem{jothimurugesan2018variance}
E.~Jothimurugesan, A.~Tahmasbi, P.~B. Gibbons, and S.~Tirthapura.
\newblock Variance-reduced stochastic gradient descent on streaming data.
\newblock In {\em NeurIPS}, pages 9906--9915, 2018.

\bibitem{kifer2004detecting}
D.~Kifer, S.~Ben-David, and J.~Gehrke.
\newblock Detecting change in data streams.
\newblock In {\em VLDB}, pages 180--191, 2004.

\bibitem{klinkenberg2004learning}
R.~Klinkenberg.
\newblock Learning drifting concepts: Example selection vs. example weighting.
\newblock {\em IDA}, 8(3):281--300, 2004.

\bibitem{kolter2007dynamic}
J.~Z. Kolter and M.~A. Maloof.
\newblock Dynamic weighted majority: An ensemble method for drifting concepts.
\newblock {\em JMLR}, 8:2755--2790, 2007.

\bibitem{koychev2000gradual}
I.~Koychev.
\newblock Gradual forgetting for adaptation to concept drift.
\newblock In {\em ECAI Workshop on Current Issues in Spatio-Temporal
  Reasoning}, 2000.

\bibitem{lewis2004rcv1}
D.~D. Lewis, Y.~Yang, T.~G. Rose, and F.~Li.
\newblock {RCV}1: A new benchmark collection for text categorization research.
\newblock {\em JMLR}, 5:361--397, 2004.

\bibitem{lu2017dynamic}
Y.~Lu, Y.-m. Cheung, and Y.~Y. Tang.
\newblock Dynamic weighted majority for incremental learning of imbalanced data
  streams with concept drift.
\newblock In {\em IJCAI}, pages 2393--2399, 2017.

\bibitem{pesaranghader2016fast}
A.~Pesaranghader and H.~L. Viktor.
\newblock Fast hoeffding drift detection method for evolving data streams.
\newblock In {\em ECML PKDD}, pages 96--111, 2016.

\bibitem{pesaranghader2016framework}
A.~Pesaranghader, H.~L. Viktor, and E.~Paquet.
\newblock A framework for classification in data streams using multi-strategy
  learning.
\newblock In {\em ICDS}, pages 341--355, 2016.

\bibitem{pesaranghader2018mcdiarmid}
A.~Pesaranghader, H.~L. Viktor, and E.~Paquet.
\newblock Mc{D}iarmid drift detection methods for evolving data streams.
\newblock In {\em IJCNN}, pages 1--9, 2018.

\bibitem{sebastiao2007change}
R.~Sebasti{\~a}o and J.~Gama.
\newblock Change detection in learning histograms from data streams.
\newblock In {\em PAI}, pages 112--123, 2007.

\bibitem{sun2018concept}
Y.~Sun, K.~Tang, Z.~Zhu, and X.~Yao.
\newblock Concept drift adaptation by exploiting historical knowledge.
\newblock {\em IEEE Trans. Neural Netw. Learn. Syst.}, 29(10):4822--4832, 2018.

\bibitem{wang2018minimizing}
G.~Wang, D.~Zhao, and L.~Zhang.
\newblock Minimizing adaptive regret with one gradient per iteration.
\newblock In {\em IJCAI}, pages 2762--2768, 2018.

\bibitem{widmer1996learning}
G.~Widmer and M.~Kubat.
\newblock Learning in the presence of concept drift and hidden contexts.
\newblock {\em Machine learning}, 23(1):69--101, 1996.

\bibitem{yi2016tracking}
J.~Yi, T.~Yang, L.~Zhang, and R.~Jin.
\newblock Tracking slowly moving clairvoyant: Optimal dynamic regret of online
  learning with true and noisy gradient.
\newblock In {\em ICML}, pages 449--457, 2016.

\bibitem{zhao2020handling}
P.~Zhao, L.-W. Cai, and Z.-H. Zhou.
\newblock Handling concept drift via model reuse.
\newblock {\em Machine Learning}, 109:533--568, 2020.

\end{thebibliography}
